\documentclass[11pt]{article} 
\usepackage{fullpage}
\usepackage{amsmath,amsfonts} 
\usepackage{amssymb,bm,empheq}
\usepackage{graphicx,hyperref,rotating}
\usepackage{mymath}
\newcommand{\citep}[1]{\cite{#1}}
\newcommand{\citet}[1]{\cite{#1}}

\usepackage{algorithm,algorithmic,color} \graphicspath{{./}} 
\begin{document}

\title{Unsupervised Supervised Learning II:\\ Training Margin Based Classifiers without Labels}
\author{Krishnakumar Balasubramanian\textsuperscript{*}\\School of Computational Science and Engineering\\College of Computing\\Georgia Institute of Technology\\Atlanta, GA \and Pinar Donmez\\ Yahoo! Labs \\701 First Ave., Sunnyale CA 94089, USA \\ \and Guy Lebanon\\School of Computational Science and Engineering\\College of Computing\\Georgia Institute of Technology\\Atlanta, GA}
\date{\today}
\maketitle
\thispagestyle{empty}
\let\oldthefootnote\thefootnote
\renewcommand{\thefootnote}{\fnsymbol{footnote}}
\footnotetext[1]{To whom correspondence should be addressed. Email: \url{krishnakumar3@gatech.edu }}
\let\thefootnote\oldthefootnote

\maketitle

\maketitle
\begin{abstract}  
Many popular linear classifiers, such as logistic regression, boosting, or SVM, are trained by optimizing a margin-based risk   function. Traditionally, these risk functions are computed based on   a labeled dataset. We develop a novel technique for estimating such   risks using only unlabeled data and the marginal label distribution. We prove that the proposed risk estimator is consistent on high-dimensional datasets and demonstrate it on synthetic and real-world data. In particular, we   show how the estimate is used for evaluating classifiers in transfer   learning, and for training classifiers with no labeled data   whatsoever.
\end{abstract} 

\section{Introduction} \label{sec:intro} Many popular linear classifiers, such as logistic regression, boosting, or SVM, are trained by optimizing a margin-based risk function. For standard linear classifiers $\hat Y=\text{sign} \sum\theta_jX_j$ with $Y\in\{-1,+1\}$, and $X,\theta\in\R^d$ the margin is defined as the product
\begin{align}
  Y f_{\theta}(X)\quad \text{where} \quad f_{\theta}(X) \defeq   \sum_{j=1}^d \theta_jX_j. \label{eq:defF}
\end{align}
Training such classifiers involves choosing a particular value of $\theta$. This is done by minimizing the risk or expected loss
\begin{align} \label{eq:defR} R(\theta) &= \E_{p(X,Y)}   L(Y,f_{\theta}(X))
\end{align}
with the three most popular loss functions
\begin{align} \label{eq:loss1}
  L_1(Y,f_{\theta}(X)) &= \exp\left(-Y f_{\theta}(X)\right)\\
  L_2(Y,f_{\theta}(X)) &= \log \left( 1+\exp\left(-Y f_{\theta}(X)     \right)\right) \label{eq:loss2}\\
  L_3(Y,f_{\theta}(X)) &= (1-Yf_{\theta}(X))_+. \label{eq:loss3}
\end{align}
being exponential loss $L_1$ (boosting), logloss $L_2$ (logistic regression) and hinge loss $L_3$ (SVM) respectively ($A_+$ above corresponds to $A$ if $A>0$ and 0 otherwise). 

Since the risk $R(\theta)$ depends on the unknown distribution $p$, it is usually replaced during training with its empirical counterpart \begin{align} 
  R_n(\theta) &= \frac{1}{n} \sum_{i=1}^n L(Y^{(i)},f_{\theta}(X^{(i)}))   \label{eq:empiricalLoss}
\end{align}
based on a labeled training set
\begin{align} \label{eq:labeledData}   (X^{(1)},Y^{(1)}),\ldots,(X^{(n)},Y^{(n)})\iid p
\end{align}
leading to the following estimator
\begin{align} \nonumber
  \hat\theta_n &= \argmin_{\theta} R_n(\theta).
\end{align}
Note, however, that evaluating and minimizing $R_n$ requires labeled data \eqref{eq:labeledData}. While suitable in some cases, there are certainly situations in which labeled data is difficult or impossible to obtain.
 
In this paper we construct an estimator for $R(\theta)$ using only unlabeled data, that is using
\begin{align} \label{eq:unlabeledData} X^{(1)},\ldots,X^{(n)} \iid p
\end{align}
instead of \eqref{eq:labeledData}. Our estimator is based on  the observations that when the data is high dimensional ($d\to\infty$) the quantities
\begin{align} \label{eq:condInProd} 
f_{\theta}(X)|\{Y=y\},\quad   y\in\{-1,+1\}
\end{align}
are often normally distributed. This phenomenon is supported by empirical evidence and may also be derived using non-iid central limit theorems. We then observe that the limit distributions of \eqref{eq:condInProd} may be estimated from unlabeled data \eqref{eq:unlabeledData} and that these distributions may be used to measure margin-based losses such as \eqref{eq:loss1}-\eqref{eq:loss3}. 
We examine two novel unsupervised applications: (i) estimating margin-based losses in transfer learning and (ii) training margin-based classifiers. We investigate these applications theoretically and also provide empirical results on synthetic and real-world data. Our empirical evaluation shows the effectiveness of the proposed framework in risk estimation and classifier training without any labeled data.

The consequences of estimating $R(\theta)$ without labels are indeed profound. Label scarcity is a well known problem which has lead to the emergence of semisupervised learning: learning using a few labeled examples and many unlabeled ones. The techniques we develop lead to a new paradigm that goes beyond semisupervised learning in requiring no labels whatsoever.

\section{Unsupervised Risk Estimation} \label{sec:riskEstimation} 
In this section we describe in detail the proposed estimation framework and discuss its theoretical properties. Specifically, we construct an estimator for $R(\theta)$ \eqref{eq:defR} using the unlabeled data \eqref{eq:unlabeledData} which we denote $\hat R_n(\theta\,;X^{(1)},\ldots,X^{(n)})$ or simply $\hat R_n(\theta)$ (to distinguish it from $R_n$ in \eqref{eq:empiricalLoss}).

Our estimation is based on two assumptions. The first assumption is that the label marginals $p(Y)$ are known and that $p(Y=1)\neq p(Y=-1)$. While this assumption may seem restrictive at first, there are many cases where it holds. Examples include medical diagnosis ($p(Y)$ is the well known marginal disease frequency), handwriting recognition or OCR ($p(Y)$ is the easily computable marginal frequencies of different letters in the English language), life expectancy prediction ($p(Y)$ is based on marginal life expectancy tables). In these and other examples $p(Y)$ is known with great accuracy even if labeled data is unavailable. Furthermore, this assumption may be replaced with a weaker form in which we know the ordering of the marginal distributions e.g., $p(Y=1)>p(Y=-1)$, but without knowing the specific values of the marginal distributions.

The second assumption is that the quantity $f_{\theta}(X)|Y$ follows a normal distribution. As $f_{\theta}(X)|Y$ is a linear combination of random variables, it is frequently normal when $X$ is high dimensional. From a theoretical perspective this assumption is motivated by the central limit theorem (CLT). The classical CLT states that $f_{\theta}(X)=\sum_{i=1}^d\theta_i X_i|Y$ is approximately normal for large $d$ if the data components $X_1,\ldots,X_d$ are iid given $Y$. A more general CLT states that $f_{\theta}(X)|Y$ is asymptotically normal if $X_1,\ldots,X_d|Y$ are independent (but not necessary identically distributed). Even more general CLTs state that $f_{\theta}(X)|Y$ is asymptotically normal if $X_1,\ldots,X_d|Y$ are not independent but their dependency is limited in some way. We examine this issue in Section~\ref{sec:CLT} and also show that the normality assumption holds empirically for several standard datasets. 

To derive the estimator we rewrite \eqref{eq:defR} by taking expectation with respect to $Y$ and $\alpha=f_{\theta}(X)$
\begin{align} \label{eq:risk2}
R(\theta) = \E_{p(f_{\theta}(X),Y)}   L(Y,f_{\theta}(X)) 
= \sum_{y\in\{-1,+1\}} p(y) \int_{\R} p(f_{\theta}(X)=\alpha|y)   L(y,\alpha) \, d\alpha.
\end{align}

Equation~\eqref{eq:risk2} involves three terms $L(y,\alpha)$, $p(y)$ and $p(f_{\theta}(X)=\alpha|y)$. The loss function $L$ is known and poses no difficulty. The second term $p(y)$ is assumed to be known (see discussion above). The third term is assumed to be normal 
$f_{\theta}(X)\,|\,\{Y=y\} = \sum_i \theta_i X_i \,| \, \{Y=y\}\sim N(\mu_y,\sigma_y)$ with parameters $\mu_y,\sigma_y$, $y\in\{-1,1\}$ that are estimated by maximizing the likelihood of a Gaussian mixture model. These estimated parameters are used to construct the plug-in estimator $\hat R_n(\theta)$ as follows.

\newcommand*\widefbox[1]{\fbox{\hspace{1em}#1\hspace{1em}}}

\begin{empheq}[box=\widefbox]{align} 
 \label{eq:ll1}
\ell_{n}(\mu,\sigma)
&= \sum_{i=1}^n \log \sum_{y^{(i)}\in\{-1,+1\}} p(y^{(i)})   p_{\mu_y,\sigma_y}(f_{\theta}(X^{(i)})|y^{(i)}).  \\  \label{eq:ll}
(\hat\mu^{(n)},\hat\sigma^{(n)})&=\argmax_{\mu,\sigma}     \ell_{n}(\mu,\sigma)\\ 
\hat R_{n}(\theta) &=  \sum_{y\in\{-1,+1\}}  p(y) \int_{\R}   p_{\hat\mu^{(n)}_y,\hat\sigma^{(n)}_y}(f_{\theta}(X)=\alpha|y)   L(y,\alpha) \, d\alpha. 
\label{eq:pluginEstimate}
\end{empheq}

We make the following observations.
\begin{enumerate}
\item Although we do not denote it explicitly, $\mu_y$ and $\sigma_y$ are functions of $\theta$. 
\item The loglikelihood \eqref{eq:ll1} does not use labeled data (it marginalizes over the label $y^{(i)}$). 
\item The parameter of the loglikelihood \eqref{eq:ll1} are $\mu=(\mu_{1},\mu_{-1})$ and $\sigma=(\sigma_{1},\sigma_{-1})$ rather than the parameter $\theta$ associated with the margin-based classifier. We consider the latter one as a fixed constant at this point. 
\item The estimation problem \eqref{eq:ll} is equivalent to the problem of maximum likelihood for means and variances of a Gaussian mixture model where the label marginals are assumed to be known. It is well known that in this case (barring the symmetric case of a uniform $p(y)$) the MLE converges to the true parameter values.
\item The estimator $\hat R_n$ \eqref{eq:pluginEstimate} is consistent in the limit of infinite unlabeled data
\[P\left(\lim_{n\to\infty} \hat R_{n}(\theta)=R(\theta)\right)=1.\] 
\item The two risk estimators $\hat R_n(\theta)$ \eqref{eq:pluginEstimate} and $R_n(\theta)$  \eqref{eq:empiricalLoss} approximate the expected loss $R(\theta)$. The latter uses labeled samples and is typically more accurate than the former for a fixed $n$.
\item Under suitable conditions $\argmin_{\theta} \hat R_n(\theta)$ converges to the expected risk minimizer
\[ P\left(\,\lim_{n\to\infty}  \,\argmin_{\theta\in\Theta} \,R_{n}(\theta)\,=\,\argmin_{\theta\in\Theta}\, R(\theta)\,\right)\,=\,1.\] 
This far reaching conclusion implies that in cases where $\argmin_{\theta} R(\theta)$ is the Bayes classifier (as is the case with exponential loss, log loss, and hinge loss) we can retrieve the optimal classifier without a single labeled data point.  
\end{enumerate}

\subsection{Asymptotic Normality of $f_{\theta}(X)|Y$} \label{sec:CLT}
The quantity $f_{\theta}(X)|Y$ is essentially a sum of $d$ random variables which for large $d$ is likely to be normally distributed. One way to verify this is empirically, as we show in  Figures~\ref{fig:CLT}-\ref{fig:CLT2} which contrast the histogram with a fitted normal pdf for text, digit images, and face images data. For these datasets the dimensionality $d$ is sufficiently high to provide a nearly normal $f_{\theta}(X)|Y$. For example, in the case of text documents ($X_i$ is the relative number of times word $i$ appeared in the document) $d$ corresponds to the vocabulary size which is typically a large number in the range $10^3-10^5$. Similarly, in the case of image classification ($X_i$ denotes the brightness of the $i$-pixel) the dimensionality is on the order of $10^2-10^4$.

Figures~\ref{fig:CLT}-\ref{fig:CLT2} show that in these cases of text and image data $f_{\theta}(X)|Y$ is approximately normal for both randomly drawn $\theta$ vectors (Figure~\ref{fig:CLT}) and for $\theta$ representing estimated classifiers (Figure~\ref{fig:CLT2}). The single caveat in this case is that normality may not hold when $\theta$ is sparse, as may happen for example for $l_1$ regularized models (last row of Figure~\ref{fig:CLT2}).

\begin{figure} 
\centering     
\begin{tabular}{ccc}
{\scriptsize RCV1 text data} & &
{\scriptsize face images} \\
  \includegraphics[width=0.31\textwidth]{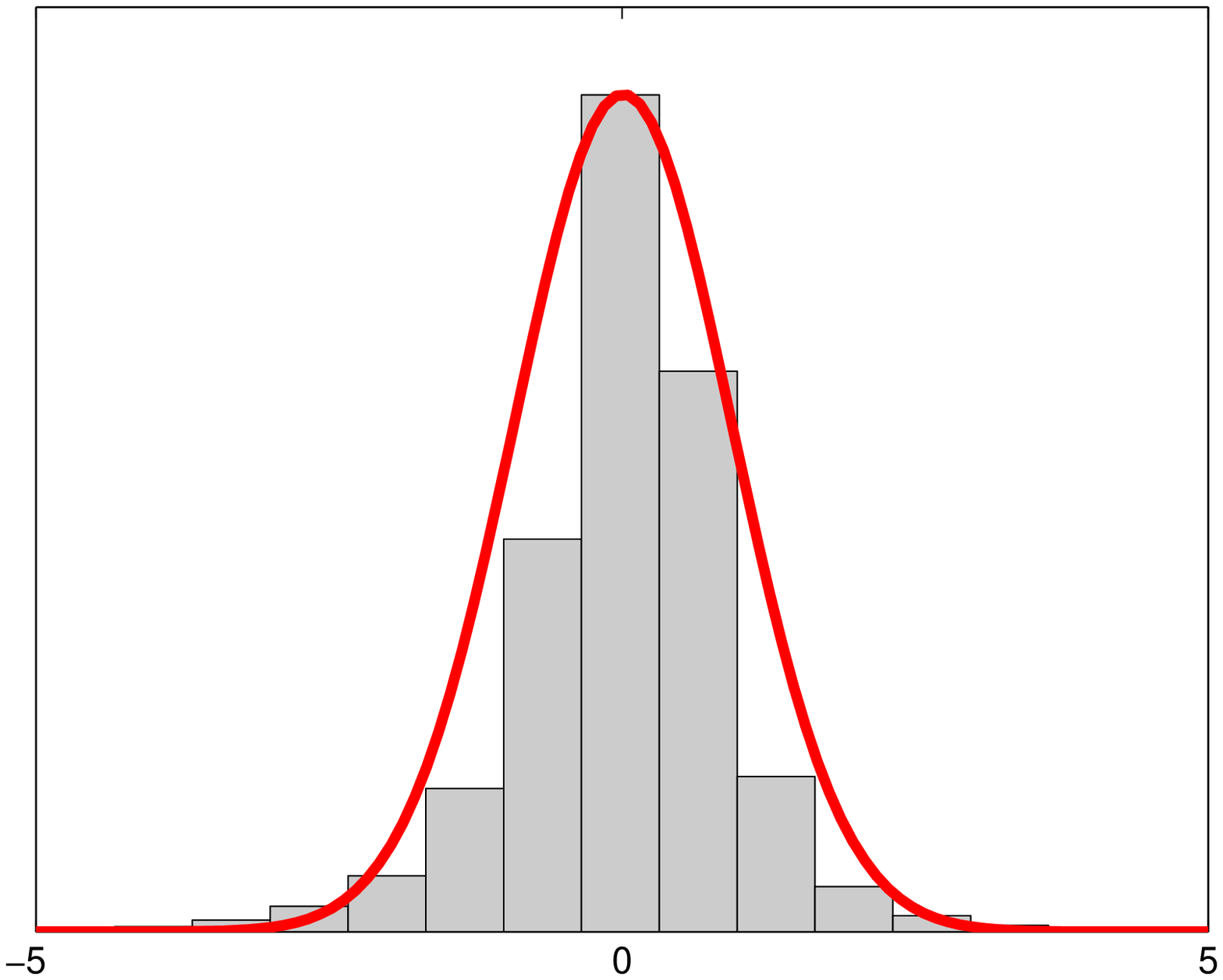}&   \includegraphics[width=0.31\textwidth]{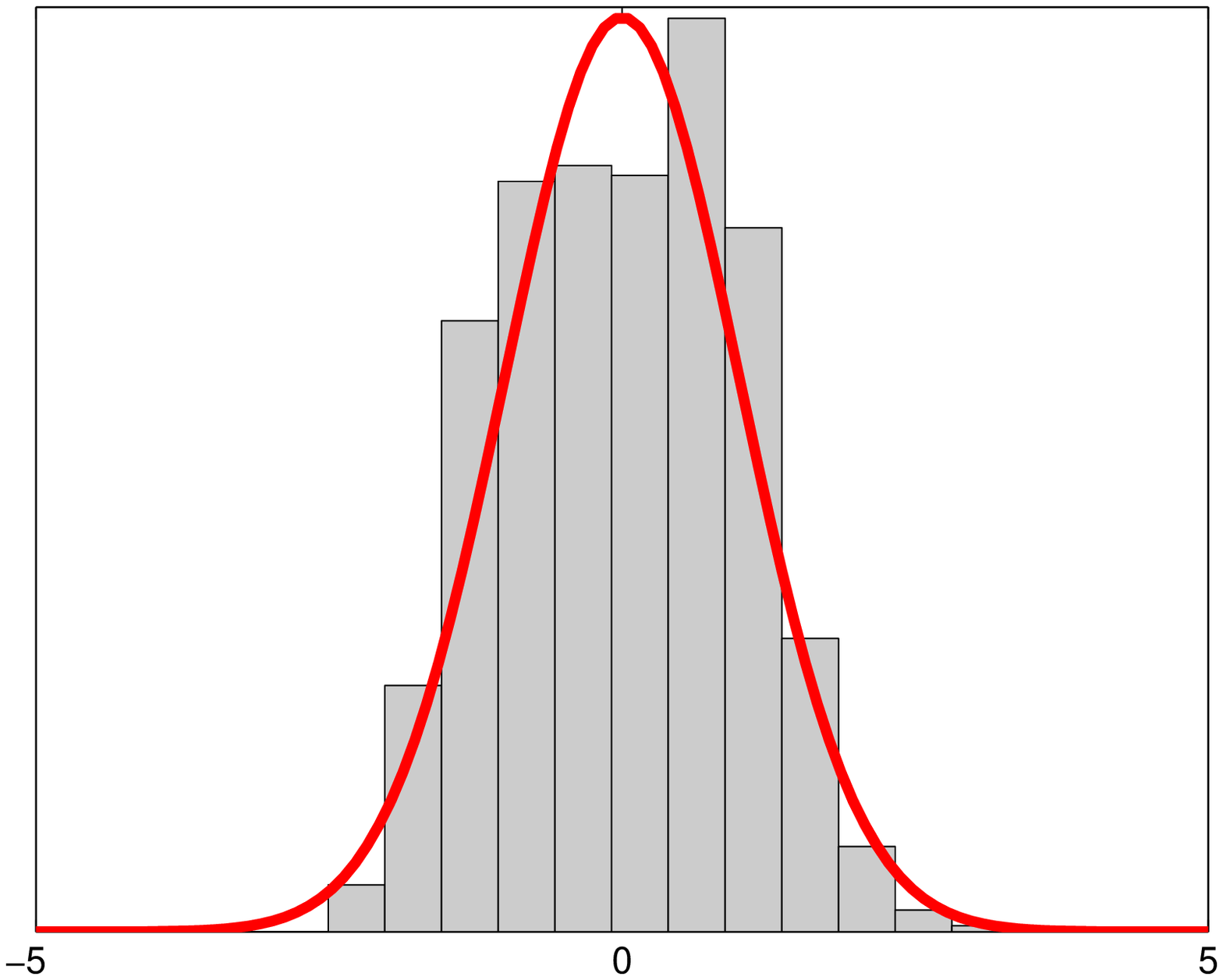}&  \includegraphics[width=0.31\textwidth]{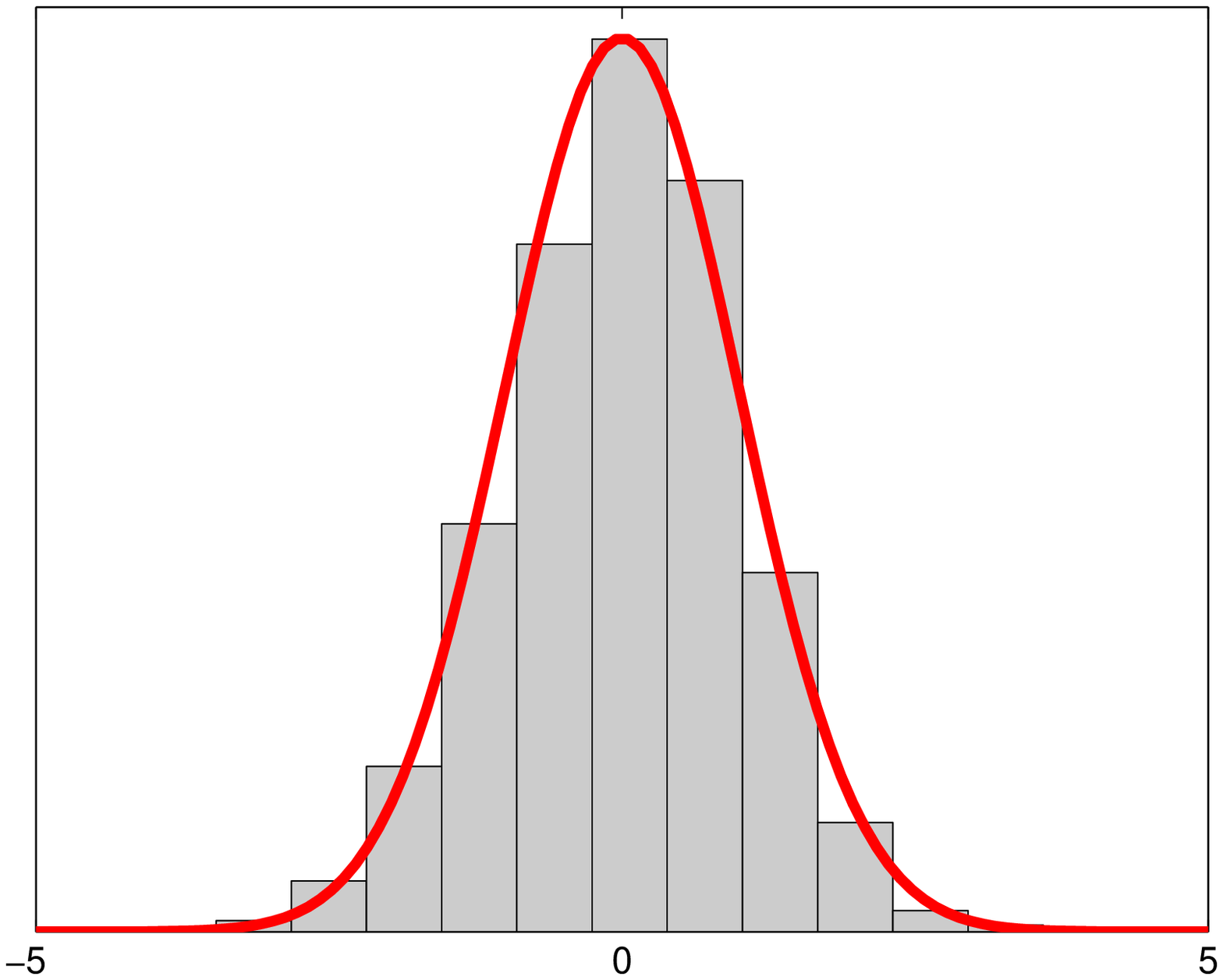} \\
  \includegraphics[width=0.31\textwidth]{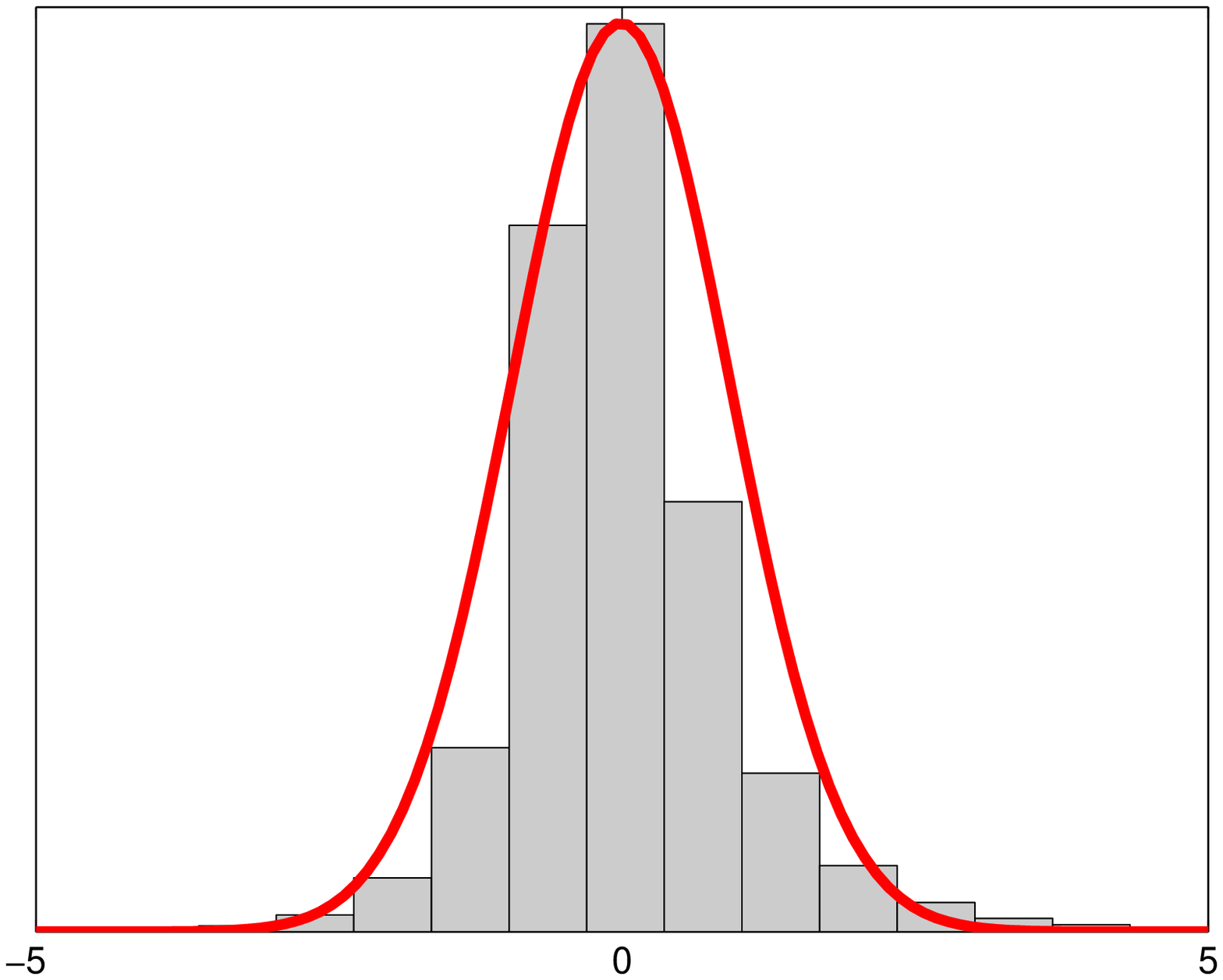}&   \includegraphics[width=0.31\textwidth]{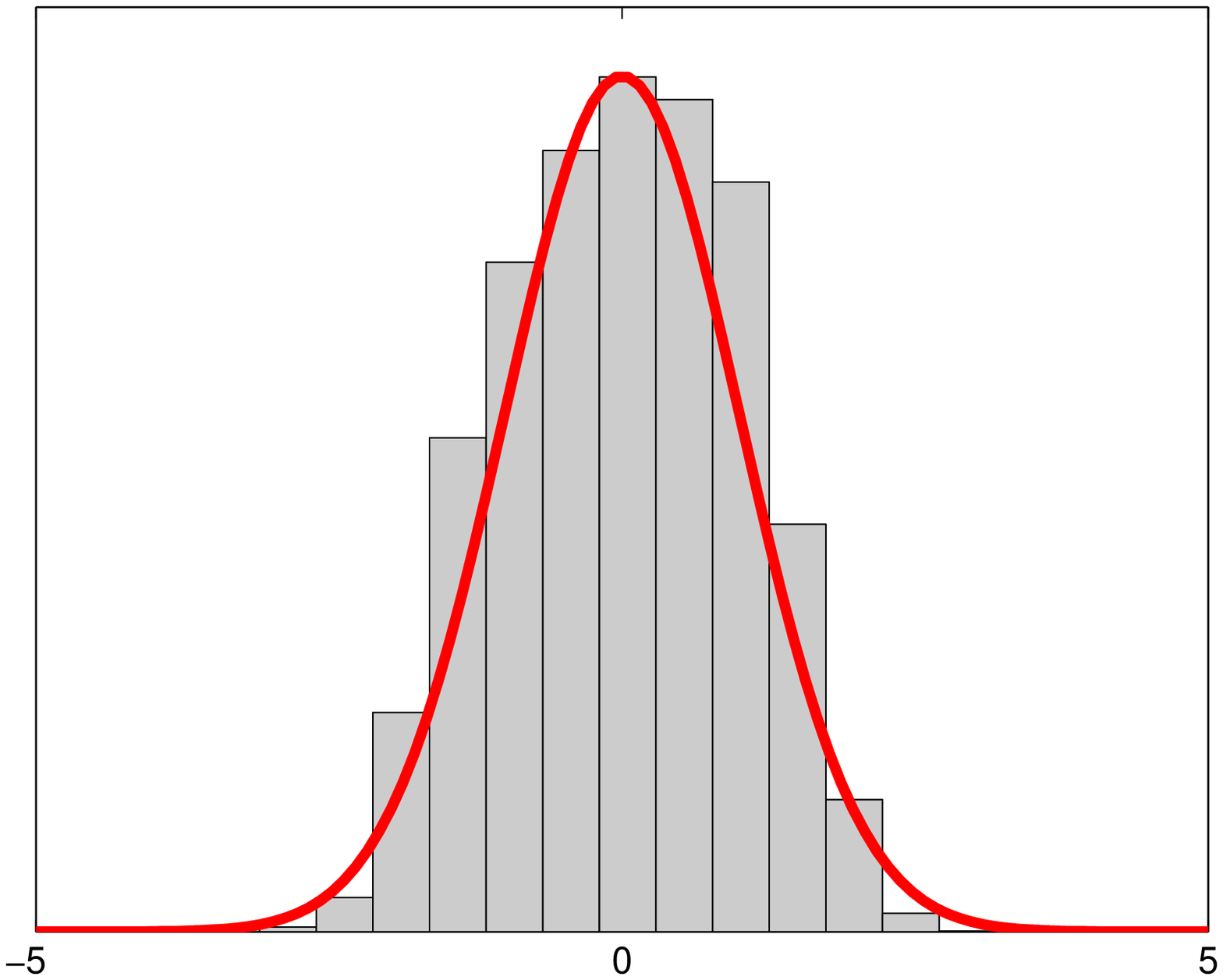}&  \includegraphics[width=0.31\textwidth]{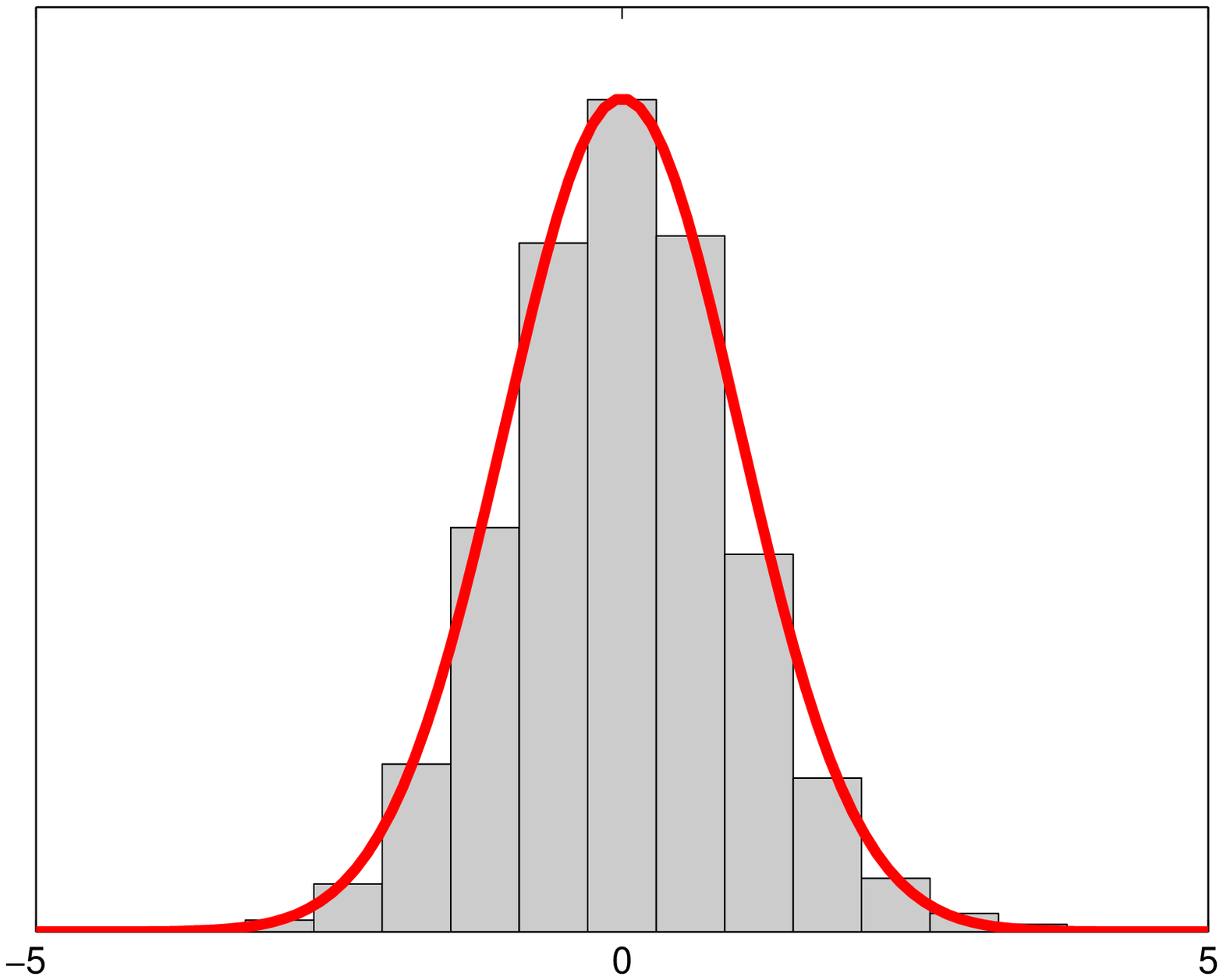} \\
  \includegraphics[width=0.31\textwidth]{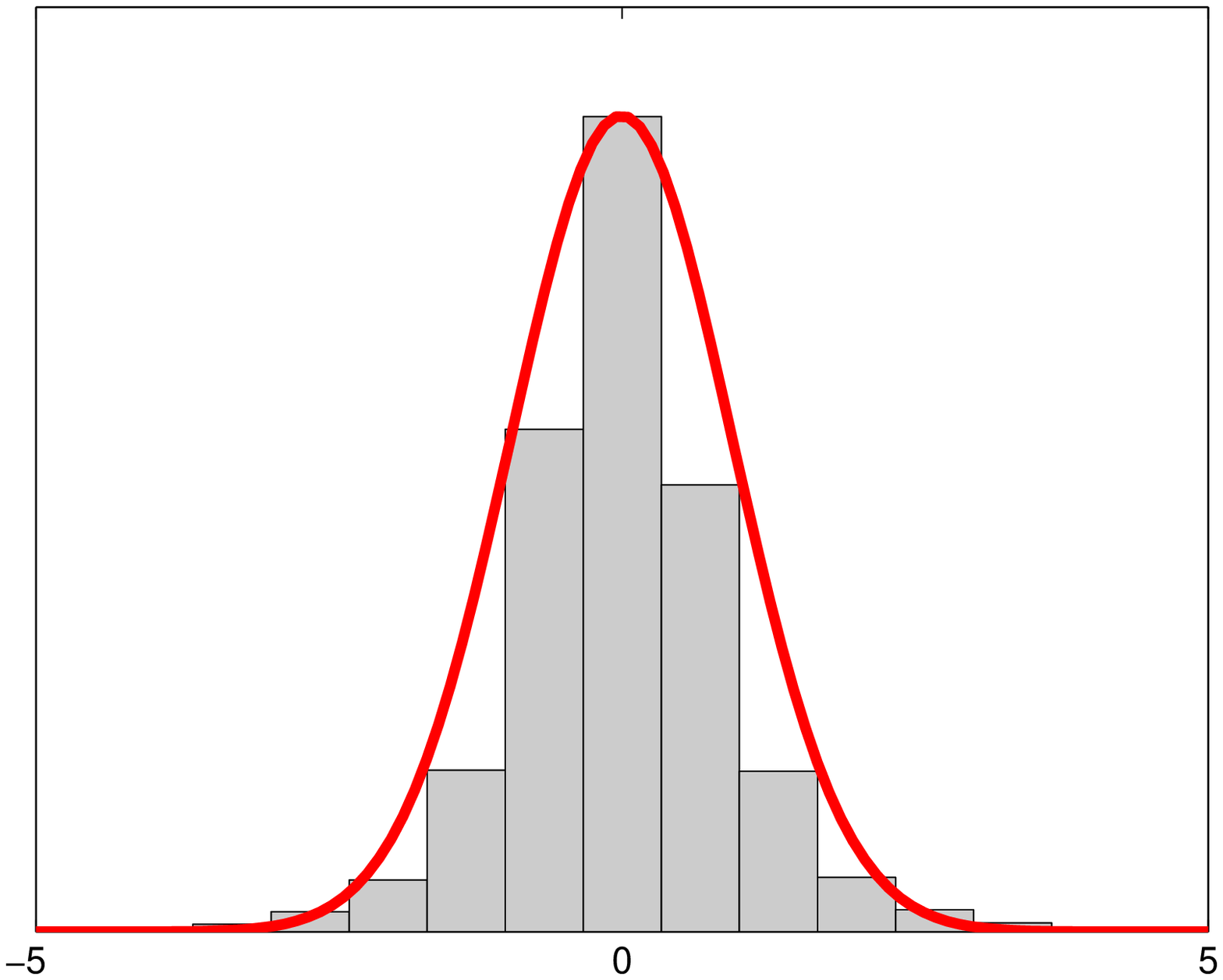}&   \includegraphics[width=0.31\textwidth]{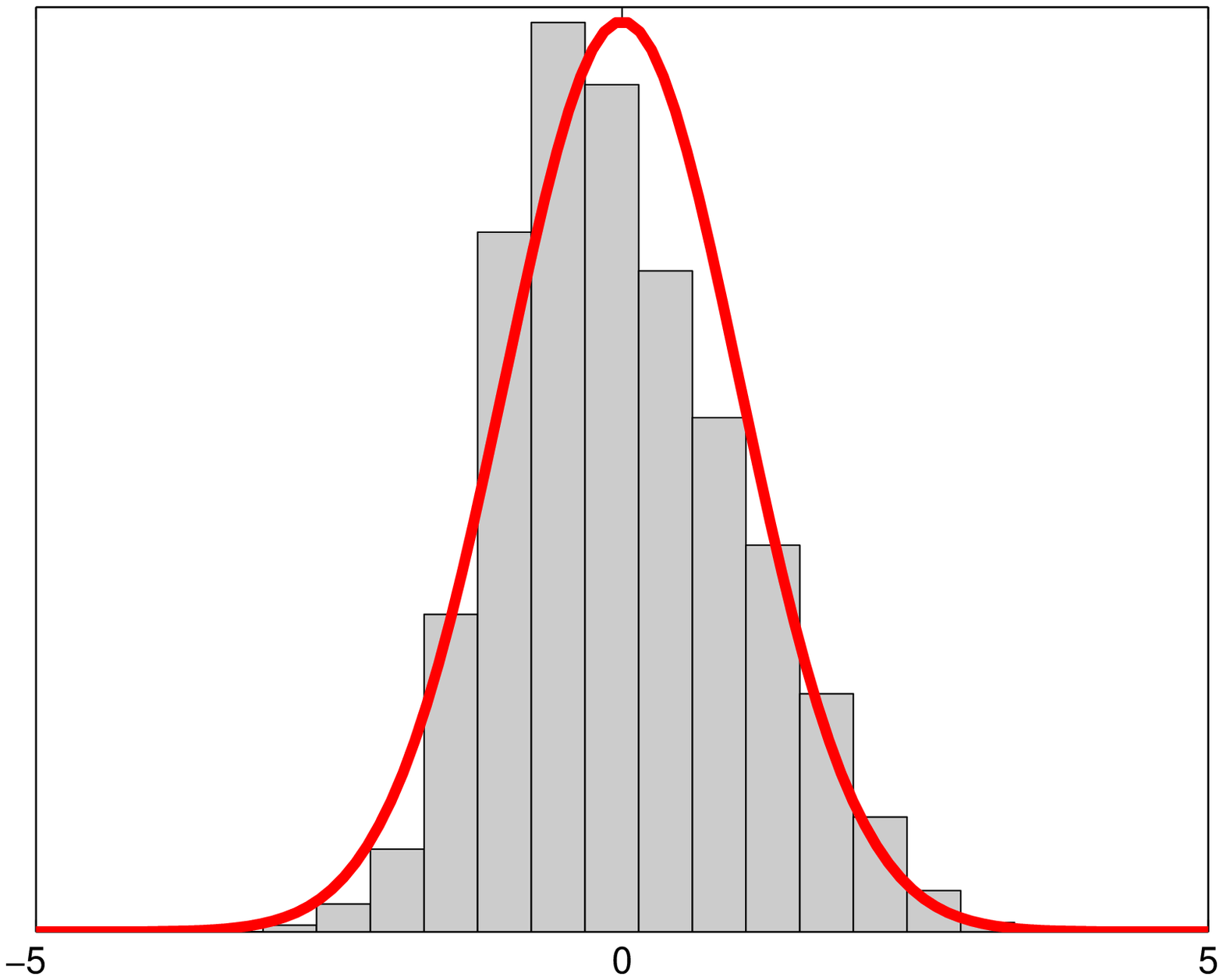}&  \includegraphics[width=0.31\textwidth]{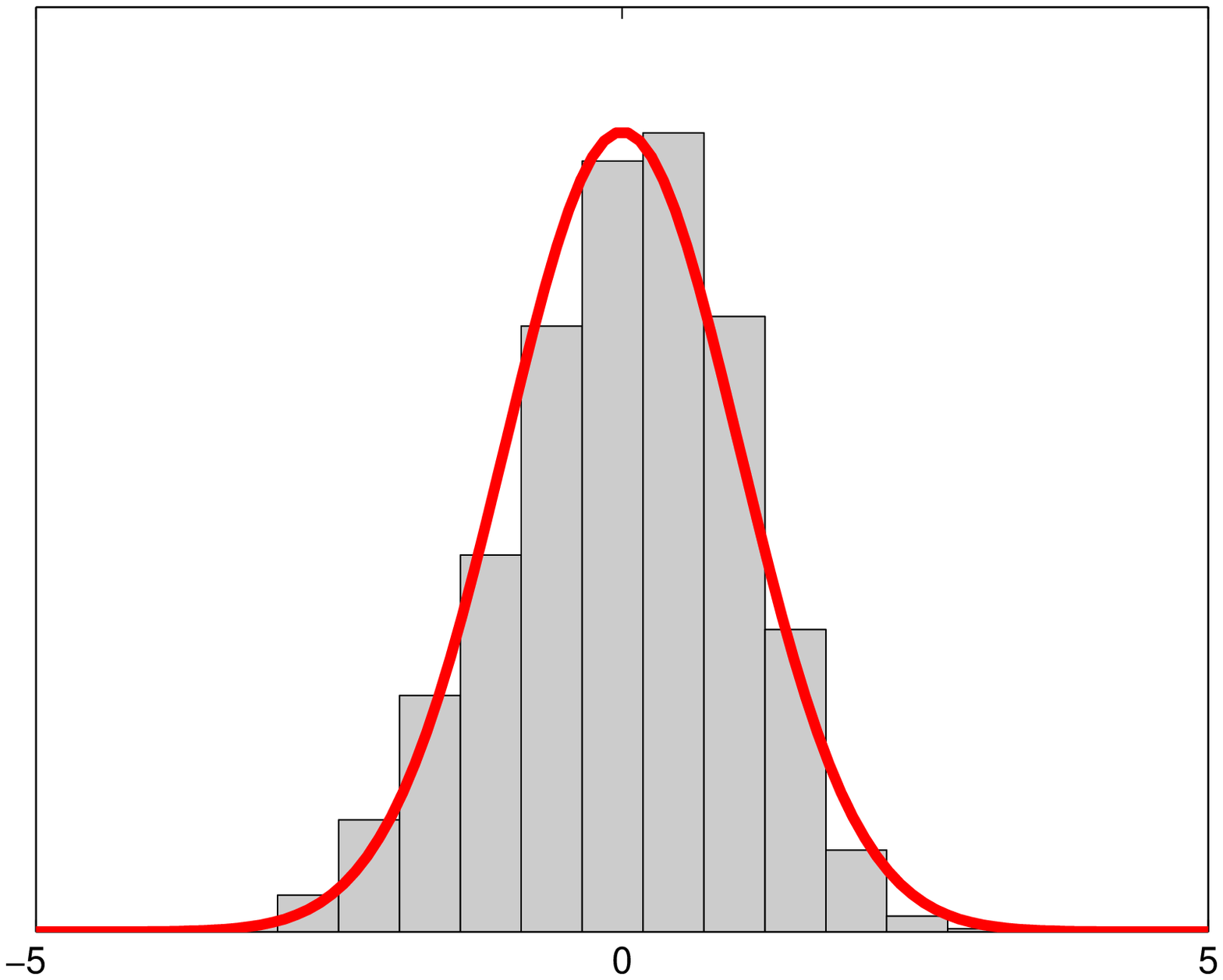} \\
  \includegraphics[width=0.31\textwidth]{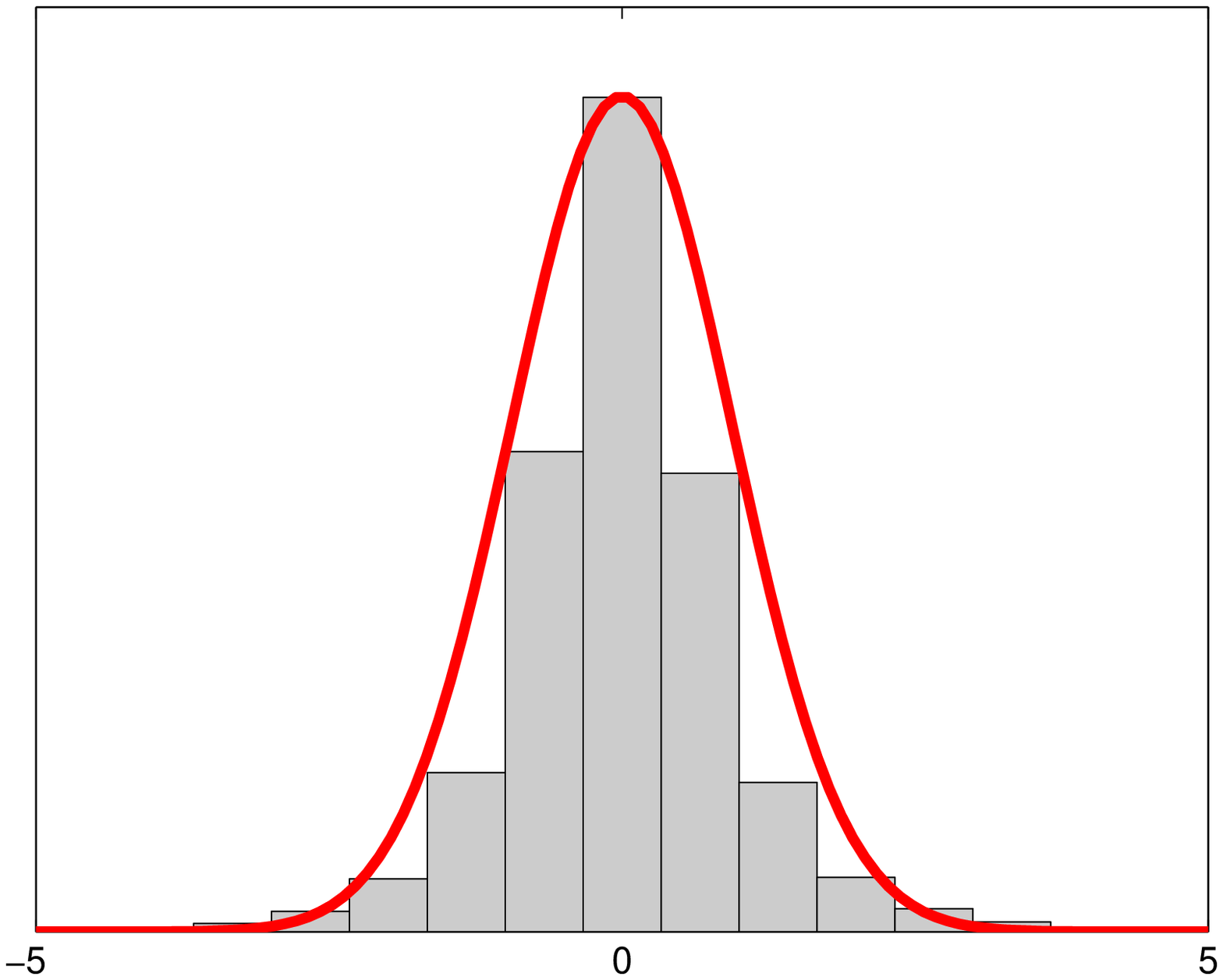}&   \includegraphics[width=0.31\textwidth]{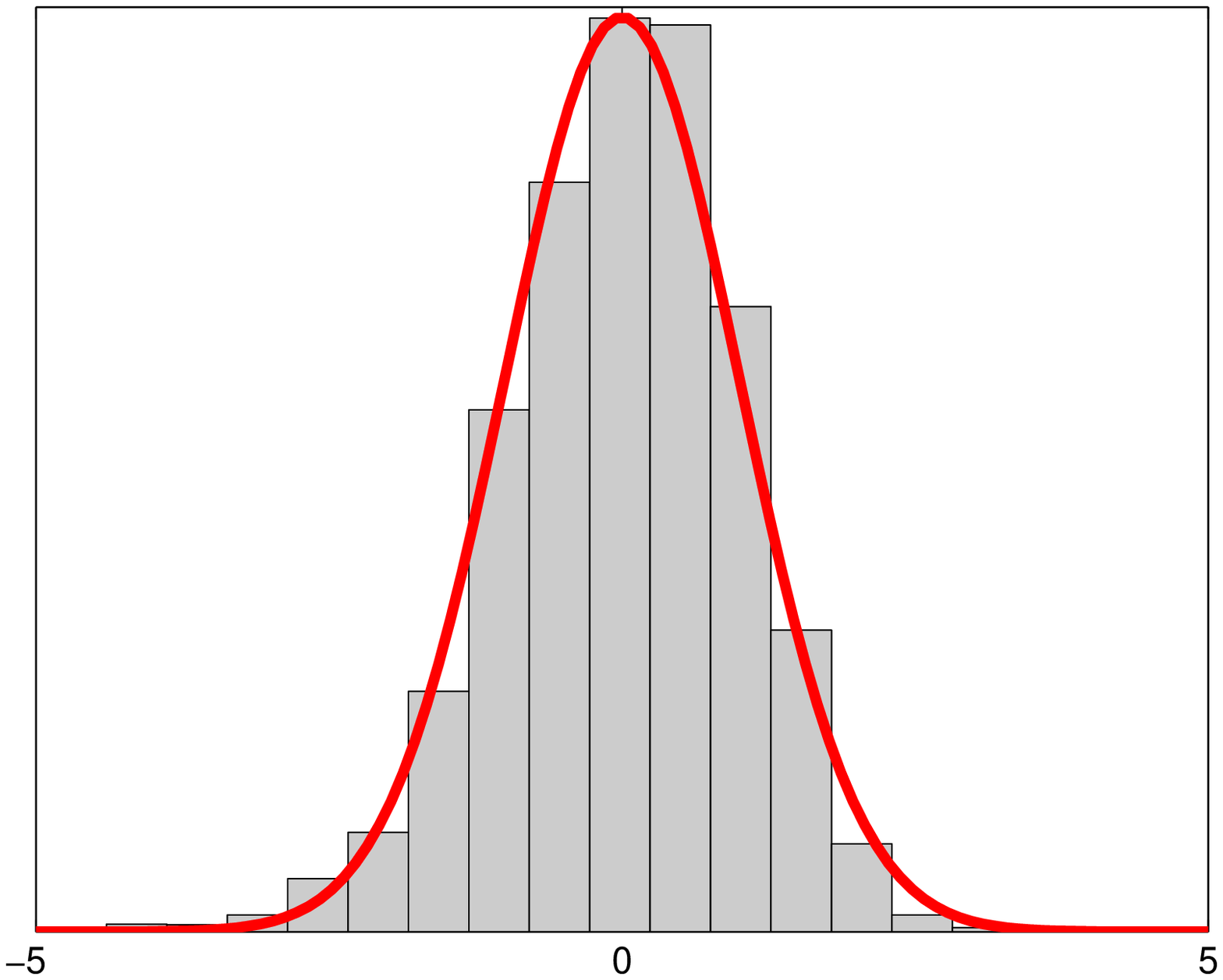}&  \includegraphics[width=0.31\textwidth]{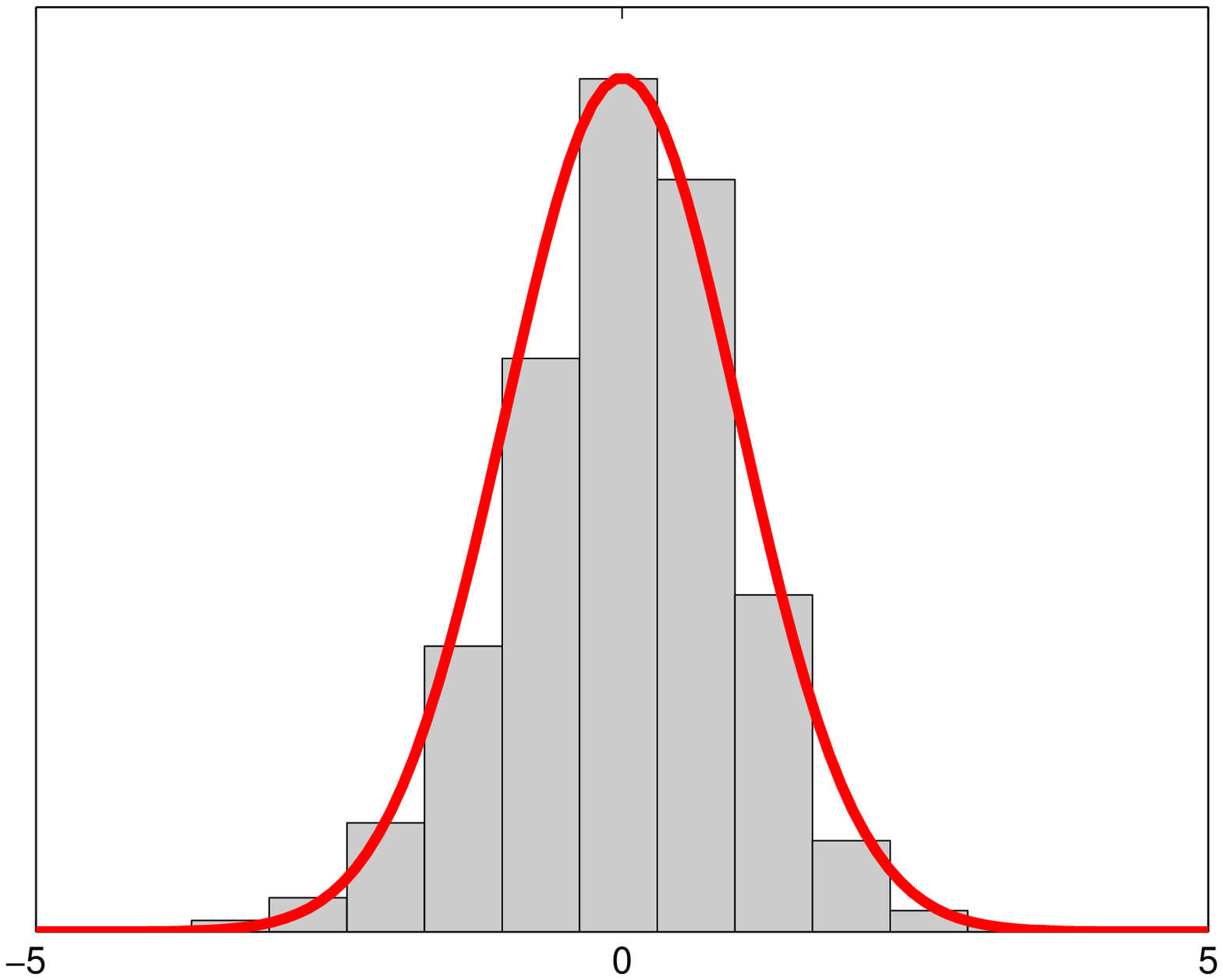} \\
&{\scriptsize MNIST handwritten digit images}&
\end{tabular}
  \caption{\small Centered histograms of $f_{\theta}(X)|\{Y=1\}$ overlayed with the pdf of a fitted Gaussian for randomly drawn $\theta$ vectors ($\theta_i\sim U(-1/2,1/2)$). The columns represent datasets (RCV1 text data \citep{lewis04rcv}, MNIST digit images, and face images \citep{Pham-etal-2002}) and the rows represent multiple random draws. For uniformity we subtracted the empirical mean and divided by the empirical standard deviation. The twelve panels show that even in moderate dimensionality (RCV1: 1000 top words, MNIST digits: 784 pixels, face images: 400 pixels) the assumption that $f_{\theta}(X)|Y$ is normal holds often for randomly drawn $\theta$.}\label{fig:CLT}
\end{figure}

\begin{figure} 
\centering
\begin{tabular}{ccccc}&
{\scriptsize RCV1 text data} & 
 & 
{\scriptsize face images} & \\ &
  \includegraphics[width=0.27\textwidth]{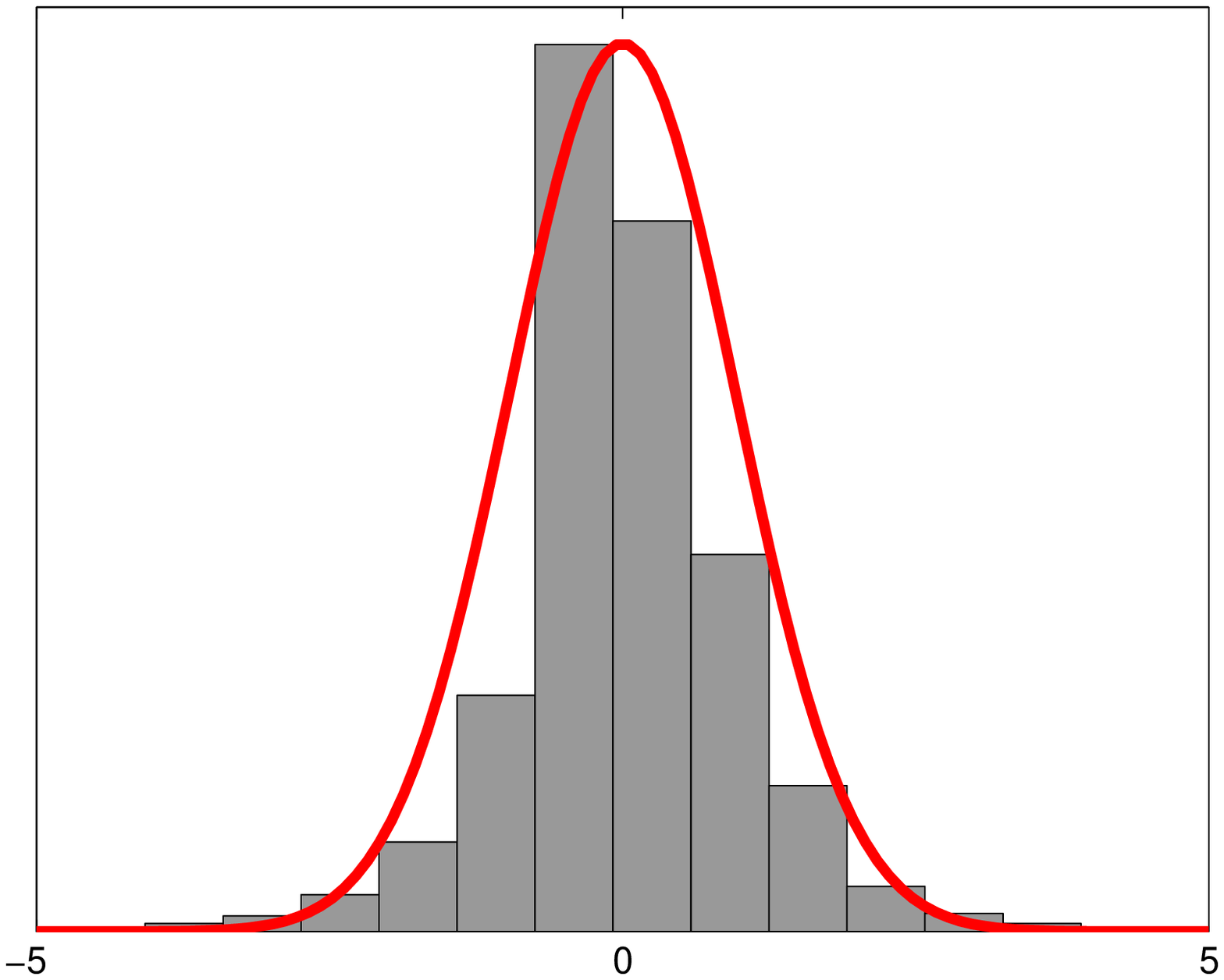}&\hspace{-0.2in}
  \includegraphics[width=0.27\textwidth]{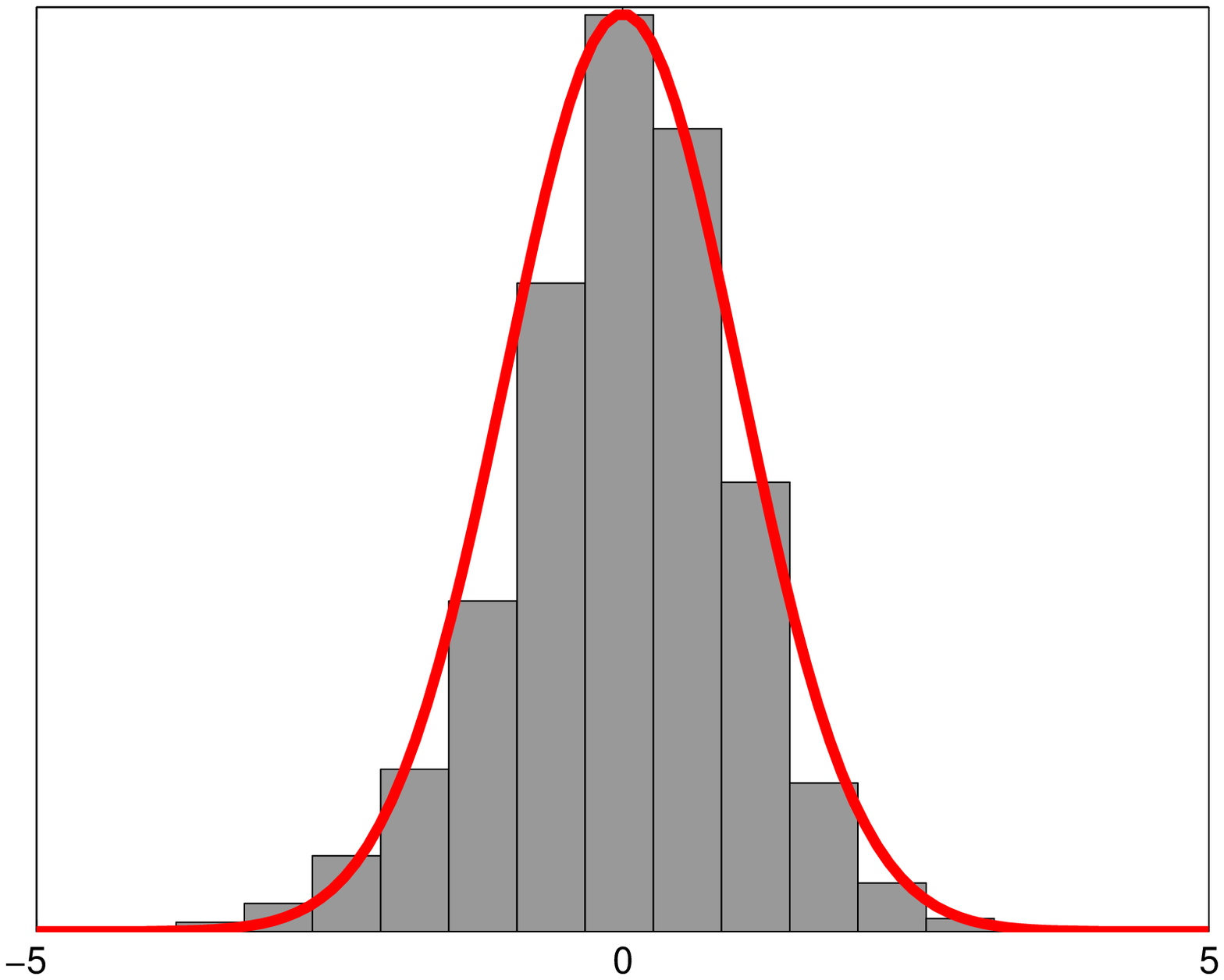}&\hspace{-0.2in}
  \includegraphics[width=0.27\textwidth]{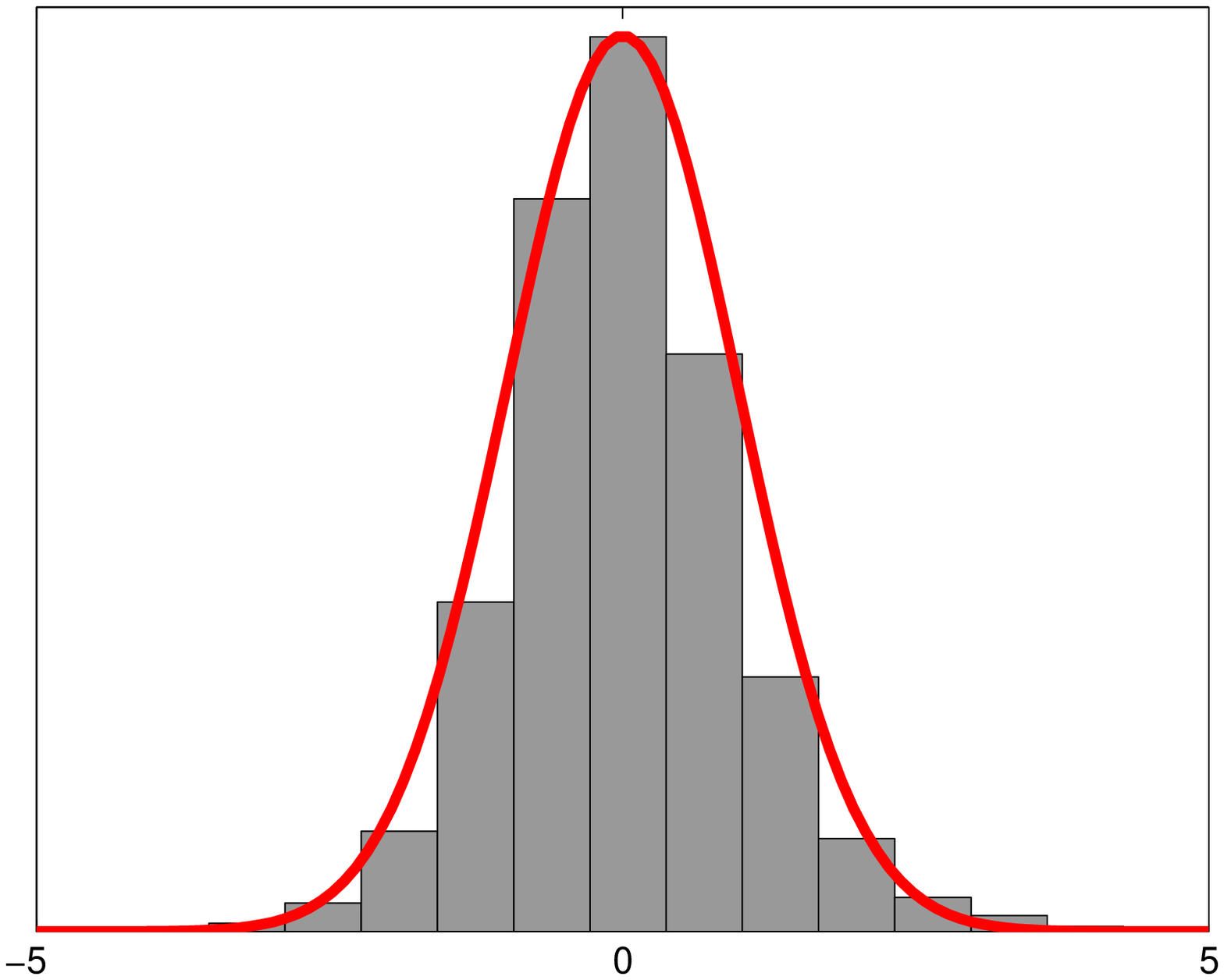}&  \begin{sideways} \scriptsize \hspace{0.4in}  Fisher's LDA\end{sideways}\\
  \begin{sideways}\scriptsize \hspace{0.4in}log. regression\end{sideways}&
  \includegraphics[width=0.27\textwidth]{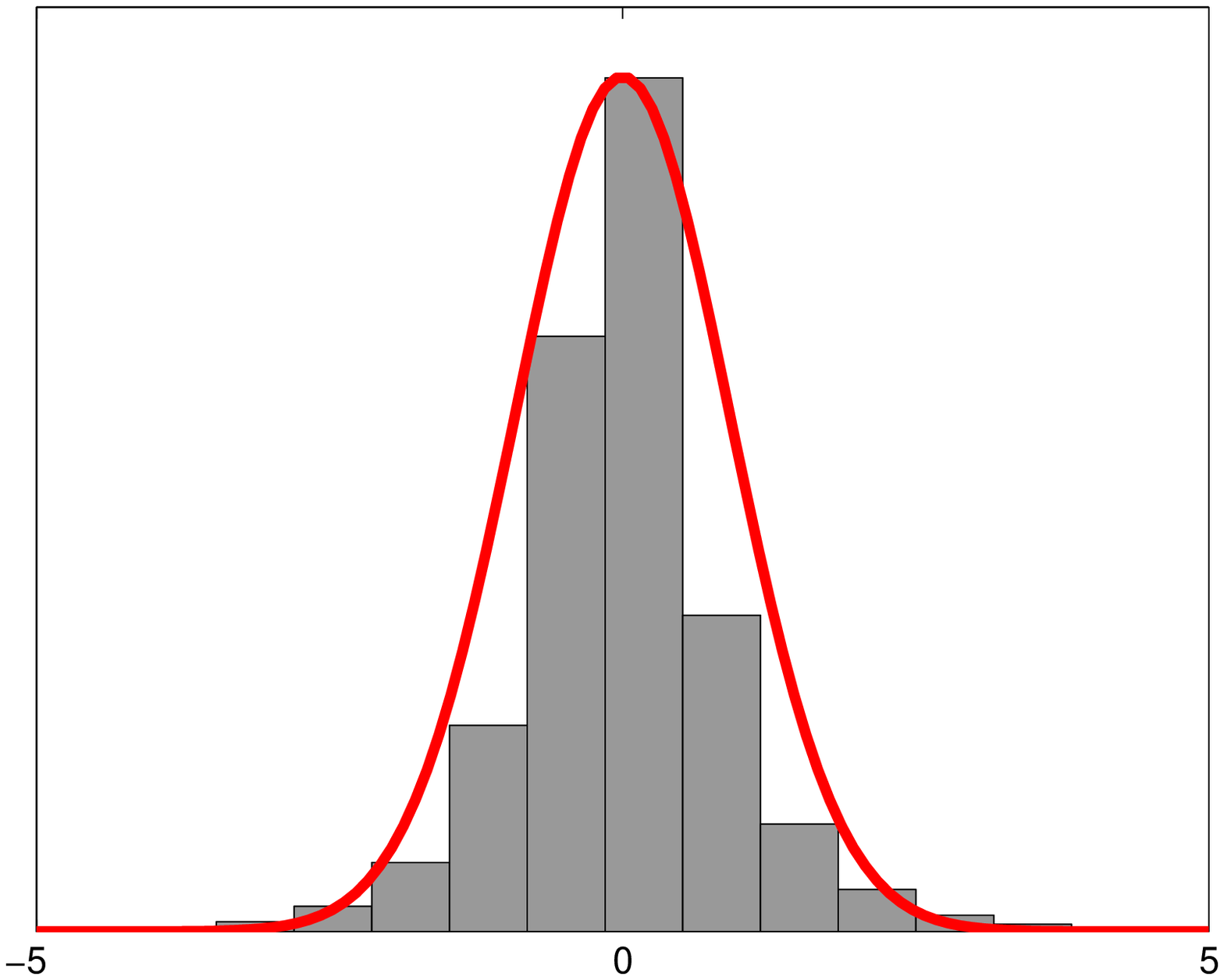}&\hspace{-0.2in}
  \includegraphics[width=0.27\textwidth]{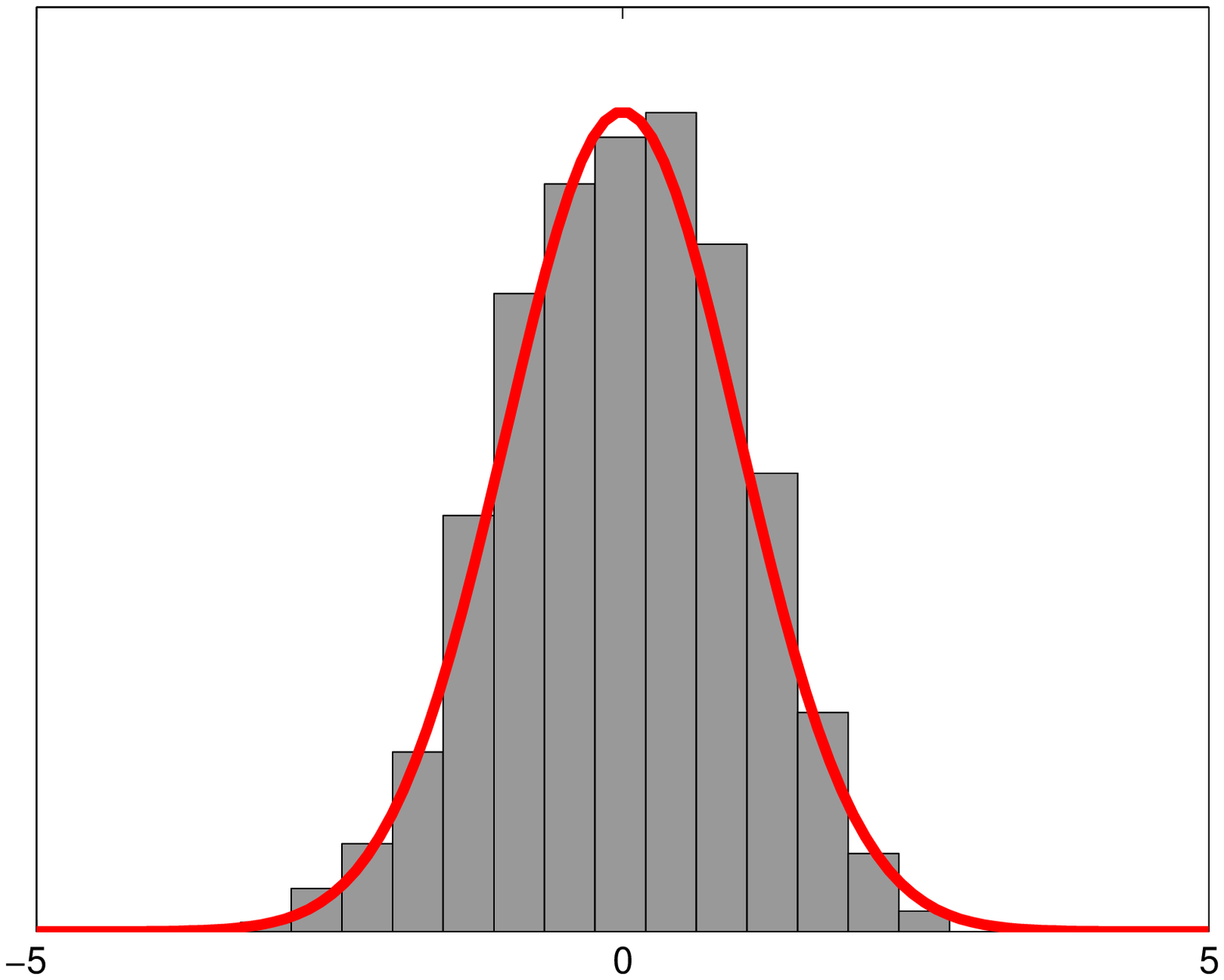}&\hspace{-0.2in}
  \includegraphics[width=0.27\textwidth]{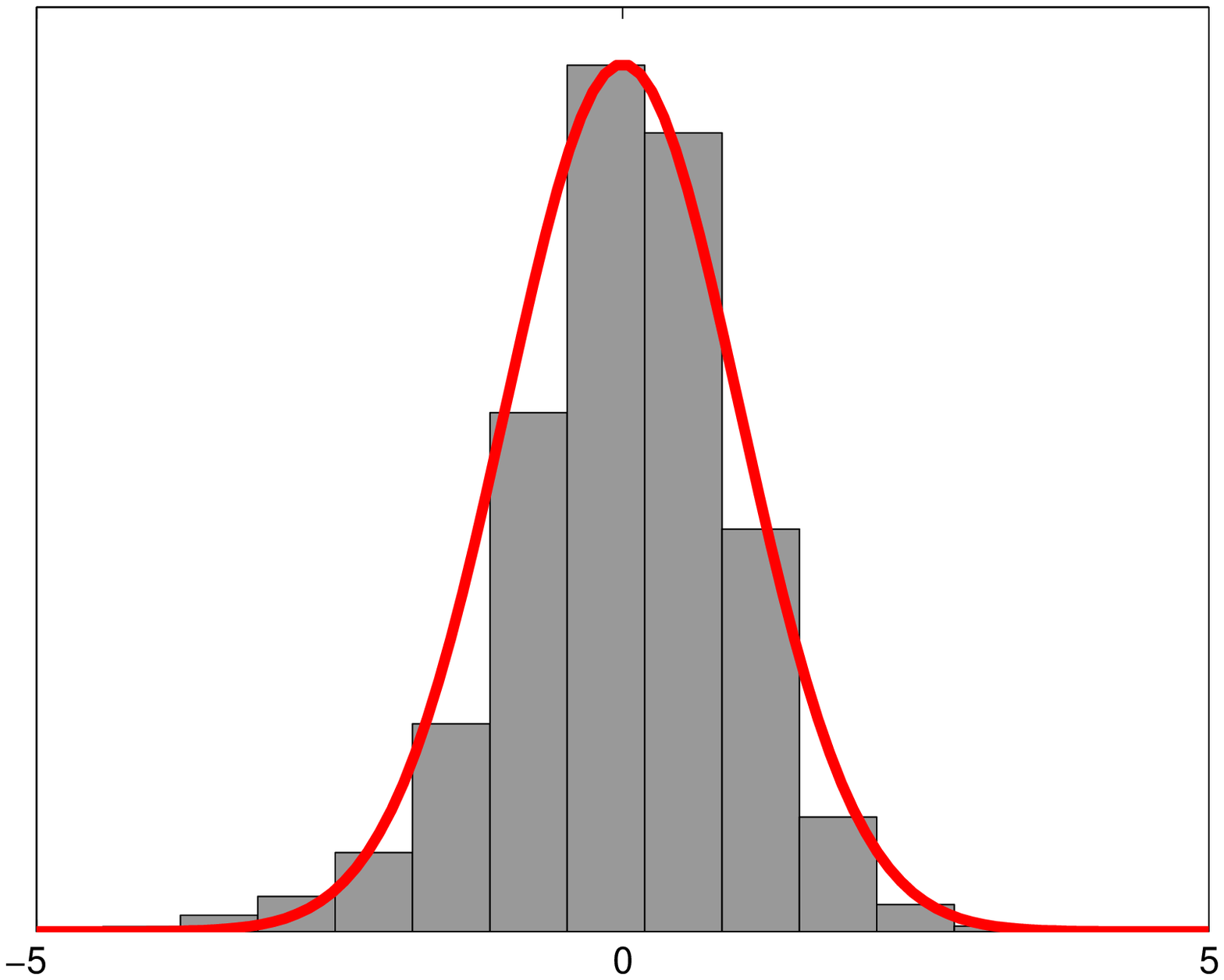}&\\
&
  \includegraphics[width=0.27\textwidth]{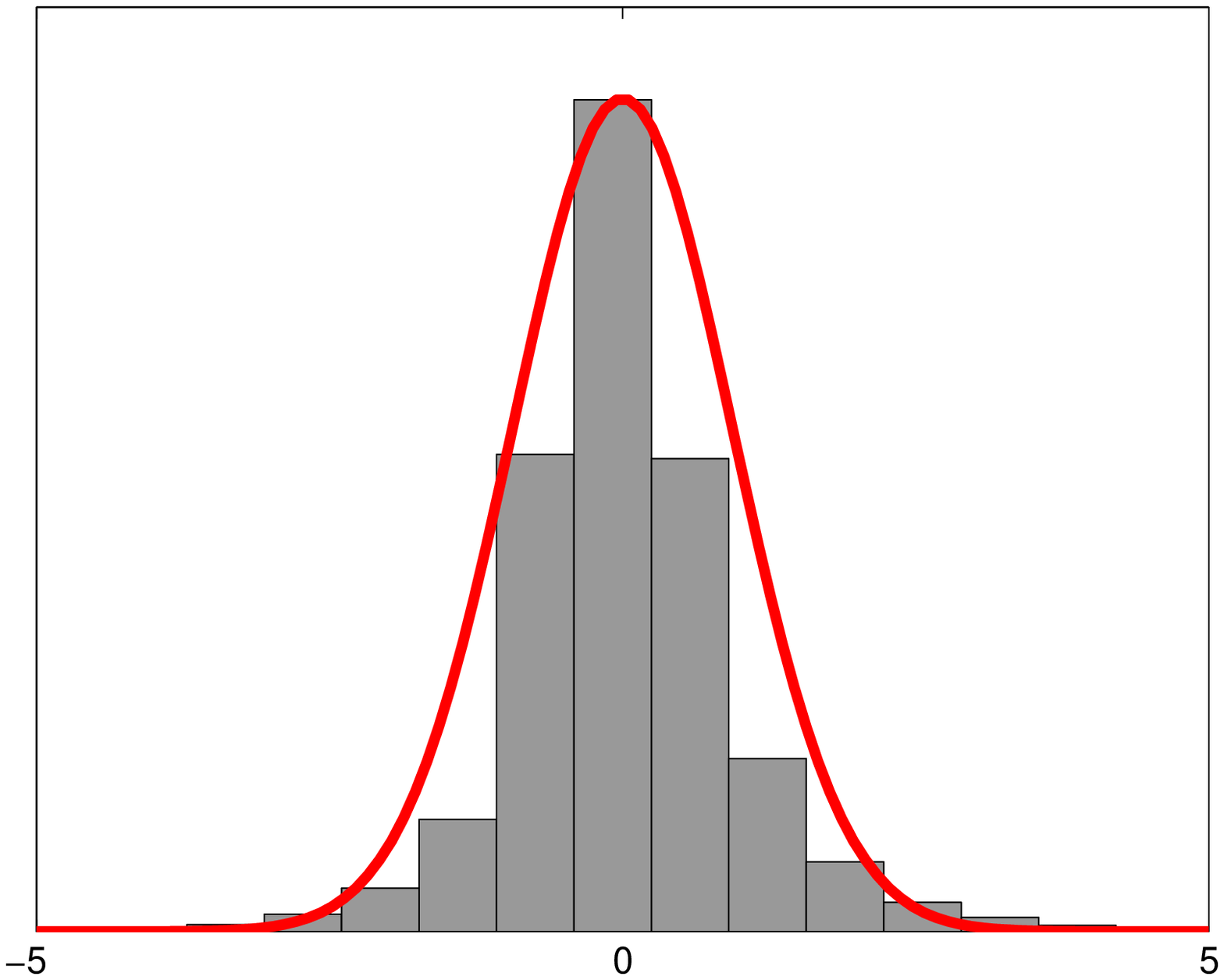}&\hspace{-0.2in}
  \includegraphics[width=0.27\textwidth]{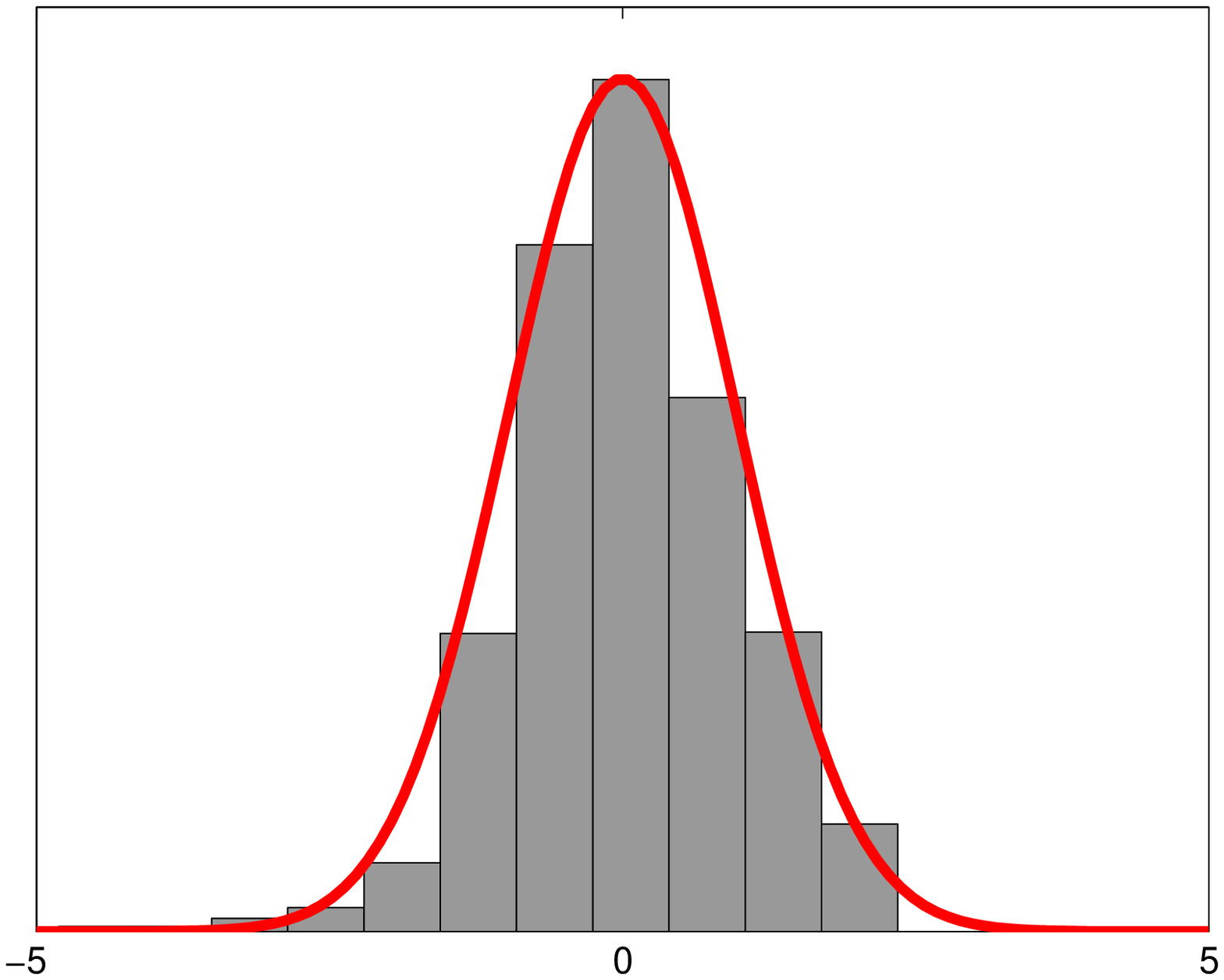}&\hspace{-0.2in}
  \includegraphics[width=0.27\textwidth]{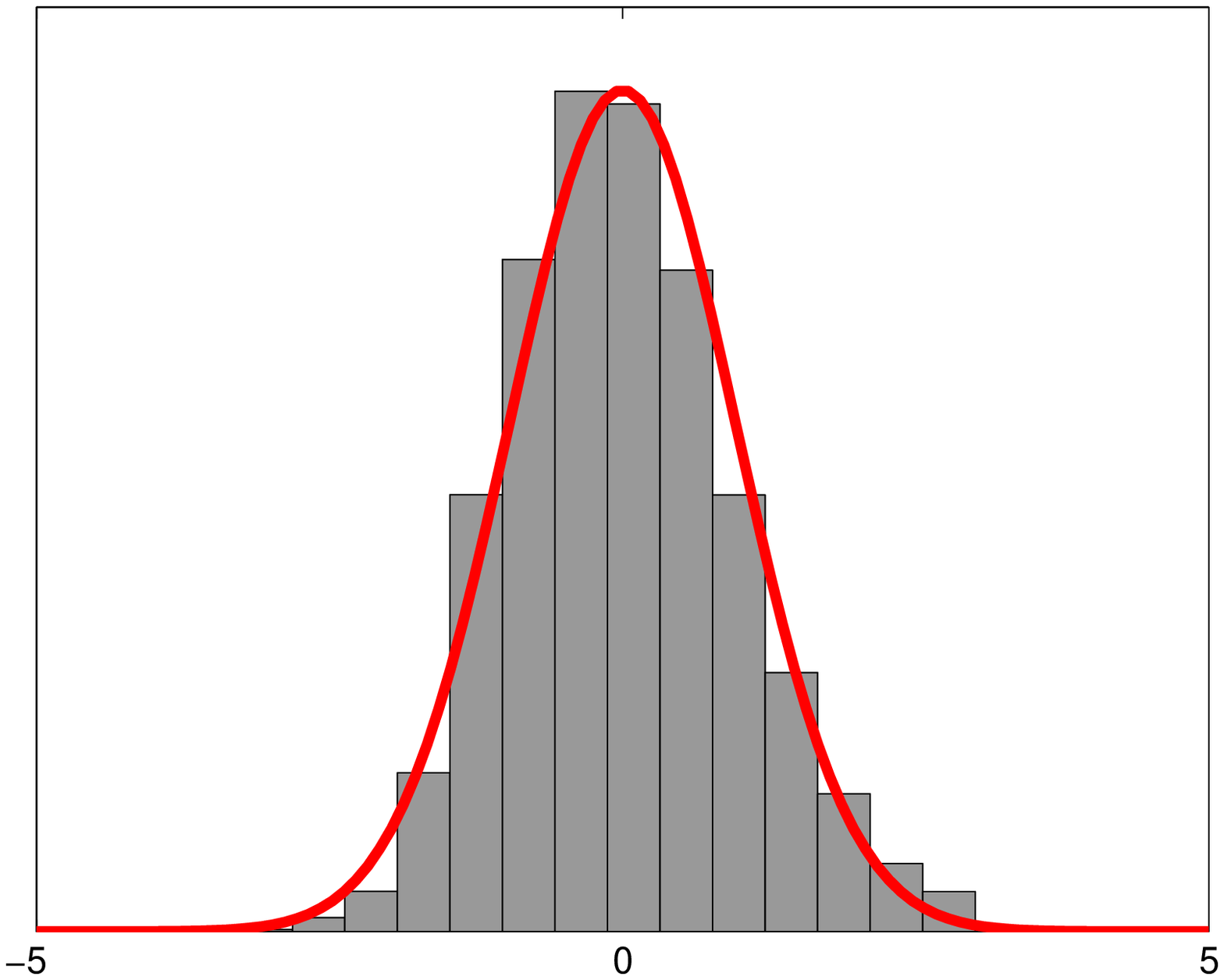}&  \begin{sideways}\scriptsize log. regression ($l_2$ regularized) \end{sideways}\\
  \begin{sideways}\scriptsize log. regression ($l_1$ regularized) \end{sideways}&
  \includegraphics[width=0.27\textwidth]{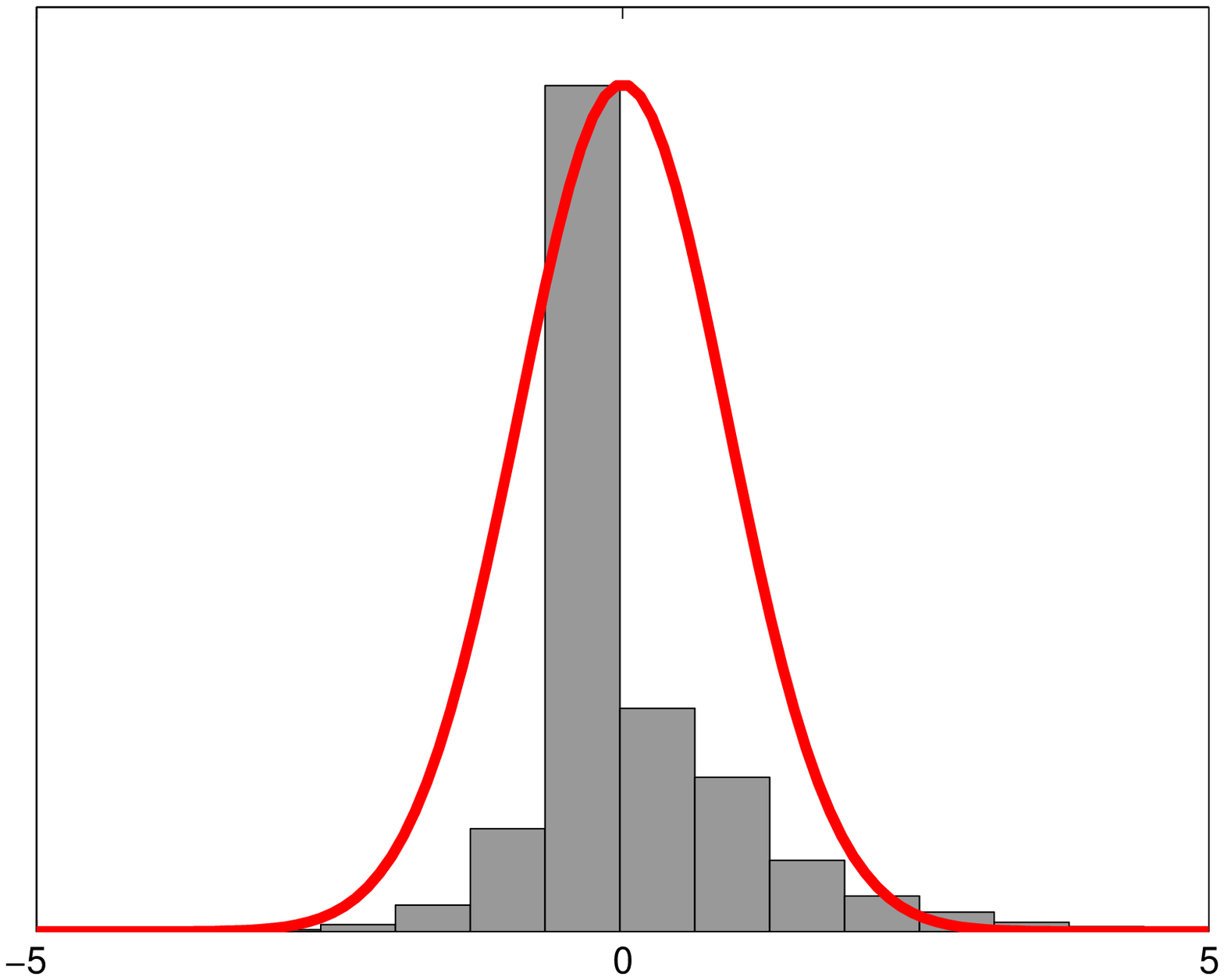}&\hspace{-0.2in}
  \includegraphics[width=0.27\textwidth]{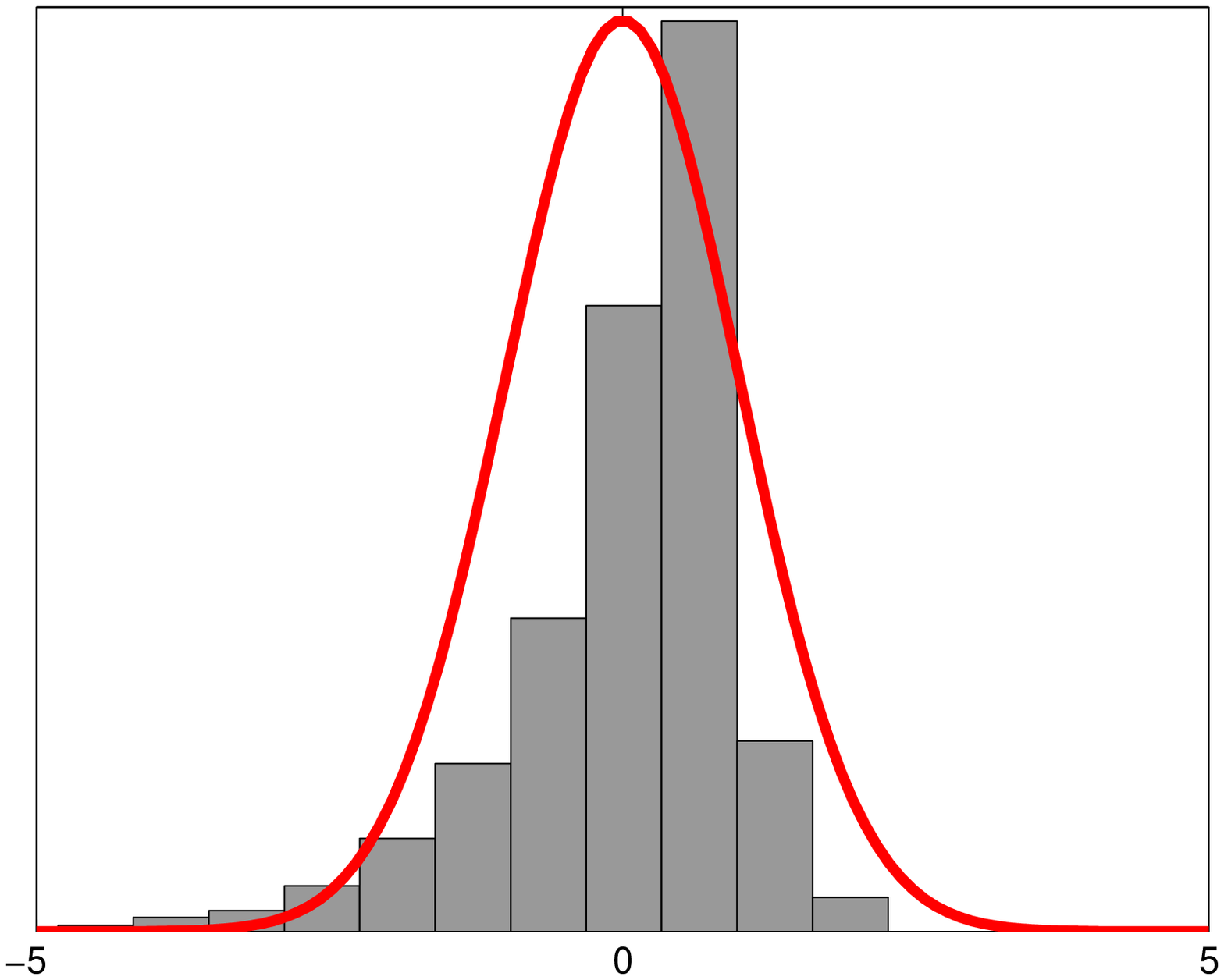}&\hspace{-0.2in}
  \includegraphics[width=0.27\textwidth]{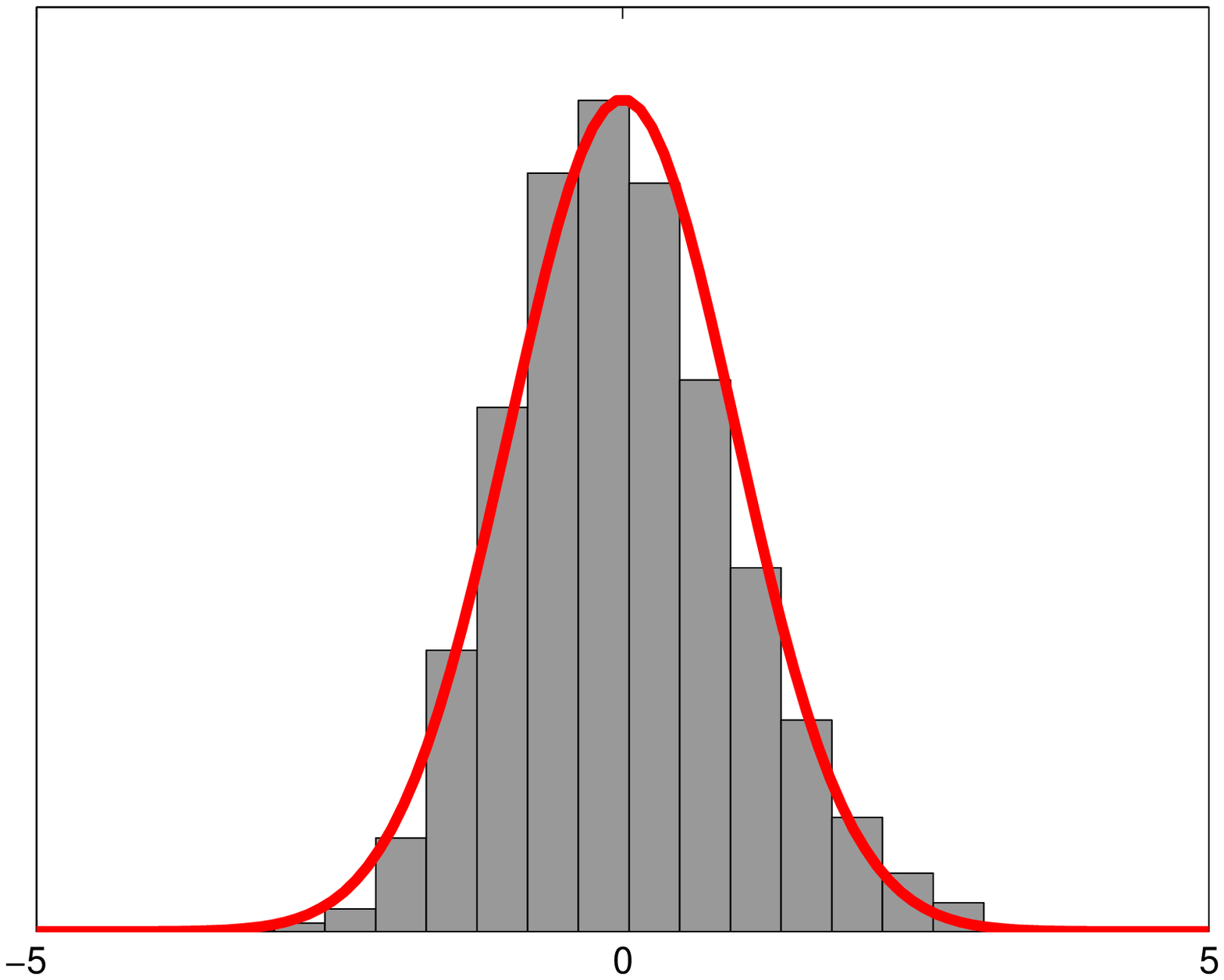}\\
&&{\scriptsize MNIST handwritten digit images}&&
\end{tabular}
  \caption{\small Centered histograms of $f_{\theta}(X)|\{Y=1\}$ overlayed with the pdf of a fitted Gaussian for multiple $\theta$ vectors (four rows: Fisher's LDA, logistic regression, $l_2$ regularized logistic regression, and $l_1$ regularized logistic regression-all regularization parameters were selected by cross validation) and datasets (columns: RCV1 text data \citep{lewis04rcv}, MNIST digit images, and face images \citep{Pham-etal-2002}). For uniformity we subtracted the empirical mean and divided by the empirical standard deviation. The twelve panels show that even in moderate dimensionality (RCV1: 1000 top words, MNIST digits: 784 pixels, face images: 400 pixels) the assumption that $f_{\theta}(X)|Y$ is normal holds well for fitted $\theta$ values (except perhaps for $l_1$ regularization in the last row which promotes sparse $\theta$).}\label{fig:CLT2}
\end{figure}

From a theoretical standpoint normality may be argued using a central limit theorem. We examine below several progressingly more general central limit theorems and discuss whether these theorems are likely to hold in practice for high dimensional data. The original central limit theorem states that $\sum_{i=1}^d Z_i$ is approximately normal for large $d$ if $Z_i$ are iid.
\begin{prop}[de-Moivre]
  If $Z_i, i\in \mathbb{N}$ are iid with expectation $\mu$ and   variance $\sigma^2$ and $\bar{Z}_d=d^{-1}\sum_{i=1}^d Z_i$ then we   have the following convergence in distribution
  \[ \sqrt{d}(\bar{Z}_d -\mu)/\sigma \tood N(0,1) \quad \text{as }   d\to\infty.\]
\end{prop}
As a result, the quantity $\sum_{i=1}^d Z_i$ (which is a linear transformation of $\sqrt{d}(\bar{Z}_d -\mu)/\sigma$) is approximately normal for large $d$. This relatively restricted theorem is unlikely to hold in most practical cases as the data dimensions are often not iid.

A more general CLT does not require the summands $Z_i$ to be identically distributed.
\begin{prop}[Lindberg]
  For $Z_i, i\in \mathbb{N}$ independent with expectation $\mu_i$ and   variance $\sigma^2_i$, and denoting $s_d^2=\sum_{i=1}^d \sigma_i^2$,   we have the following convergence in distribution as $d\to\infty$
  \[ s_d^{-1} \sum_{i=1}^d (Z_i - \mu_i) \tood N(0,1) \] if the   following condition holds for every $\epsilon>0$
  \begin{align}\lim_{d\to\infty} s_d^{-2} \sum_{i=1}^d \E     (Z_i-\mu_i)^2 1_{\{|X_i-\mu_i|>\epsilon s_d\}}=0.
  \label{eq:LindbergCondition}\end{align} 
\end{prop}
This CLT is more general as it only requires that the data dimensions be independent. The condition \eqref{eq:LindbergCondition} is relatively mild and specifies that contributions of each of the $Z_i$ to the variance $s_d$ should not dominate it. Nevertheless, the Lindberg CLT is still inapplicable for dependent data dimensions.

More general CLTs replace the condition that $Z_i, i\in\mathbb{N}$ be independent with the notion of $m(k)$-dependence.
\begin{defn}
  The random variables $Z_i,i\in\mathbb{N}$ are said to be   $m(k)$-dependent if whenever $s-r>m(k)$ the two sets   $\{Z_1,\ldots,Z_r\}$, $\{Z_s,\ldots,Z_k\}$ are independent.
\end{defn}
An early CLT for $m(k)$-dependent RVs is \citep{Hoeffding1948}. Below is a slightly weakened version of the CLT in  \citep{Berk1973}.
\begin{prop}[Berk] \label{prop:Berk} For each $k\in\mathbb{N}$ let   $d(k)$ and $m(k)$ be increasing sequences and suppose that   $Z_1^{(k)},\ldots,Z_{d(k)}^{(k)}$ is an $m(k)$-dependent sequence of   random variables. If
  \begin{enumerate}
  \item $\E|Z_i^{(k)}|^2 \leq M$ for all $i$ and $k$
  \item $\Var (Z_{i+1}^{(k)} +\ldots+ Z_j^{(k)})\leq (j-i)K$ for all     $i,j,k$
  \item $\lim_{k\to\infty} \Var (Z_{1}^{(k)} +\ldots+     Z_{d(k)}^{(k)})/d(k)$ exists and is non-zero
  \item $\lim_{k\to\infty} m^2(k)/d(k)=0$
  \end{enumerate}
  then $\frac{\sum_{i=1}^{d(k)} Z_i^{(k)}}{\sqrt{d(k)}}$ is   asymptotically normal as $k\to\infty$.
\end{prop}

Proposition~\ref{prop:Berk} states that under mild conditions the sum of $m(k)$-dependent RVs is asymptotically normal. If $m(k)$ is a constant i.e., $m(k)=m$, $m(k)$-dependence implies that a $Z_i$ may only depend on its neighboring dimensions. Or in other words, dimensions that are removed from each other are independent. The full power of Proposition~\ref{prop:Berk} is invoked when $m(k)$ grows with $k$ relaxing the independence restriction as the dimensionality grows. Intuitively, the dependency of the summands is not fixed to a certain order, but it cannot grow too rapidly. 

A more realistic variation of $m(k)$ dependence where the dependency of each variable is specified using a dependency graph (rather than each dimension depends on neighboring dimensions) is advocated in a number of papers, including the following recent result by \cite{rinott1994}. 
\begin{defn} \label{def:gd}
A graph  $\mathcal{G} = \left(\mathcal{V},\mathcal{E}\right)$  indexing random variables is called a dependency graph if for any pair of disjoint subsets of $\mathcal{V}$, $A_1$ and $A_2$ such that no edge in $\mathcal{E}$ has one endpoint in $A_1$ and the other in $A_2$, we have independence between $\{Z_i: i \in A_1\}$ and $\{Z_i: i \in A_2\}$. The degree $d(v)$ of a vertex is the number of edges connected to it and the maximal degree is $\max_{v \in \mathcal{V}} d(v)$. 
\end{defn}
\begin{prop}[Rinott] \label{lab:prop1}
Let $Z_1, \ldots , Z_n$ be random variables having a dependency graph whose maximal degree is strictly less than $D$, satisfying $|Z_i-\E Z_i| \leq B$ a.s., $\forall i$, $\E (\sum_{i=1}^{n} Z_i) = \lambda$ and $\Var (\sum_{i=1}^{n} Z_i) = \sigma^2 >0$, Then for any $w \in \mathbb{R}$,
\[
\left|P\left(\frac{\sum_{i=1}^{n}Z_i-\lambda}{\sigma} \leq w \right) -\Phi (w) \right	| \leq \frac{1}{\sigma} \left(\frac{1}{\sqrt{2\pi}}DB + 16 \left(\frac{n}{\sigma^2}\right)^{1/2} D^{3/2}B^2 +10 \left(\frac{n}{\sigma^2}\right)D^2B^3	\right)
\] 
\end{prop}
The above theorem states a stronger result than convergence in distribution to a Gaussian in that it states a uniform rate of convergence of the CDF. Such results are known in the literature as Berry Essen bounds. When $D$ and $B$ are bounded and $\Var (\sum_{i=1}^{n} Z_i)=O(n)$ it yields a CLT with an optimal convergence rate of $n^{-1/2}$.

The question of whether the above CLTs apply in practice is a delicate one. For text one can argue that the appearance of a word depends on some words but is independent of other words. Similarly for images it is plausible to say that the brightness of a pixel is independent of pixels that are spatially far removed from it. In practice one needs to verify the normality assumption empirically, which is simple to do by comparing the empirical histogram of $f_{\theta}(X)$ with that of a fitted mixture of Gaussians. As the figures above indicate this holds for text and image data for most values of $\theta$, assuming it
 is not sparse. 

\subsection{Unsupervised Consistency of $\hat R_n(\theta)$}  \label{sec:consistency} 

We start with proving identifiability of the maximum likelihood estimator (MLE) for a mixture of two Gaussians with known ordering of mixture proportions. Invoking classical consistency results in conjunction with identifiability we show consistency of the MLE estimator for $(\mu,\sigma)$ parameterizing the distribution of $f_{\theta}(X)|Y$. As a result consistency of the estimator $\hat R_n(\theta)$ follows. 

\begin{defn}
A parametric family $\{p_{\alpha}:\alpha\in A\}$ is identifiable   when $p_{\alpha}(x)= p_{\alpha'}(x), \forall x$ implies   $\alpha=\alpha'$.
\end{defn}

\begin{prop} \label{prop:identifiability} Assuming known label marginals with $p(Y=1)\neq p(Y=-1)$ the Gaussian mixture family
\begin{align} \label{eq:GaussianMix}
    p_{\mu,\sigma}(x) = p(y=1) N(x\,; \mu_1,\sigma_1^2)  + p(y=-1) N(x\,; \mu_{-1},\sigma_{-1}^2)
\end{align}
is identifiable.
\end{prop}
\begin{proof}
It can be shown that the family of Gaussian mixture model with no apriori information about label marginals is identifiable up to a permutation of the labels   $y$ \citep{Teicher1963}. We proceed by assuming with no loss of generality that   $p(y=1)>p(y=-1)$. The alternative case ${p(y=1)}<p(y=-1)$ may be   handled in the same manner. Using the result of \citep{Teicher1963}   we have that if $p_{\mu,\sigma}(x)=p_{\mu',\sigma'}(x)$ for all $x$,   then $(p(y),\mu,\sigma)=(p(y),\mu',\sigma')$ up to a permutation of   the labels. Since permuting the labels violates our assumption   ${p(y=1)}> {p(y=-1)}$ we establish $(\mu,\sigma)=(\mu',\sigma')$ proving identifiability.
\end{proof}
The assumption that $p(y)$ is known is not entirely crucial. It may be relaxed by assuming that it is known whether $p(Y=1)>p(Y=-1)$ or ${p(Y=1)}<{p(Y=-1)}$. Proving Proposition~\ref{prop:identifiability} under this much weaker assumption follows identical lines. 
\begin{prop}\label{prop:GaussianMixConsistency}  
Under the assumptions of Proposition~\ref{prop:identifiability}  the MLE estimates for $(\mu,\sigma)=(\mu_1,\mu_{-1},\sigma_1,\sigma_{-1})$
\begin{align}
(\hat\mu^{(n)},\hat\sigma^{(n)})&=\argmax_{\mu,\sigma}   \ell_{n}(\mu,\sigma)\\ 
\ell_{n}(\mu,\sigma)
&= \sum_{i=1}^n \log \sum_{y^{(i)}\in\{-1,+1\}} p(y^{(i)})   p_{\mu_y,\sigma_y}(f_{\theta}(X^{(i)})|y^{(i)}).  
\end{align}
are consistent i.e., 
$(\hat\mu^{(n)}_1,\hat\mu^{(n)}_{-1},\hat\sigma^{(n)}_1,\hat\sigma^{(n)}_{-1})$ converge as $n\to\infty$ to the true parameter values with   probability 1.
\end{prop}
\begin{proof}
Denoting $p_{\eta}(z)=\sum_y p(y)p_{\mu_y,\sigma_y}(z|y)$ with $\eta=(\mu,\sigma)$ we note that $p_{\eta}$ is identifiable (see Proposition~\ref{prop:identifiability}) in $\eta$ and the available samples $z^{(i)}=f_{\theta}(X^{(i)})$ are iid samples from $p_{\eta}(z)$. We therefore use standard statistics theory which indicates that the MLE for identifiable parametric model is strongly consistent \citep{Ferguson1996}.
\end{proof}
\begin{prop} \label{prop:riskConsistency} Under the assumptions of Proposition~\ref{prop:identifiability} and assuming the loss $L$ is   given by one of \eqref{eq:loss1}-\eqref{eq:loss3} with a normal $f_{\theta}(X)|Y\sim N(\mu_y,\sigma_y^2)$, the plug-in risk   estimate 
\begin{align} \label{eq:plugin}
\hat R_{n}(\theta) =  \sum_{y\in\{-1,+1\}}  p(y)  & \int_{\R}   p_{\hat\mu^{(n)}_y,\hat\sigma^{(n)}_y}(f_{\theta}(X)=\alpha|y)   L(y,\alpha) \, d\alpha. 
\end{align}
is consistent, i.e., for all $\theta$, \[ P\left(\lim_n \hat R_n(\theta) = R(\theta)\right) = 1.\] 
\end{prop}
\begin{proof}
  The plug-in risk estimate $\hat R_n$ in \eqref{eq:plugin} is a   continuous function (when $L$ is given by \eqref{eq:loss1},   \eqref{eq:loss2} or \eqref{eq:loss3}) of $\hat \mu_1^{(n)},\hat   \mu_{-1}^{(n)},\hat \sigma_1^{(n)},\hat \sigma_{-1}^{(n)}$ (note   that $\mu_y$ and $\sigma_y$ are functions of $\theta$), which we   denote $\hat R_n (\theta) = h(\hat \mu_1^{(n)},\hat   \mu_{-1}^{(n)},\hat \sigma_1^{(n)},\hat \sigma_{-1}^{(n)})$.

Using Proposition~\ref{prop:GaussianMixConsistency} we have that
\begin{align*}
\lim_{n\to\infty} (\hat \mu_1^{(n)},\hat \mu_{-1}^{(n)},\hat     \sigma_1^{(n)},\hat \sigma_{-1}^{(n)}) = (\mu_1^{\text{true}},     \mu_{-1}^{\text{true}}, \sigma_1^{\text{true}},     \sigma_{-1}^{\text{true}})
\end{align*}
with probability 1. Since continuous functions preserve limits we have \[ \lim_{n\to\infty} h(\hat \mu_1^{(n)},\hat \mu_{-1}^{(n)},\hat \sigma_1^{(n)},\hat \sigma_{-1}^{(n)})= h(\mu^{\text{true}}_1,\mu^{\text{true}}_{-1},\sigma^{\text{true}}_1,\sigma^{\text{true}}_{-1})\] with probability 1 which implies convergence $\lim_{n\to\infty}\hat R_n(\theta) = R(\theta)$ with probability 1.
\end{proof}

\subsection{Unsupervised Consistency of $\argmin \hat R_n(\theta)$} 
The convergence above $\hat R_n(\theta)\to R(\theta)$ is pointwise in $\theta$. If the stronger concept of uniform convergence is assumed over $\theta\in\Theta$ we obtain consistency of $\argmin_{\theta} \hat R_n(\theta)$. This surprising result indicates that in some cases it is possible to retrieve the expected risk minimizer (and therefore the Bayes classifier in the case of the hinge loss, log-loss and exp-loss) using only unlabeled data. We show this uniform convergence using a modification of Wald's classical MLE consistency result \citep{Ferguson1996}. 

Denoting 
\begin{align*}
p_{\eta}(z)&=\sum_{y\in\{-1,+1\}} p(y) p_{\mu_y,\sigma_y}(f(X)=z|y), \quad \eta=(\mu_1,\mu_{-1},\sigma_1,\sigma_{-1})
\end{align*}
we first show that the MLE converges to the true parameter value $\hat \eta_n\to\eta_0$ uniformly. Uniform convergence of the risk estimator $\hat R_{n}(\theta)$ follows. Since changing $\theta\in\Theta$ results in a different $\eta\in E$ we can state the uniform convergence in $\theta\in\Theta$ or alternatively in $\eta\in E$.
\begin{prop} \label{prop:unifConvMLE}
Let $\theta$ take values in $\Theta$ for which $\eta\in E$ for some compact set $E$. Then assuming the conditions in Proposition~\ref{prop:riskConsistency} the convergence of the MLE to the true value $\hat \eta_n\to \eta_0$ is uniform in $\eta_0\in E$ (or alternatively $\theta\in\Theta$).
\end{prop}
\begin{proof}
We start by making the following notation 
\begin{align*}
U(z,\eta,\eta_0)&=\log p_{\eta}(z)-\log p_{\eta_0}(z)\\
\alpha(\eta,\eta_0)&=E_{p_{\eta_0}} U(z,\eta,\eta_0)=-D(p_{\eta_0},p_{\eta}) \leq 0
\end{align*} 
with the latter quantity being non-positive and 0 iff $\eta=\eta_0$ (due to Shannon's inequality and identifiability of $p_{\eta}$).

For $\rho>0$ we define the compact set $S_{\eta_0,\rho}=\{\eta\in E: \|\eta-\eta_0\|\geq \rho\}$. Since $\alpha(\eta,\eta_0)$ is continuous it achieves its maximum (with respect to $\eta$) on $S_{\eta_0,\rho}$ denoted by $\delta_{\rho}(\eta_0)=\max_{\eta\in S_{\eta_0,\rho}} \alpha(\eta,\eta_0)<0$ which is negative since $\alpha(\eta,\eta_0)=0$ iff $\eta=\eta_0$. Furthermore, note that $\delta_{\rho}(\eta_0)$ is itself continuous in $\eta_0\in E$ and since $E$ is compact it achieves its maximum 
\[\delta=\max_{\eta_0\in E} \delta_{\rho}(\eta_0)=\max_{\eta_0\in E}\,\,\,\max_{\eta\in S_{\eta_0,\rho}}\,\,\, \alpha(\eta,\eta_0)<0\]
which is negative for the same reason.

Invoking the uniform strong law of large numbers \citep{Ferguson1996} we have $n^{-1}\sum_{i=1}^n U(z^{(i)},\eta,\eta_0)\to \alpha(\eta,\eta_0)$ uniformly over  $(\eta,\eta_0)\in E^2$. Consequentially, there exists $N$ such that for $n>N$   (with probability 1)
\[\sup_{\eta_0\in E}\,\,\,\sup_{\eta\in S_{\eta_0,\rho}} \,\,\, \frac{1}{n} \sum_{i=1}^n  U(z^{(i)},\eta,\eta_0)<\delta/2<0.\]
But since $n^{-1}\sum_{i=1}^n  U(z^{(i)},\eta,\eta_0)\to 0$ for $\eta=\eta_0$ it follows that the MLE 
\[\hat \eta_n = \,\,\max_{\eta\in E} \,\,\frac{1}{n}\sum_{i=1}^n  U(z^{(i)},\eta,\eta_0)\] 
is outside $S_{\eta_0,\rho}$ (for $n>N$ uniformly in $\eta_0\in E$) which implies $\|\hat\eta_n-\eta_0\|\leq \rho$. Since $\rho>0$ is arbitrarily and $N$ does not depend on $\eta_0$ we have $\hat\eta_n\to\eta_0$ uniformly over $\eta_0\in E$.
\end{proof}

\begin{prop} \label{prop:unifConvRisk}
Assuming that $X,\Theta$ are bounded in addition to the assumptions of Proposition~\ref{prop:unifConvMLE} the convergence $\hat R_n(\theta)\to R(\theta)$ is uniform in $\theta\in\Theta$.
\end{prop}
\begin{proof}
Since $X,\Theta$ are bounded the margin value $f_{\theta}(X)$ is bounded with probability 1. As a result the loss function is bounded in absolute value by a constant $C$. We also note that a mixture of two Gaussian model (with known mixing proportions) is Lipschitz continuous in its parameters
\begin{multline*}
 \Bigg|\sum_{y\in\{-1,+1\}}  p(y)p_{\hat\mu^{(n)}_y,\hat\sigma^{(n)}_y}(z)  -\sum_{y\in\{-1,+1\}}  p(y) p_{\mu^{true}_y,\sigma^{true}_y}(z) \Bigg| \\ \leq  t(z) \, \cdot \, \Big|\Big|(\hat \mu_1^{(n)},\hat \mu_{-1}^{(n)},\hat     \sigma_1^{(n)},\hat \sigma_{-1}^{(n)}) - (\mu_1^{\text{true}},     \mu_{-1}^{\text{true}}, \sigma_1^{\text{true}},     \sigma_{-1}^{\text{true}})\Big|\Big|
\end{multline*}
which may be verified by noting that the partial derivatives of 
$p_{\eta}(z)=\sum_y p(y)p_{\mu_y,\sigma_y}(z|y)$ 
\begin{align*}
\frac{\partial p_{\eta}(z) } {\partial \hat\mu^{(n)}_1}&= 
\frac{p(y=1) (z - \hat\mu^{(n)}_1)} {(2\pi)^{1/2} \hat\sigma^{(n)^3}_1}
e^{-\frac{ (z - \hat\mu^{(n)}_1)^2 } {2 \hat\sigma^{(n)^3}_1} }\\
\frac{\partial p_{\eta}(z)}{\partial \hat\mu^{(n)}_{-1}}&= \frac{p(y=-1) (z - \hat\mu^{(n)}_{-1})}{(2\pi)^{1/2} \hat\sigma^{(n)^3}_{-1}} e^{-\frac{(z - \hat\mu^{(n)}_{-1})^2}{2 \hat\sigma^{(n)^3}_{-1}}}\\
\frac{\partial p_{\eta}(z)}{\partial \hat\sigma^{(n)}_1}&= -\frac{p(y=1) (z - \hat\mu^{(n)}_1)^2}{(2\pi)^{3/2} \hat\sigma^{(n)^6}_1} e^{-\frac{(z - \hat\mu^{(n)}_1)^2}{2 \hat\sigma^{(n)^2}_1}}\\
\frac{\partial p_{\eta}(z)}{\partial \hat\sigma^{(n)}_{-1}}&= -\frac{p(y=-1) (z - \hat\mu^{(n)}_{-1})^2}{(2\pi)^{3/2} \hat\sigma^{(n)^6}_{-1}} e^{-\frac{(z - \hat\mu^{(n)}_{-1})^2}{2 \hat\sigma^{(n)^2}_{-1}}}\\
\end{align*}
are bounded for a compact $E$. These observations, together with Proposition~\ref{prop:unifConvMLE} lead to 
\begin{align*}
|\hat R_n(\theta)-R(\theta)| &\leq  \sum_{y\in\{-1,+1\}}  p(y) \int   \Big |p_{\hat\mu^{(n)}_y,\hat\sigma^{(n)}_y}(f_{\theta}(X)=\alpha)  
-p_{\mu^{\text{true}}_y,\sigma^{\text{true}}_y}(f_{\theta}(X)=\alpha)  \Big|\,
|L(y,\alpha)| d\alpha\\
& \leq C \int \Big|\sum_{y\in\{-1,+1\}}  p(y)p_{\hat\mu^{(n)}_y,\hat\sigma^{(n)}_y}(\alpha)  -\sum_{y\in\{-1,+1\}}  p(y) p_{\mu^{\text{true}}_y,\sigma^{\text{true}}_y}(\alpha) \Big|  \, d\alpha \\
& \leq C\,  \| (\hat \mu_1^{(n)},\hat \mu_{-1}^{(n)},\hat     \sigma_1^{(n)},\hat \sigma_{-1}^{(n)}) - (\mu_1^{\text{true}},     \mu_{-1}^{\text{true}}, \sigma_1^{\text{true}},     \sigma_{-1}^{\text{true}}) \| \int_a^b t(z)dz \\
&\leq C'\,  \| (\hat \mu_1^{(n)},\hat \mu_{-1}^{(n)},\hat     \sigma_1^{(n)},\hat \sigma_{-1}^{(n)}) - (\mu_1^{\text{true}},     \mu_{-1}^{\text{true}}, \sigma_1^{\text{true}},     \sigma_{-1}^{\text{true}}) \| \to 0 \\
\end{align*}
uniformly over $\theta\in\Theta$.
\end{proof}

\begin{prop}
Under the assumptions of Proposition~\ref{prop:unifConvRisk} 
\[ P\left(\lim_{n\to\infty} \argmin_{\theta\in\Theta} \hat R_n(\theta)=\argmin_{\theta\in\Theta} R(\theta)\right)=1.\]
\end{prop}
\begin{proof} 
We denote $t^*=\argmin R(\theta)$, $t_n = \argmin \hat R_n(\theta)$. Since $\hat R_n(\theta)\to R(\theta)$ uniformly, for each $\epsilon>0$ there exists $N$ such that for all $n>N$, $|\hat R_n(\theta)-R(\theta)|<\epsilon$. 

Let $S=\{\theta: \|\theta-t^*\|\geq \epsilon \}$ and $\min_{\theta\in S} R(\theta)>R(t^*)$ ($S$ is compact and thus $R$ achieves its minimum on it). There exists $N'$ such that for all $n>N'$ and $\theta\in S$, $\hat R_n(\theta) \geq R(t^*)+\epsilon$. On the other hand, $\hat R_n(t^*)\to R(t^*)$ which together with the previous statement implies that there exists $N''$ such that for $n>N''$, $\hat R_n(t^*) < \hat R_n(\theta)$ for all $\theta\in S$. We thus conclude that for $n>N''$, $t_n\not\in S$. Since we showed that for each $\epsilon>0$ there exists $N$ such that for all $n>N$ we have $\|t_n-t^*\|\leq \epsilon$, $t_n\to t^*$ which concludes the proof.
\end{proof}

\subsection{Asymptotic Variance} \label{sec:asymVar}
In addition to consistency, it is useful to characterize the accuracy of our estimator $\hat R_n(\theta)$ as a function of $p(y),\mu,\sigma$. We do so by computing the asymptotic variance of the estimator which equals the inverse Fisher information 
\begin{align*}
\sqrt{n} (\hat \eta^{\text{mle}}_n -\eta_0) \tood N(0,I^{-1}(\eta^{\text{true}}))
\end{align*}
and analyzing its dependency on the model parameters. We first derive the asymptotic variance of MLE for mixture of Gaussians (we denote below $\eta=(\eta_1,\eta_2), \eta_i=(\mu_i,\sigma_{i})$)
\begin{align}
p_{\eta}(z) &=p(Y=1) N(z; \mu_1,\sigma_1^2)  + p(Y=-1) N(z; \mu_{-1},\sigma_{-1}^2) \\
&=p_1 p_{\eta_1}(z)  + p_{-1} p_{\eta_{-1}}(z).
\end{align}
The  elements of $4 \times 4$ information matrix $I(\eta)$ 
\begin{align*}
I(\eta_i,\eta_j) = \E \left(\frac{\partial \log p_{\eta}(z) }{\partial \eta_i}\frac{\partial \log p_{\eta}(z)}{\partial \eta_j}\right)
\end{align*}
may be computing using the following derivatives 
\begin{align*}
\frac{\partial \log p_{\eta}(z) }{\partial \mu_i} &= \frac{p_i}{\sigma_i} \left(\frac{z-\mu_i}{\sigma_i}\right) \frac{p_{\eta_i}(z)}{p_{\eta}(z)}  \\
\frac{\partial \log p_{\eta}(z) }{\partial \sigma^2_i} &=\frac{p_i}{2\sigma_i} \left(\left(\frac{z-\mu_i}{\sigma_i}\right)^2-1\right)\frac{p_{\eta_i}(z)}{p_{\eta}(z)}
\end{align*}
for $i=1,-1$. Using derivations similar to the ones in \citep{behboodian1972} we obtain
\begin{align*}
I(\mu_i,\mu_j) &=\frac{p_ip_j}{\sigma_i\sigma_j} M_{11} \Big(p_{\eta_i}(z),p_{\eta_i}(z)\Big)\\
I(\mu_1,\sigma^2_i)&=\frac{p_1p_i}{2\sigma_1\sigma^2_i} \Big[M_{12} \Big(p_{\eta_i}(z),p_{\eta_i}(z)\Big)-M_{10} \Big(p_{\eta_1}(z),p_{\eta_i}(z)\Big)\Big] \\
I(\mu_{-1},\sigma^2_i)&=\frac{p_{-1}p_i}{2\sigma_{-1}\sigma^2_i} \Big[M_{21} \Big(p_{\eta_i}(z),p_{\eta_{-1}}(z)\Big)-M_{01} \Big(p_{\eta_i}(z),p_{\eta_{-1}}(z)\Big)\Big] \\
I(\sigma^2_i,\sigma^2_i)&=\frac{p^4_i}{4\sigma^4_i} \Big[M_{00} \Big(p_{\eta_i}(z),p_{\eta_i}(z)\Big)-2M_{11} \Big(p_{\eta_i}(z),p_{\eta_i}(z)\Big)+M_{22} \Big(p_{\eta_i}(z),p_{\eta_i}(z)\Big)\Big] \\
I(\sigma^2_1,\sigma^2_{-1})&=\frac{p_1p_{-1}}{4\sigma^2_1\sigma^2_{-1}} \Big[M_{00} \Big(p_{\eta_1}(z),p_{\eta_{-1}}(z)\Big)-M_{20} \Big(p_{\eta_1}(z),p_{\eta_{-1}}(z)\Big)\\&\qquad -M_{02} \Big(p_{\eta_1}(z),p_{\eta_{-1}}(z)\Big)+M_{22} \Big(p_{\eta_1}(z),p_{\eta_{-1}}(z)\Big)\Big]
\end{align*}
where 
\begin{align*}
M_{m,n}\Big(p_{\eta_i}(z),p_{\eta_j}(z)\Big) &= \int_{-\infty}^{\infty} \left(\frac{z-\mu_i}{\sigma_i}\right)^m\left(\frac{z-\mu_j}{\sigma_j}\right)^n \frac{p_{\eta_i}(z)p_{\eta_j}(z)}{p_{\eta}(z)} \,dx.
\end{align*}

In some cases it is more instructive to consider the asymptotic variance of the risk estimator $\hat R_n(\theta)$ rather than that of the parameter estimate for $\eta=(\mu,\sigma)$. This could be computed using the delta method and the above Fisher information matrix 
\begin{align*}
\sqrt{n} (\hat R_n(\theta) - R(\theta)) \tood N(0,\nabla h(\eta^{\text{true}})^T I^{-1}(\eta^{true})\nabla h(\eta^{\text{true}}))
\end{align*}
where $\nabla h$ is the gradient vector of the mapping $R(\theta)= h(\eta)$. For example, in the case of the exponential loss \eqref{eq:loss1} we get
\begin{align*}
h(\eta)&=p(Y=1)\sigma_1\sqrt{2} \exp {\Big(\frac{(\mu_1-1)^2}{2}-\frac{\mu^2_1}{2\sigma^2_1}\Big)}+p(Y=-1)\sigma_{-1}\sqrt{2} \exp {\Big(\frac{(\mu_{-1}-1)^2}{2}-\frac{\mu^2_{-1}}{2\sigma^2_{-1}}\Big)}\\
\frac{\partial h(\eta)}{\partial \mu_1}& = \frac{\sqrt{2}P(Y=1)(\mu_1(\sigma^2_1-1)-\sigma^2_1)}{\sigma_1} \exp \left(\frac{(\mu_1-1)^2}{2}-\frac{\mu^2_1}{2\sigma^2_1}\right) \\
\frac{\partial h(\eta)}{\partial \mu_{-1}}& = \frac{\sqrt{2}P(Y=-1)(\mu_{-1}(\sigma^2_{-1}-1)+\sigma^2_{-1})}{\sigma_{-1}} \exp \left(\frac{(\mu_{-1}+1)^2}{2}-\frac{\mu^2_{-1}}{2\sigma^2_{-1}}\right) \\
\frac{\partial h(\eta)}{\partial \sigma^2_1}& = \frac{P(Y=1) (\mu^2_1+\sigma^2_1)}{\sqrt{2\sigma_1}} \Big(\frac{(\mu_1-1)^2}{2}-\frac{\mu^2_1}{2\sigma^2_1}\Big) \\
\frac{\partial h(\eta)}{\partial \sigma^2_{-1}}& = \frac{P(Y=-1) (\mu^2_{-1}+\sigma^2_{-1})}{\sqrt{2\sigma_{-1}}} \Big(\frac{(\mu_{-1}+1)^2}{2}-\frac{\mu^2_{-1}}{2\sigma^2_{-1}}\Big).
\end{align*}

Figure~\ref{fig:fim} plots the asymptotic accuracy of $\hat R_n(\theta)$ for log-loss. The left panel shows that the accuracy of $\hat R_n$ increases with the imbalance of the marginal distribution $p(Y)$. The right panel shows that the accuracy of $\hat R_n$ increases with the difference between the means $|\mu_1-\mu_{-1}|$ and the variances $\sigma_1/\sigma_2$. 

\begin{figure} 
\centering
  \includegraphics[scale=0.4]{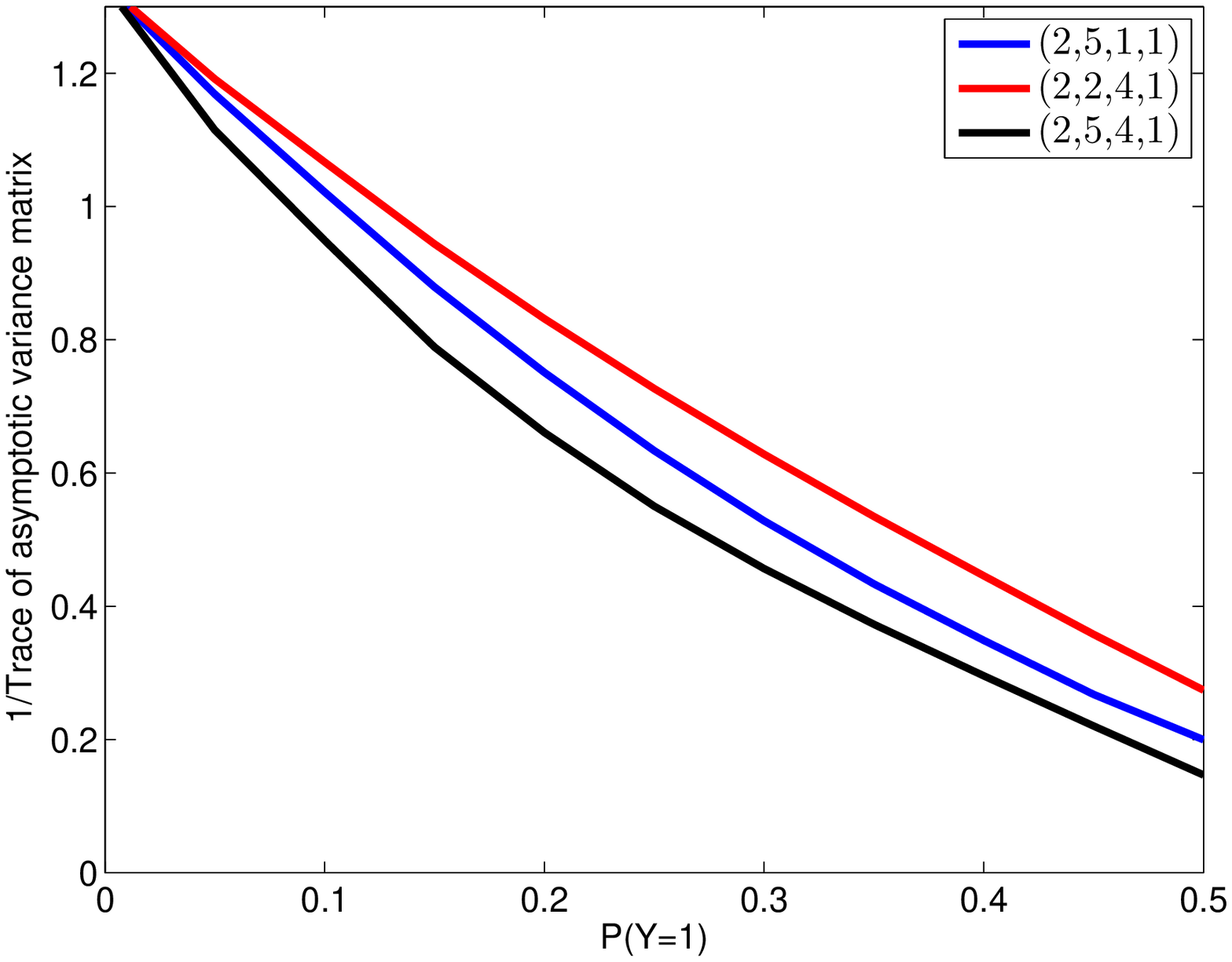}
  \includegraphics[scale=0.4]{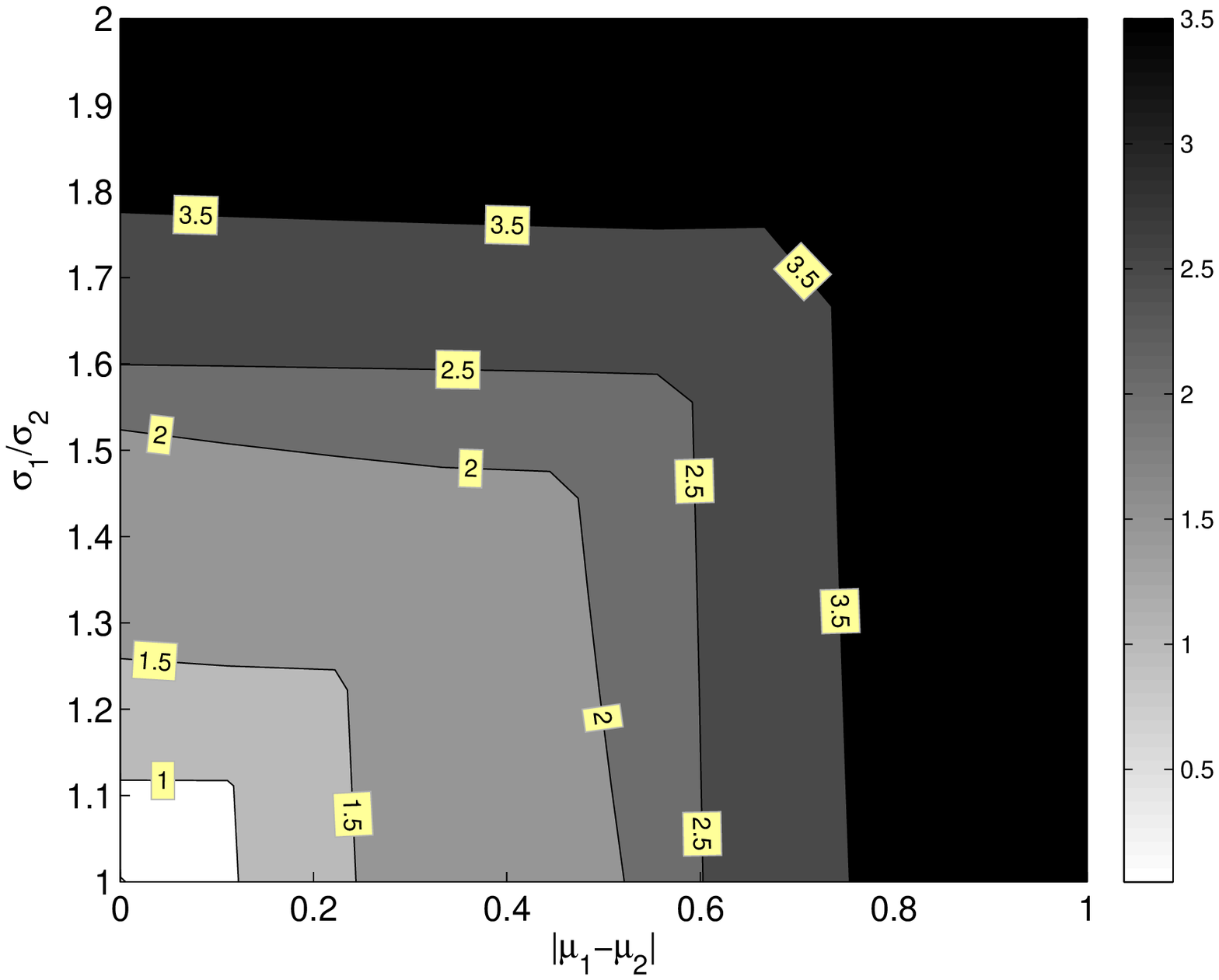}
  \caption{Left panel: asymptotic accuracy (inverse of trace of asymptotic variance) of $\hat R_n(\theta)$ for logloss as a function of the imbalance of the class marginal $p(Y)$. The accuracy increases with the class imbalance as it is easier to separate the two mixture components. Right panel: asymptotic accuracy (inverse of trace of asymptotic variance) as a function of the difference between the means $|\mu_1-\mu_{-1}|$ and the variances $\sigma_1/\sigma_2$. See text for more information.} 
\label{fig:fim}
\end{figure}

\subsection{Multiclass Classification}
Thus far, we have considered unsupervised risk estimation in binary classification. In this section we describe a multiclass extension based on standard extensions of the margin concept to multiclass classification. In this case the margin vector associated with the multiclass classifier
\begin{align}
\hat Y = \argmax_{k=1,\ldots,K} f_{\theta^k}(X), \qquad  X,\theta^k \in\mathbb{R}^d
\end{align}
is $f_{\bm \theta}(X)=(f_{\theta^1}(X),\dots,f_{\theta^K}(X))$. Following our discussion of the binary case, $f_{\theta^k}(X)|Y$, $k=1,\ldots,K$ is assumed to be normally distributed with parameters that are estimated by maximizing the likelihood of a Gaussian mixture model. We thus have $K$ Gaussian mixture models, each one with $K$ mixture components. The estimated parameters are plugged-in as before into the multiclass risk
\begin{align}
R(\bm \theta) &= E_{p(f_{\bm \theta}(X),Y)}L(Y,f_{\bm \theta}(X))
\end{align}
where $L$ is a multiclass margin based loss function such as 
\begin{align}
L(Y,f_{\bm \theta}(X)) &=  \sum_{k\neq Y} \log(1+\exp(- f_{\theta^k}(X))) \label{eq:loglossMC} \\
L(Y,f_{\bm \theta}(X)) &= \sum_{k\neq Y}(1 + f_{\theta^k}(X))_{+}. \label{eq:hingelossMC}
\end{align}
Since the MLE for a Gaussian mixture model with $K$ components is consistent (assuming $P(Y)$ is known and all probabilities $P(Y=k), k=1,\ldots,K$ are distinct) the MLE estimator for $f_{\theta^k}(X)|Y=k'$ are consistent. Furthermore, if the loss $L$ is a continuous function of these parameters (as is the case for \eqref{eq:loglossMC}-\eqref{eq:hingelossMC}) the risk estimator $\hat R_n(\bm \theta)$ is consistent as well. 

\section{Application 1: Estimating Risk in Transfer Learning}
We consider applying our estimation framework in two ways. The first application, which we describe in this section, is estimating margin-based risks in transfer learning where classifiers are trained on one domain but tested on a somewhat different domain. The transfer learning assumption that labeled data exists for the training domain but not for the test domain motivates the use of our unsupervised risk estimation. The second application, which we describe in the next section, is more ambitious. It is concerned with training classifiers without labeled data whatsoever.

In evaluating our framework we consider both synthetic and real-world data. In the synthetic experiments we generate high dimensional data from two uniform distributions $X|\{Y=1\}$ and $X|\{Y=-1\}$ with independent dimensions and prescribed $p(Y)$ and classification accuracy. This controlled setting allows us to examine the accuracy of the risk estimator as a function of $n$, $p(Y)$, and the classifier accuracy. 

Figure~\ref{fig:accPyfig} shows that the relative error of $\hat R_n(\theta)$ (measured by
$|\hat R_n(\theta)-R_n(\theta)|/R_n(\theta)$) in estimating the logloss (left) and hinge loss (right) decreases with $n$ achieving accuracy of greater than 99\% for $n>1000$. In accordance with the theoretical results in Section~\ref{sec:asymVar} the figure shows that the estimation error decreases as the classifiers become more accurate and as $p(Y)$ becomes less uniform. We found these trends to hold in other experiments as well. In the case of exponential loss, however, the estimator performed substantially worse (figure omitted). This is likely due to the exponential dependency of the loss on $Yf_{\theta}(X)$ which makes it very sensitive to outliers.

\begin{figure}
\centering
\raisebox{2.1ex}{\includegraphics[scale=0.35]{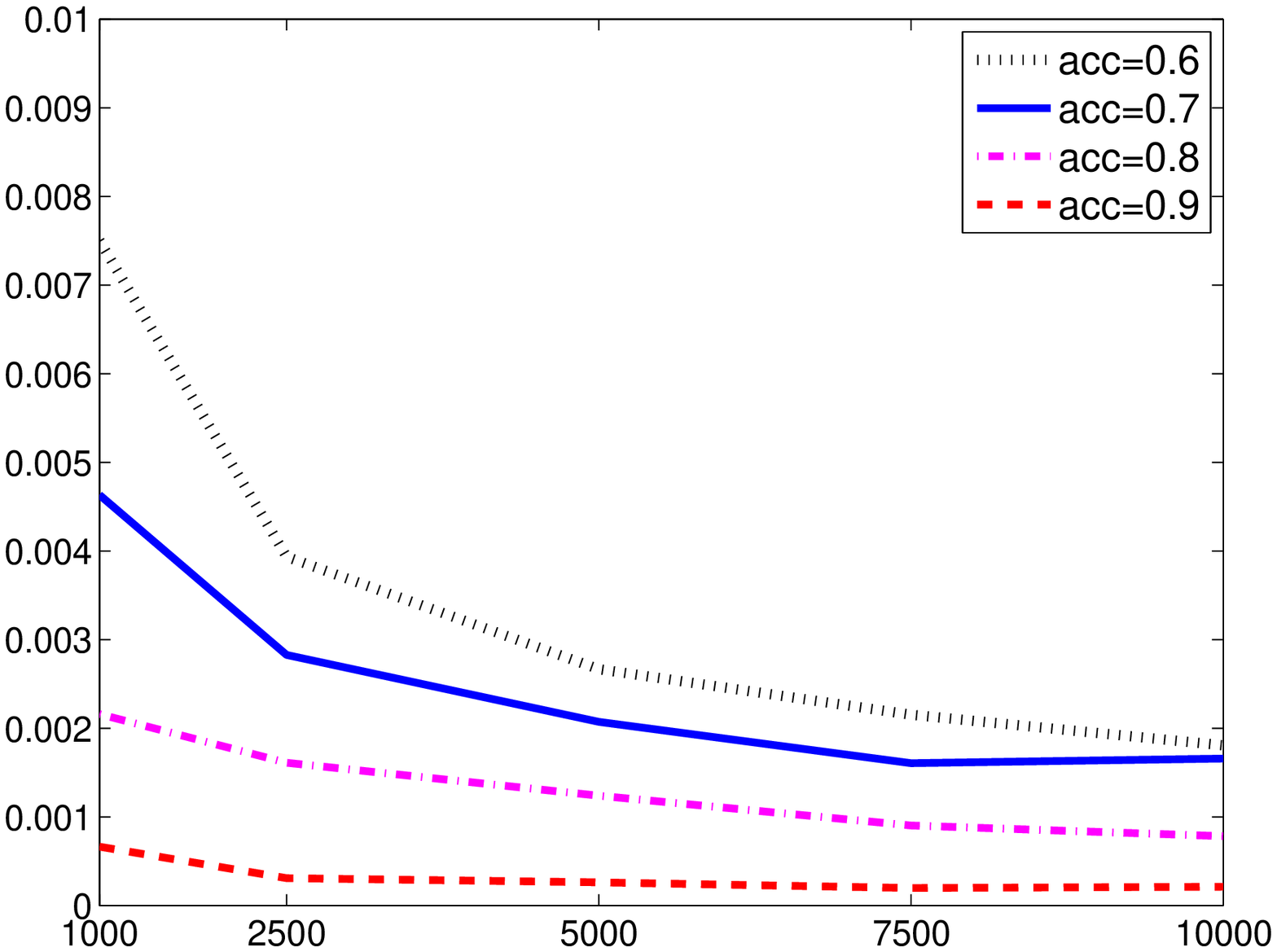}}
\includegraphics[scale=0.4815]{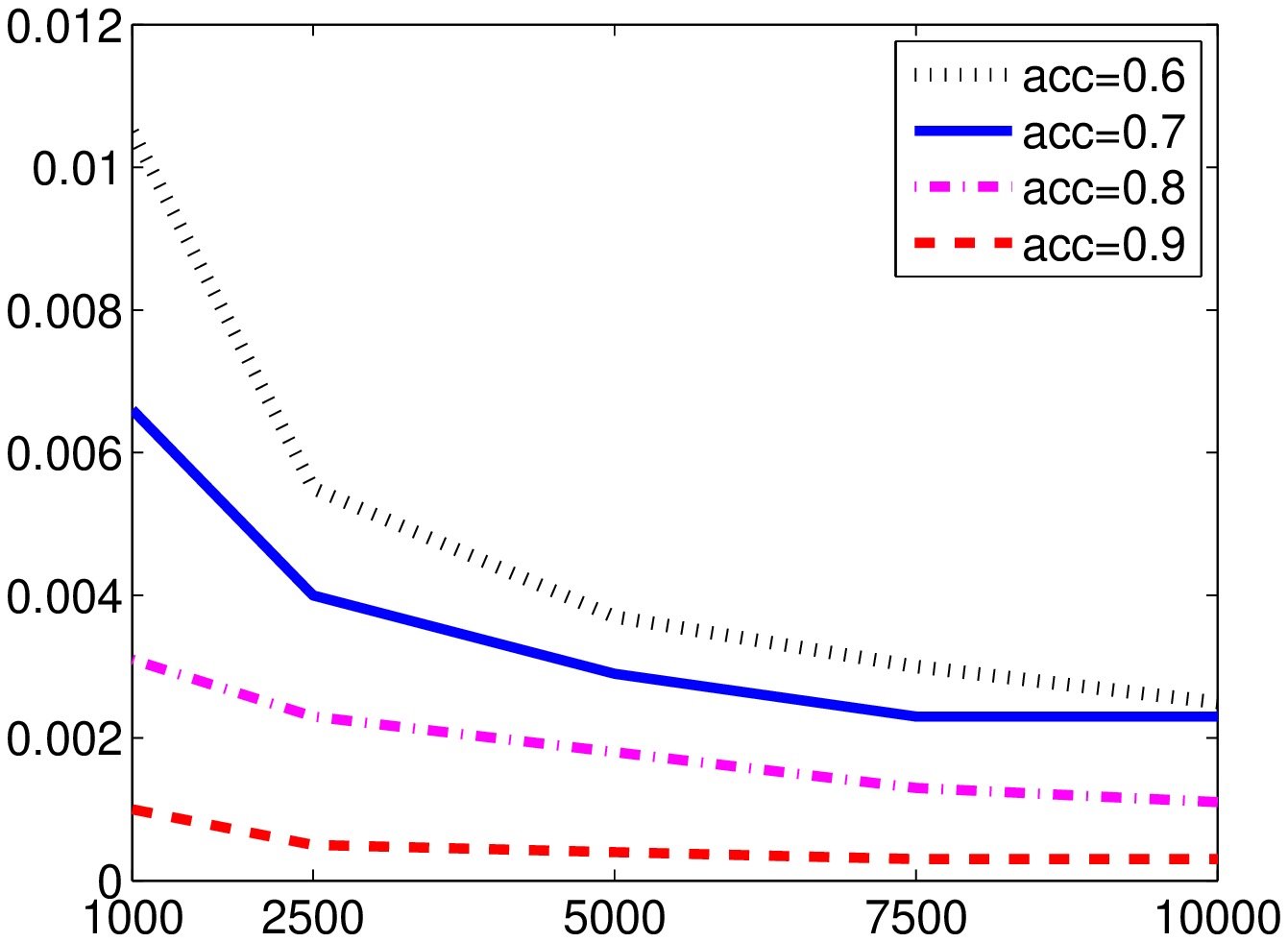}\\
\raisebox{2.1ex}{\includegraphics[scale=0.35]{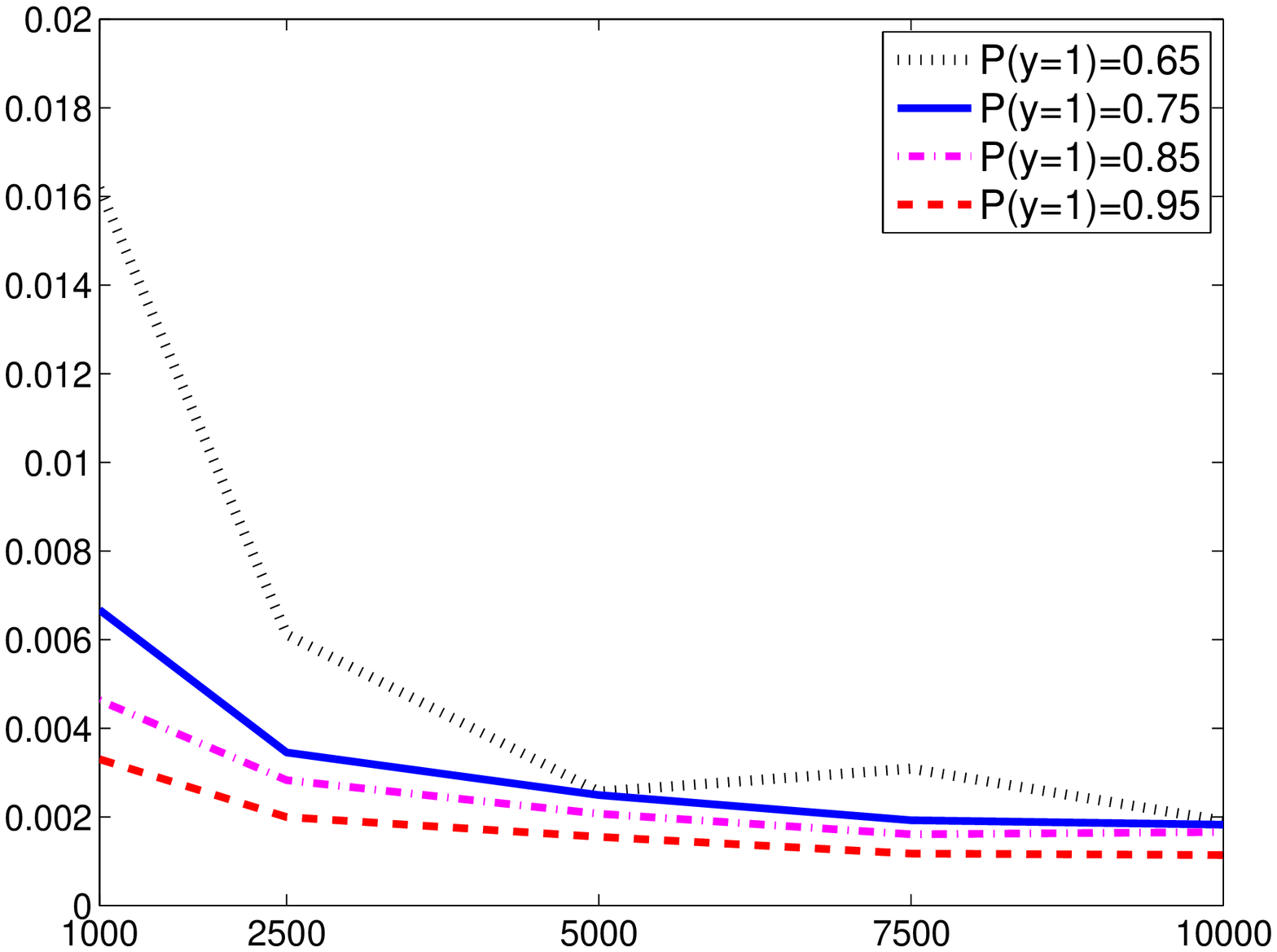}}
\includegraphics[scale=0.4815]{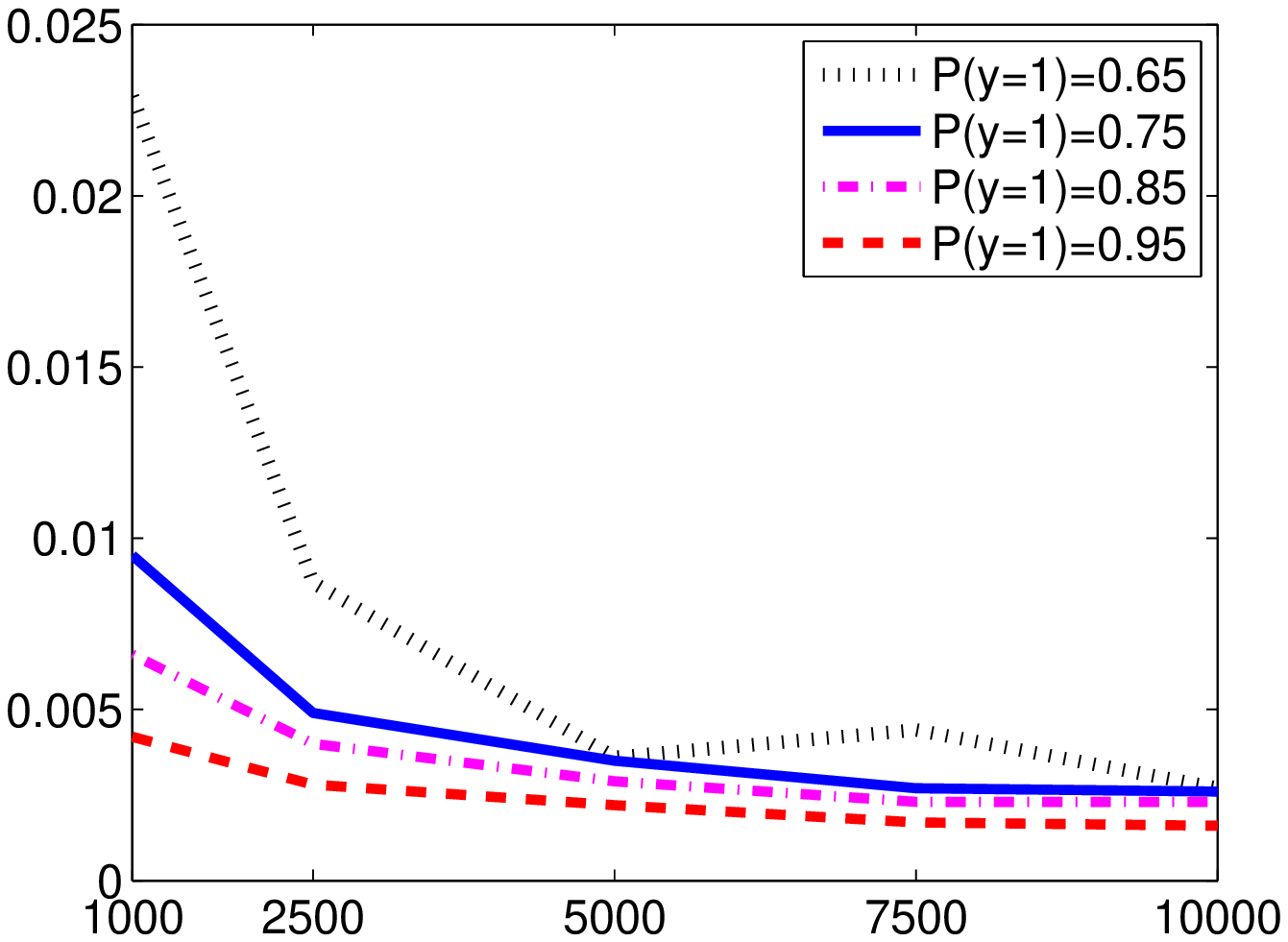}\\
\caption{The relative accuracy of $\hat R_n$ (measured by
$|\hat R_n(\theta)-R_{n}(\theta)|/R_{n}(\theta)$)
 as a function of $n$,  classifier accuracy (acc) and the label marginal $p(Y)$ (left: logloss, right: hinge-loss). The estimation error nicely decreases with $n$ (approaching 1\% at $n=1000$ and decaying further). It also decreases with the accuracy of the classifier (top) and non-uniformity of $p(Y)$ (bottom) in accordance with the theory of Section~\ref{sec:asymVar}.}\label{fig:accPyfig}
\end{figure}
  
Figure~\ref{fig:asymm} shows the accuracy of logloss estimation for a real world transfer learning experiment based on the 20-newsgroup data. Following the experimental setup of \citep{dai07} we trained a classifier (logistic regression) on one 20 newsgroup classification problem and tested it on a related problem.  Specifically, we used the hierarchical category structure to generate train and testing sets with different distributions (see Figure~\ref{fig:asymm} and \citep{dai07} for more detail). The unsupervised estimation of the logloss risk was very effective with relative accuracy greater than 96\% and absolute error less than 0.02.

\begin{figure}
  \centering { \begin{tabular}{|l|l|l|l|l|l|} \hline 
Data & $R_n$ & $|R_n-\hat R_n|$ & $|R_n-\hat R_n|/R_n$ & $n$ &       $p(Y=1)$\\ \hline 
sci vs. comp & 0.7088 & 0.0093 & 0.013 & 3590 & 0.8257\\ \hline
sci vs. rec & 0.641 & 0.0141 & 0.022 & 3958 & 0.7484\\ \hline
talk vs. rec & 0.5933 & 0.0159 & 0.026 & 3476 & 0.7126\\ \hline
talk vs. comp & 0.4678 & 0.0119 & 0.025 & 3459 & 0.7161 \\ \hline 
talk vs. sci & 0.5442 & 0.0241 & 0.044 & 3464 & 0.7151\\ \hline 
comp vs. rec & 0.4851 & 0.0049 & 0.010 & 4927 & 0.7972\\ \hline
\end{tabular}}
\caption{Error in estimating logloss for logistic regression classifiers trained on one 20-newsgroup classification task and tested on another. We followed the transfer learning setup described in \citep{dai07} which may be referred to for more detail. The train and testing sets contained samples from two top categories in the topic hierarchy but with different subcategory proportions. The first column indicates the top category classification task and the second indicates the empirical log-loss $R_n$ calculated using the true labels of the testing set \eqref{eq:empiricalLoss}. The third and forth columns indicate the absolute and relative errors of $\hat R_n$. The fifth and sixth columns indicate the train set size and the label marginal distribution.}
\label{fig:asymm}
\end{figure}

\section{Application 2: Unsupervised Learning of 
Classifiers} \label{sec:UnsuperTrain}
Our second application is a very ambitious one: training classifiers using unlabeled data by minimizing the unsupervised risk estimate $\hat \theta_n=\argmin \hat R_n(\theta)$. We evaluate the performance of the learned classifier $\hat \theta_n$ based on three quantities: (i) the unsupervised risk estimate  $\hat R_n(\hat\theta_n)$, (ii) the supervised risk estimate $R_n(\hat\theta_n)$, and (iii) its classification error rate. We also compare the performance of $\hat \theta_n=\argmin \hat R_n(\theta)$ with that of its supervised analog $\argmin R_n(\theta)$. 

We compute $\hat \theta_n=\argmin \hat R_n(\theta)$ using two algorithms (see Algorithms \ref{alg:gradDescent}-\ref{alg:gridSearch}) that start with an initial $\theta^{(0)}$ and iteratively construct a sequence of classifiers $\theta^{(1)},\ldots,\theta^{(T)}$ which steadily decrease $\hat R_n$. Algorithm~\ref{alg:gradDescent} adopts a gradient descent-based optimization. At each iteration $t$, it approximates the gradient vector $\nabla \hat R_n(\theta^{(t)})$ numerically using a finite difference approximation~\eqref{eq:central}. Algorithm~\ref{alg:gridSearch} proceeds by constructing a grid search along every dimension of $\theta^{(t)}$ and set $[\theta^{(t)}]_i$ to the grid value that minimizes $\hat R_n$. Although we focus on unsupervised training of logistic regression (minimizing unsupervised logloss estimate), the same techniques may be generalized to train other margin-based classifiers such as SVM by minimizing the unsupervised hinge-loss estimate. 

\begin{algorithm*}
\begin{algorithmic}
\caption{Unsupervised Gradient Descent}
\label{alg:gradDescent}
{
   \STATE {\bfseries Input:} $X^{(1)},\ldots,X^{(n)}\in \mathbb{R}^d$, $p(Y)$, step size $\alpha$
   \REPEAT
   \STATE Initialize $t=0$, $\theta^{(t)} =\theta^0\in \mathbb{R}^d$
   \STATE Compute $f_{\theta^{(t)}}(X^{(j)}) = \langle\theta^{(t)},X^{(j)}\rangle$ $\forall j=1,\dots,n$
   \STATE Estimate $(\hat \mu_1,\hat\mu_{-1}, \hat\sigma_1,\hat\sigma_{-1})$ by maximizing  \eqref{eq:ll}
   \FOR{$i=1$ {\bfseries to} $d$}
   \STATE Plug-in the estimates into~\eqref{eq:plugin} to approximate
   \STATE 
\begin{align}&\frac{\partial \hat R_n(\theta^{(t)})}{\partial \theta_i}=\frac{\hat R_n(\theta^{(t)} + h_i e_i) - \hat R_n(\theta^{(t)} - h_i e_i)}{2h_i}\nonumber \\
&(e_i \text{ is an all zero vector except for } [e_i]_i=1) 
\label{eq:central} 
\end{align} 
   \ENDFOR
   \STATE Form $\nabla \hat    R_n(\theta^{(t)})=(\frac{\partial \hat R_n(\theta^{(t)})}{\partial      \theta^{(t)}_1}, \dots, \frac{\partial \hat R_n(\theta^{(t)})}{\partial\theta^{(t)}_d})$ 
\STATE Update $\theta^{(t+1)} = \theta^{(t)} - \alpha    \nabla \hat R_n(\theta^{(t)})$, $t = t +1$ \UNTIL{convergence} \STATE    {\bfseries Output:} linear classifier $\theta^{\text{final}} =    \theta^{(t)}$ }
\end{algorithmic}
\end{algorithm*}

\begin{algorithm*}
\caption{Unsupervised Grid Search}
\label{alg:gridSearch}
\begin{algorithmic} 
{   \STATE {\bfseries Input:} $X^{(1)},\ldots,X^{(n)}\in \mathbb{R}^d$, $p(Y)$, grid-size $\tau$
   \STATE Initialize $\theta_i \sim \text{Uniform}(-2,2)$ for all $i$
   \REPEAT
   \FOR{$i=1$ {\bfseries to} $d$}
   \STATE Construct $\tau$ points grid in the range $[\theta_i - 4 \tau,\theta_i + 4 \tau]$ 
   \STATE Compute the risk estimate \eqref{eq:plugin} where all dimensions of $\theta^{(t)}$ are fixed except for $[\theta^{(t)}]_i$ which is evaluated at each grid point. 
\STATE Set $[\theta^{(t+1)}]_i$ to the grid value that minimized \eqref{eq:plugin}
   \ENDFOR
   \UNTIL{convergence} \STATE {\bfseries Output:} linear classifier    $\theta^{\text{final}} = \theta$ }
\end{algorithmic} 
\end{algorithm*}

Figures~\ref{fig:rcv1}-\ref{fig:mnist} display $\hat R_n(\hat\theta_n)$, $R_n(\hat\theta_n)$ and $\text{error-rate}(\hat\theta_n)$ on the training and testing sets as on two real world datasets: RCV1 (text documents) and MNIST (handwritten digit images) datasets. In the case of RCV1 we discarded all but the most frequent $504$ words (after stop-word removal) and represented documents using their tfidf scores. We experimented on the binary classification task of distinguishing the top category (positive) from the next $4$ top categories (negative) which resulted in $p(y=1)=0.3$ and $n=199328$. $70\%$ of the data was chosen as a (unlabeled) training set and the rest was held-out as a test-set. In the case of MNIST data, we normalized each of the $28\times 28=784$ pixels to have $0$ mean and unit variance. Our classification task was to distinguish images of the digit one (positive) from the digit 2 (negative) resulting in  $14867$ samples and $p(Y=1)=0.53$.  We randomly choose $70\%$ of the data as a training set and kept the rest as a testing set.

\begin{figure} \centering
\includegraphics[scale=0.4]{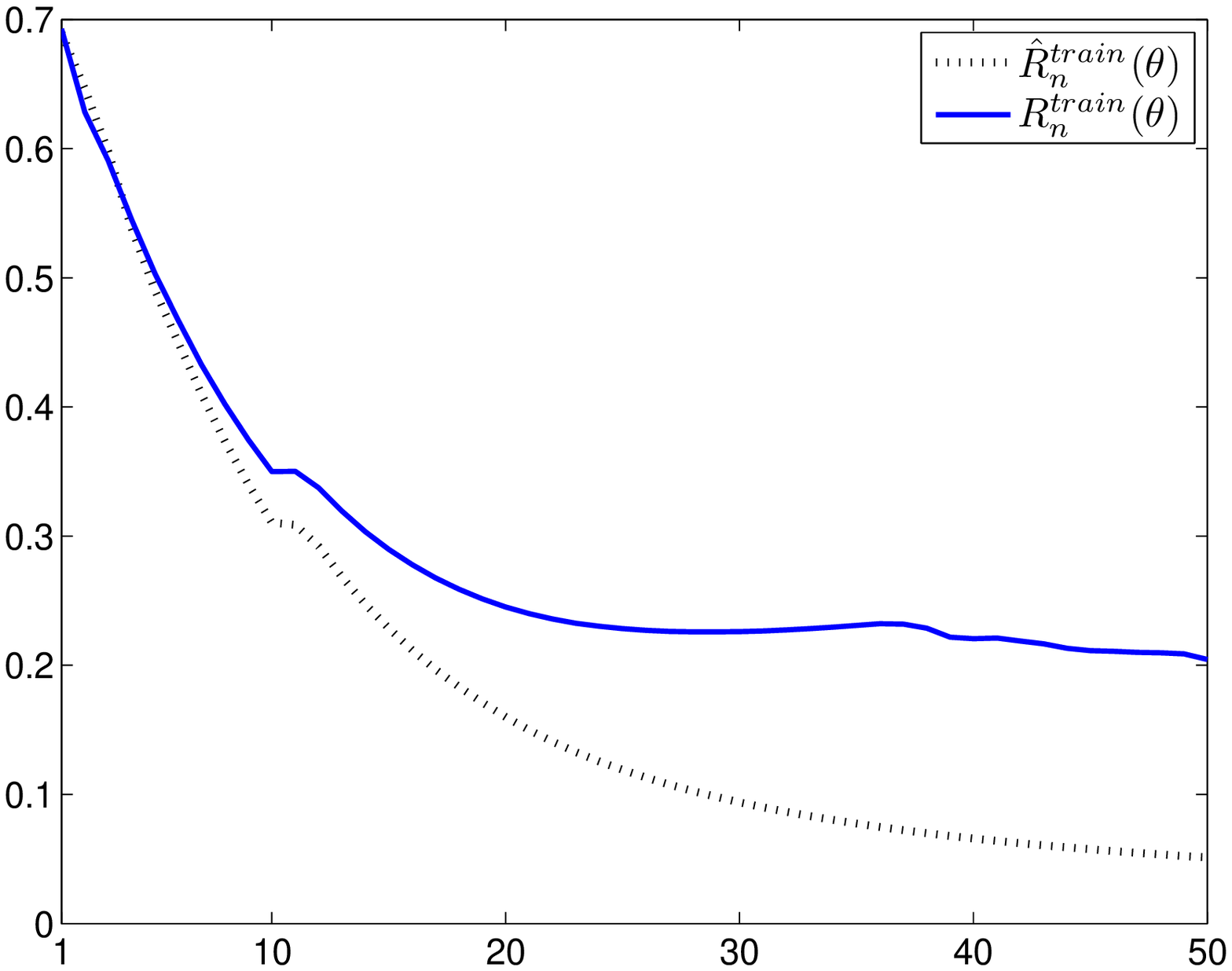}   \includegraphics[scale=0.4]{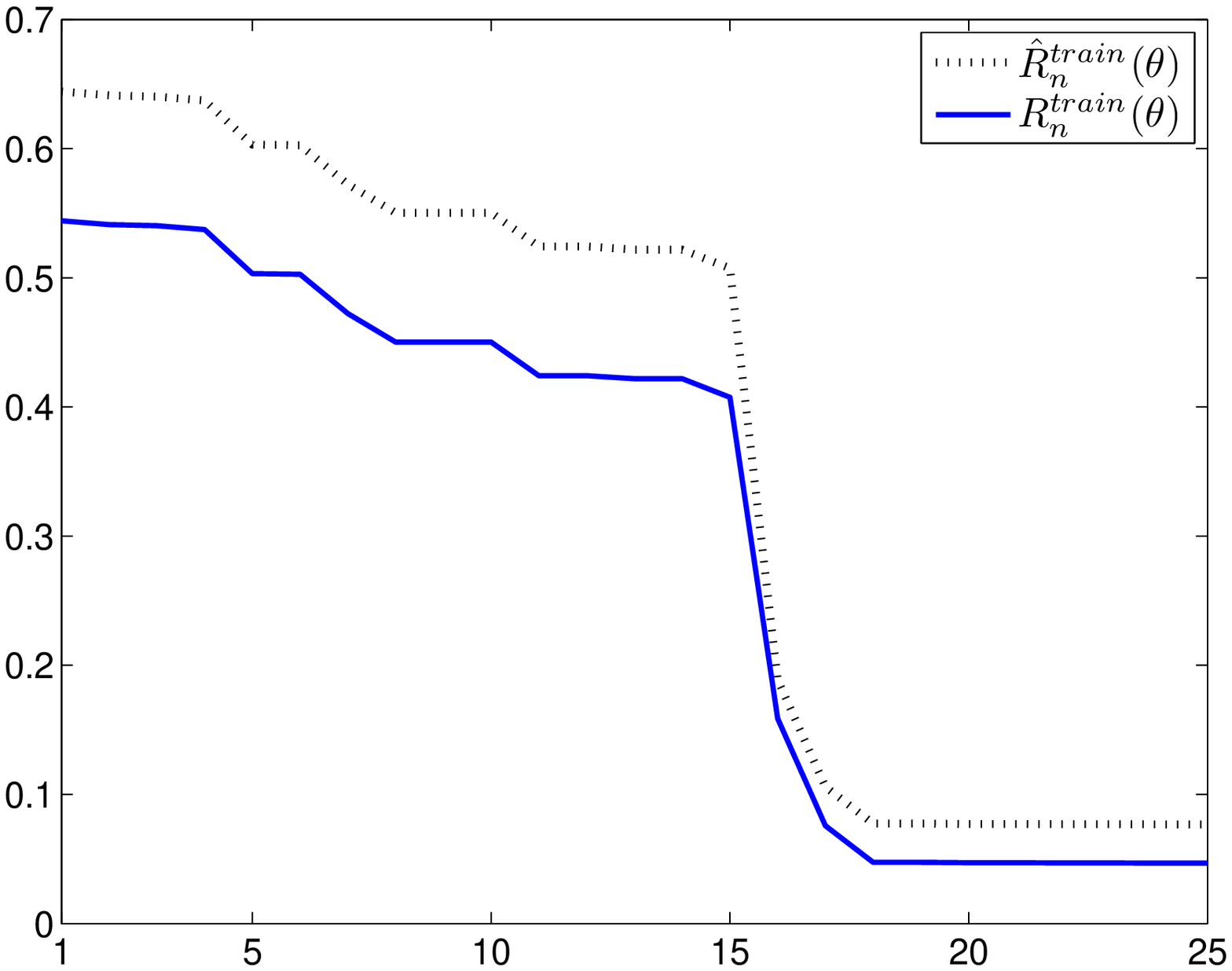}\\
\includegraphics[scale=0.4]{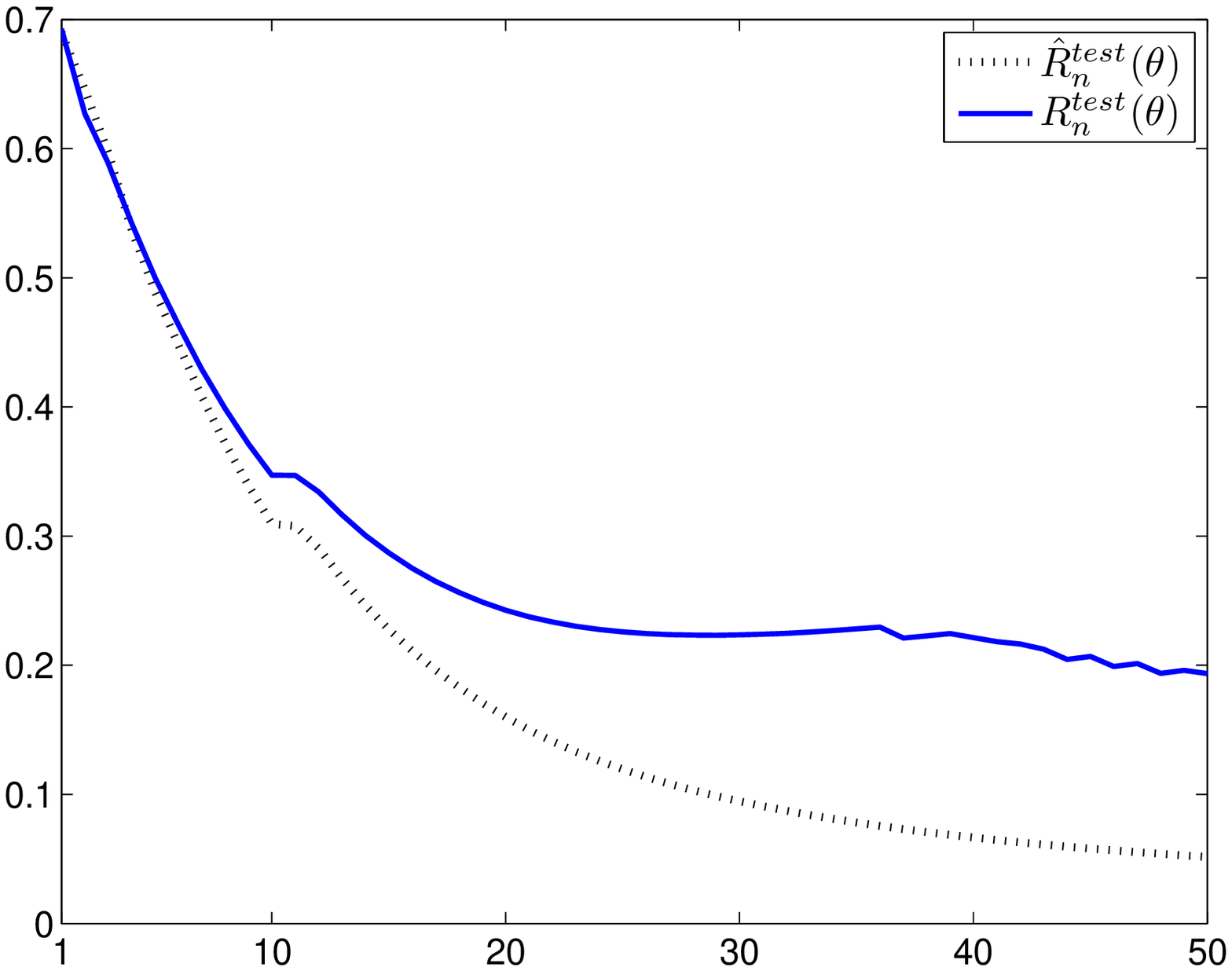}
\includegraphics[scale=0.4]{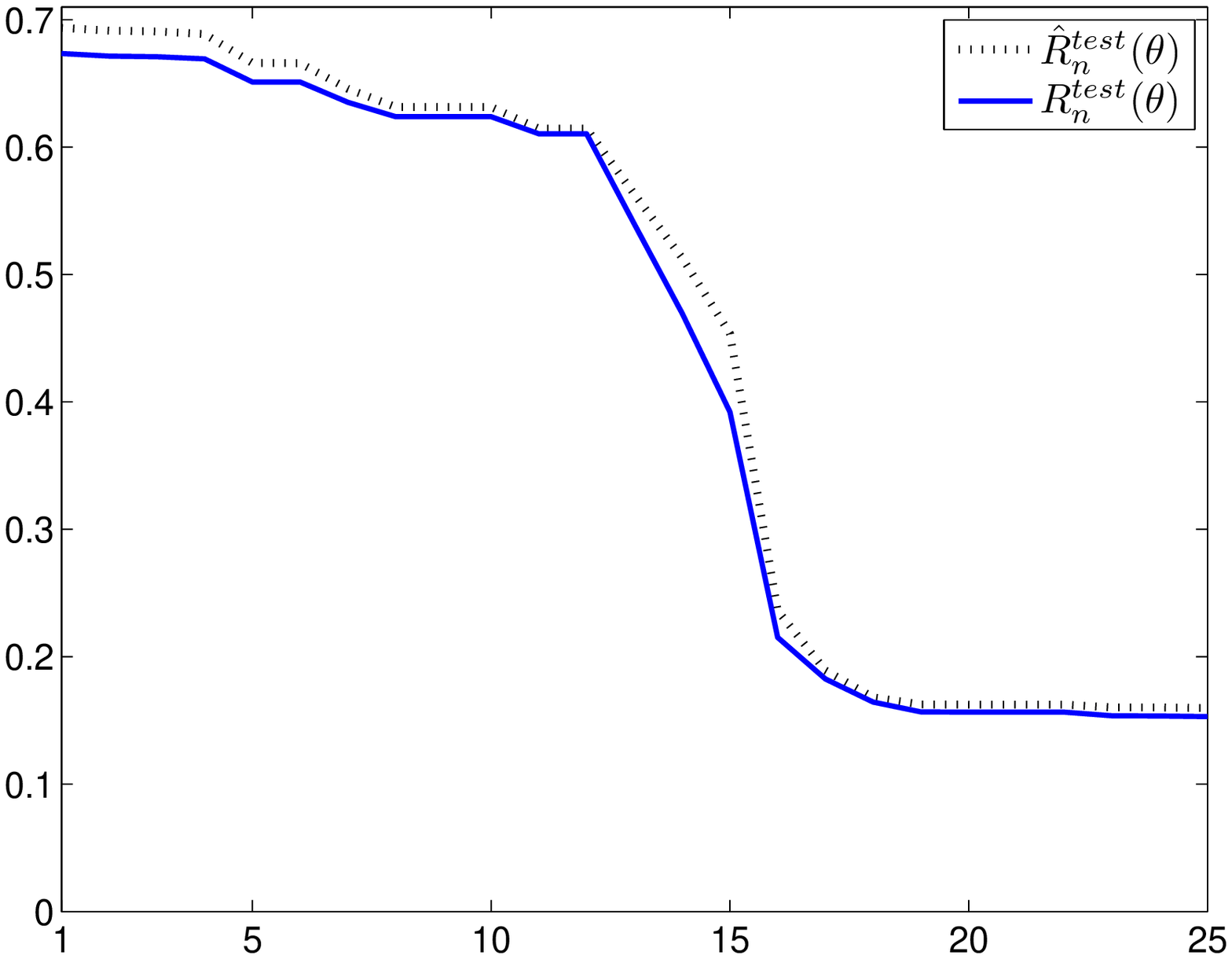} \\
\includegraphics[scale=0.4]{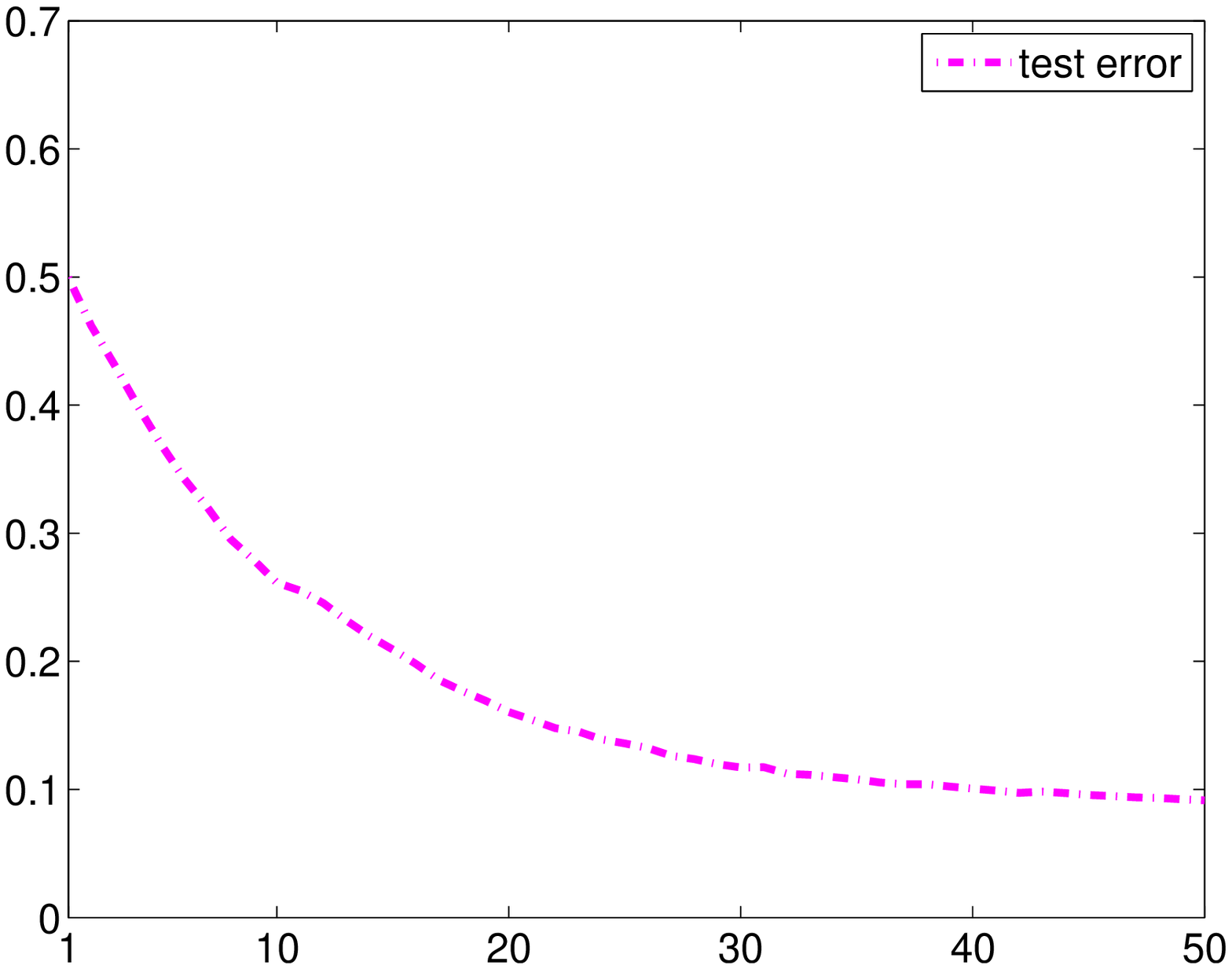}
\includegraphics[scale=0.4]{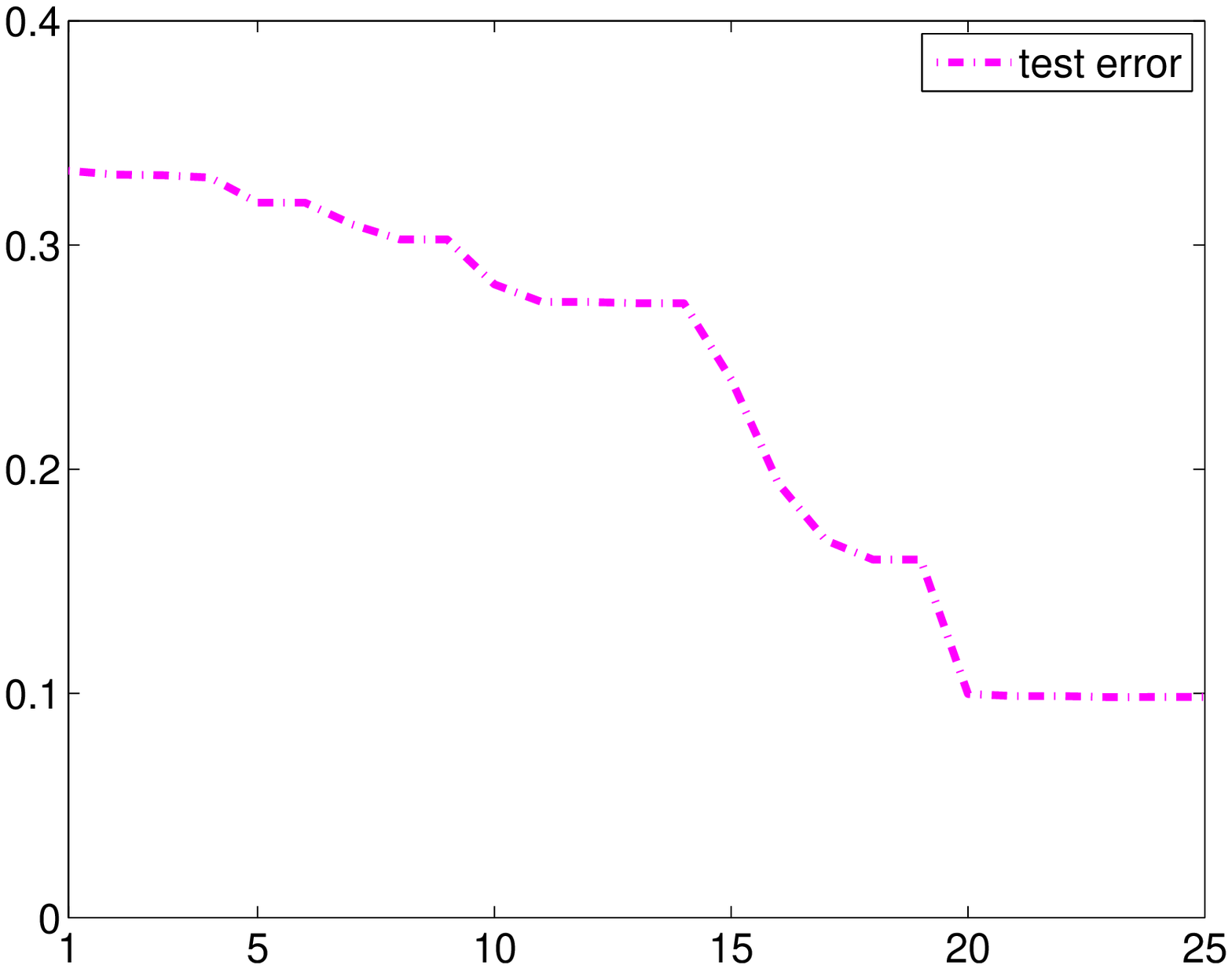}
\caption{Performance of unsupervised logistic regression classifier $\hat\theta_n$ computed using  Algorithm~\ref{alg:gradDescent} (left) and Algorithm~\ref{alg:gridSearch} (right) on the RCV1 dataset. The top two rows show the decay of the two risk estimates $\hat R_n(\hat\theta_n)$, $R_n(\hat\theta_n)$ as a function of the algorithm iterations. The risk estimates of $\hat\theta_n$ were computed using the train set (top) and the test set (middle). The bottom row displays the decay of the test set error rate of $\hat\theta_n$ as a function of the algorithm iterations. The figure shows that the algorithm obtains a relatively accurate classifier (testing set error rate 0.1, and $\hat R_n$ decaying similarly to $R_n$) without the use of a single labeled example. For comparison, the test error rate for supervised logistic regression with the same $n$ is 0.07.}  
\label{fig:rcv1}
\end{figure}

\begin{figure} \centering
  \includegraphics[scale=0.4]{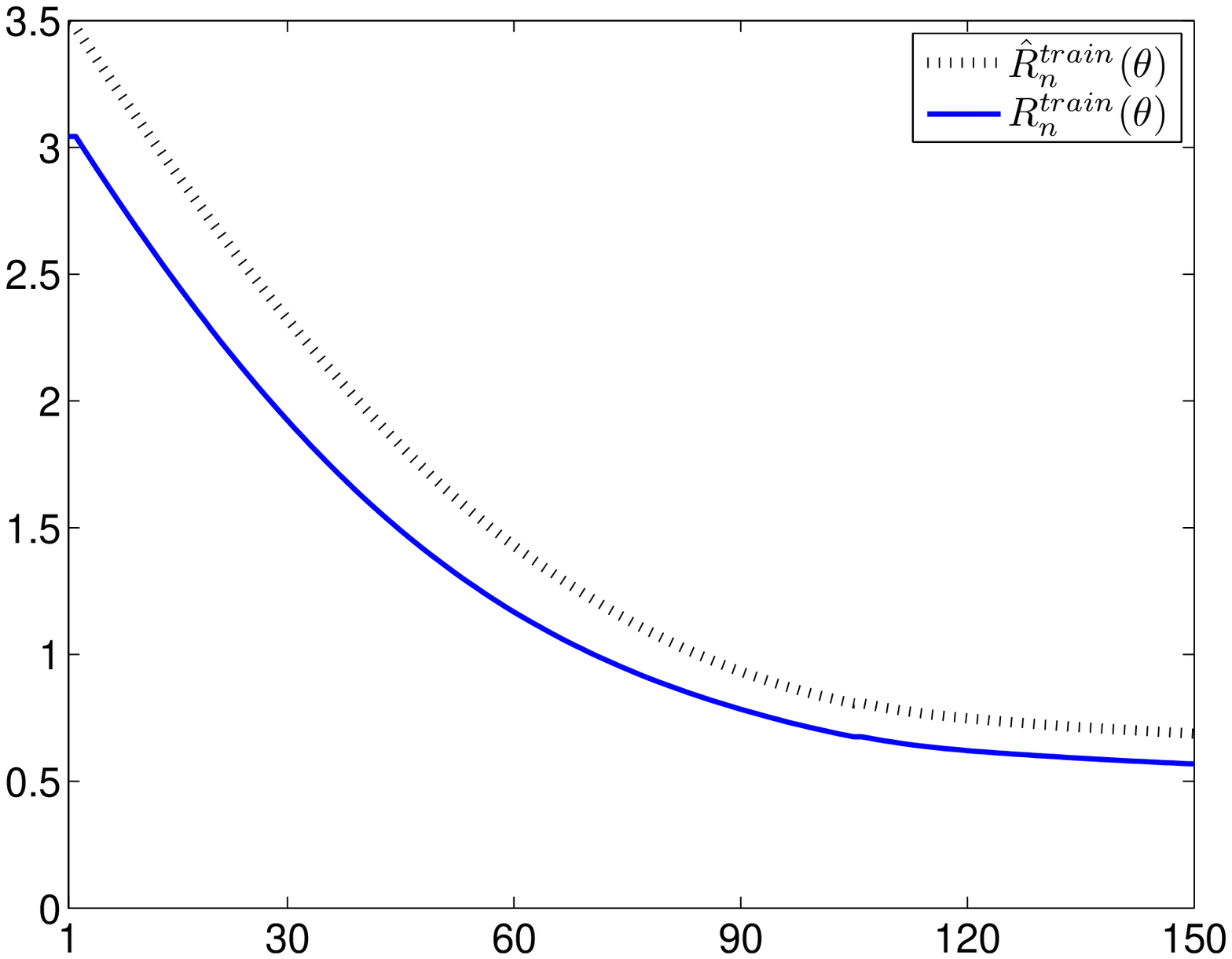}
  \includegraphics[scale=0.4]{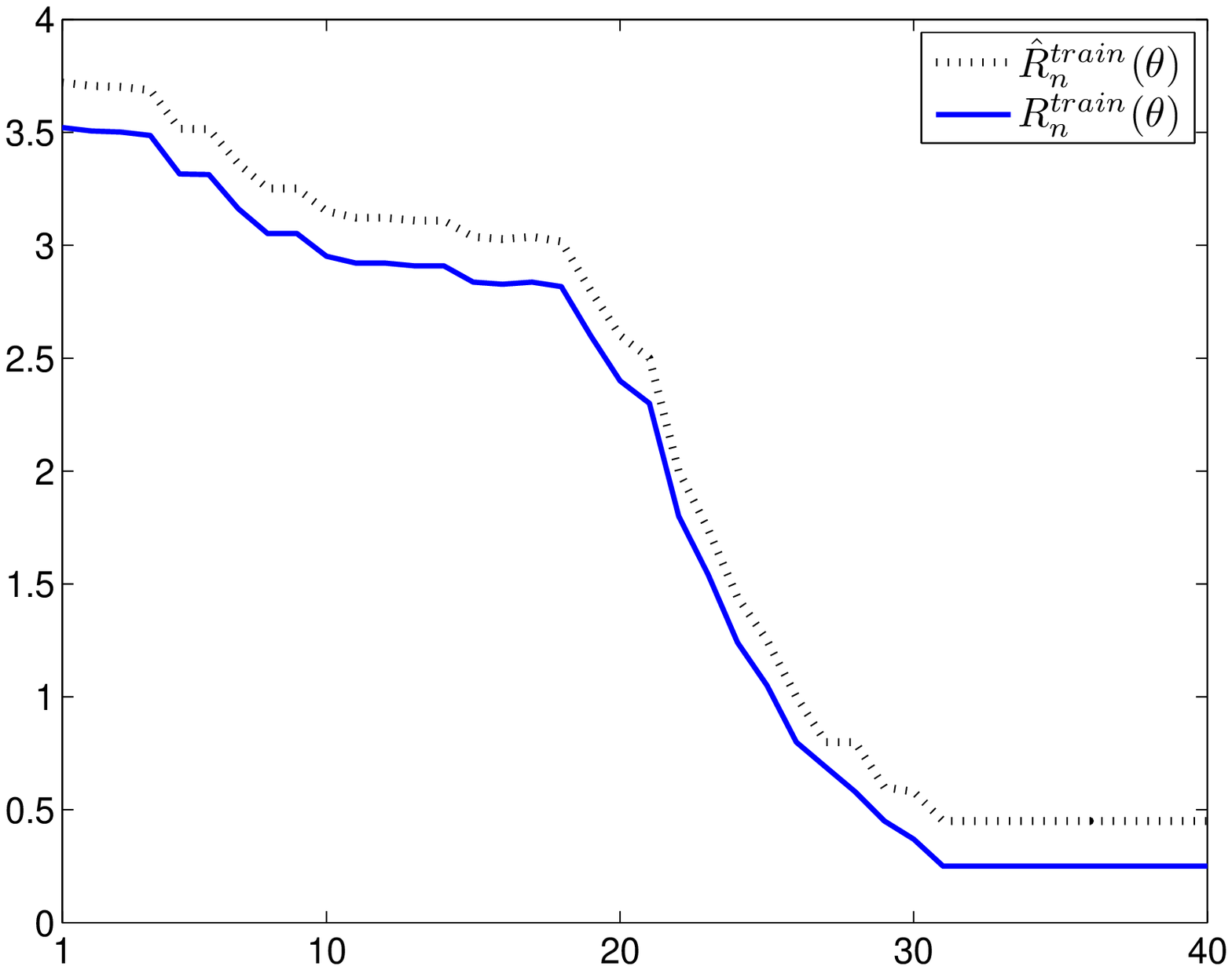}\\
  \includegraphics[scale=0.4]{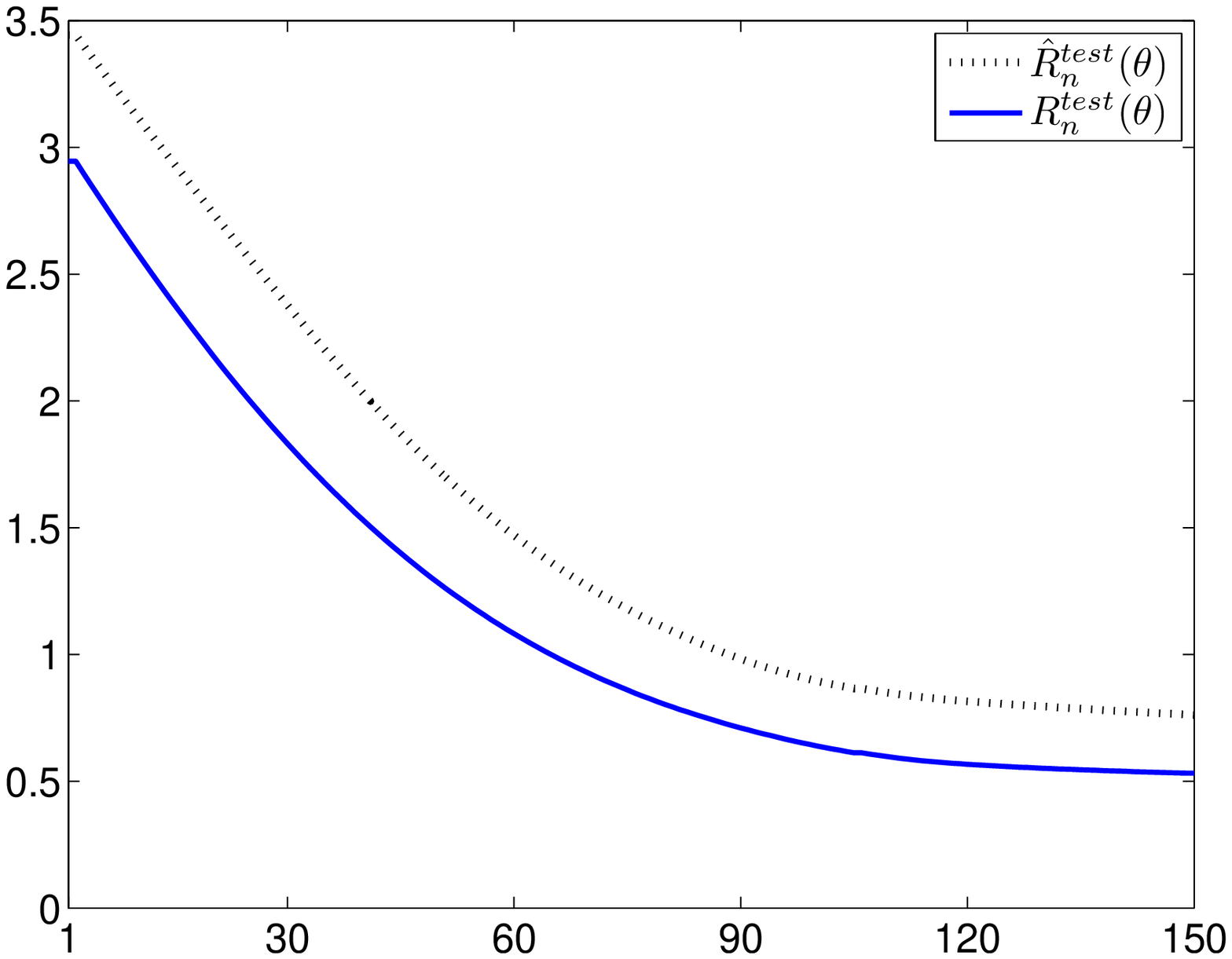}
  \includegraphics[scale=0.4]{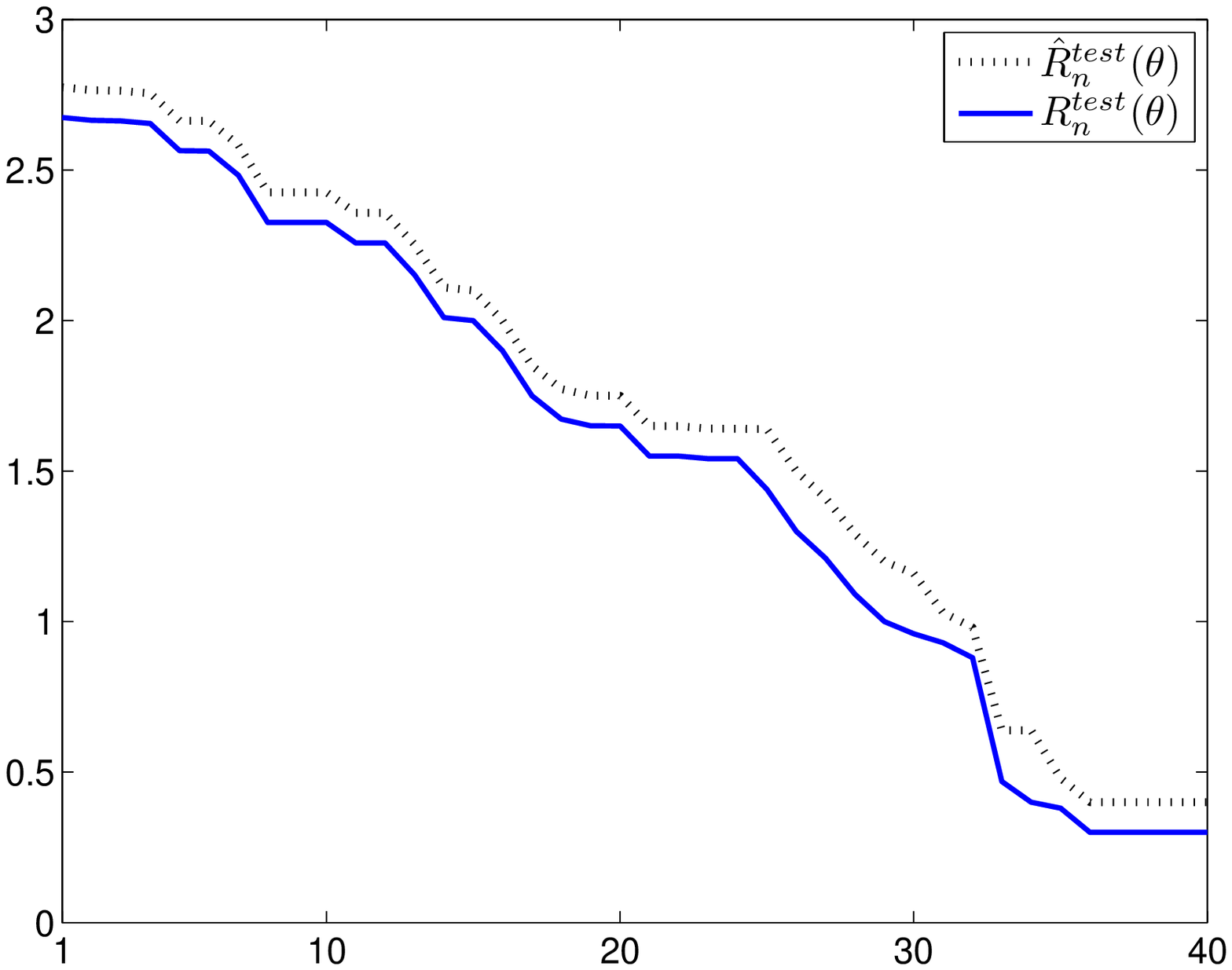}\\
  \includegraphics[scale=0.4]{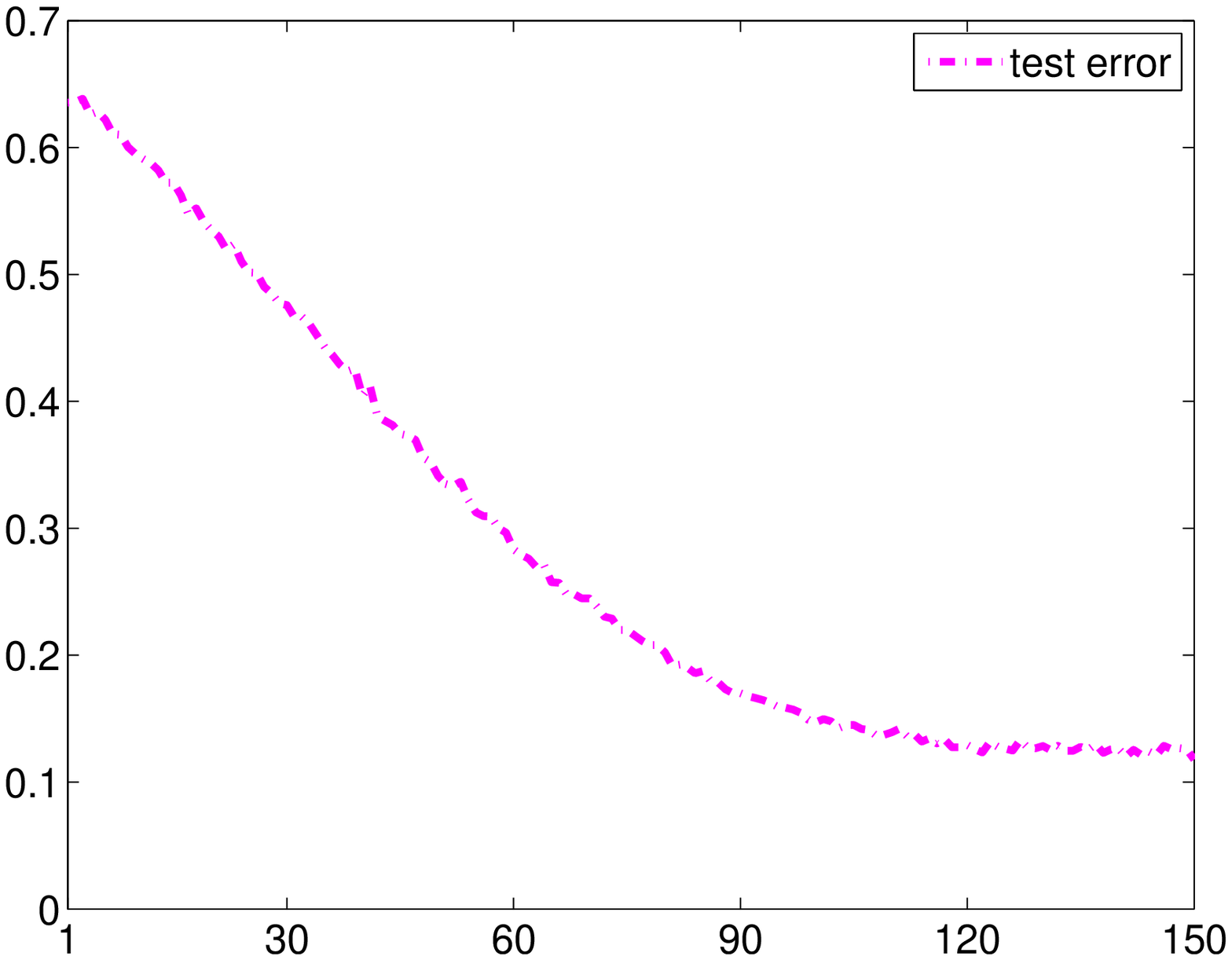}
  \includegraphics[scale=0.4]{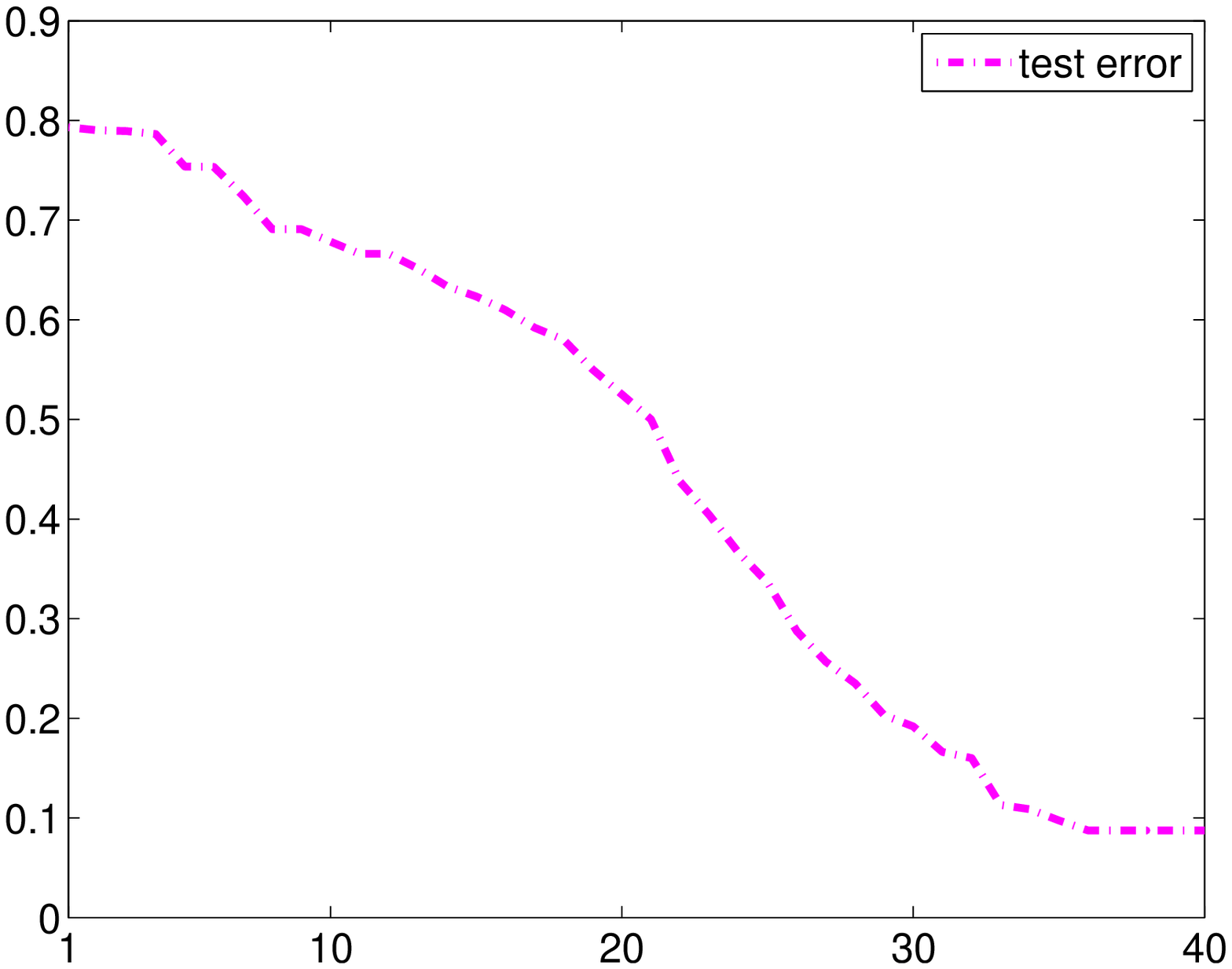}
  \caption{Performance of unsupervised logistic regression classifier $\hat\theta_n$ computed using  Algorithm~\ref{alg:gradDescent} (left) and Algorithm~\ref{alg:gridSearch} (right) on the MNIST dataset. The top two rows show the decay of the two risk estimates $\hat R_n(\hat\theta_n)$, $R_n(\hat\theta_n)$ as a function of the algorithm iterations. The risk estimates of $\hat\theta_n$ were computed using the train set (top) and the test set (middle). The bottom row displays the decay of the test set error rate of $\hat\theta_n$ as a function of the algorithm iterations. The figure shows that the algorithm obtains a relatively accurate classifier (testing set error rate 0.1, and $\hat R_n$ decaying similarly to $R_n$) without the use of a single labeled example. For comparison, the test error rate for supervised logistic regression with the same $n$ is 0.05.}  
\label{fig:mnist}
\end{figure}

Figures~\ref{fig:rcv1}-\ref{fig:mnist} indicate that minimizing the unsupervised logloss estimate is quite effective in learning an accurate classifier without labels. Both the unsupervised and supervised risk estimates $\hat R_n(\hat\theta_n)$, $R_n(\hat\theta_n)$ decay nicely when computed over the train set as well as the test set. Also interesting is the decay of the error rate. For comparison purposes supervised logistic regression with the same $n$ achieved only slightly better test set error rate: 0.05 on RCV1 (instead of 0.1) and 0.07 or MNIST (instead of 0.1).

\subsection{Inaccurate Specification of $p(Y)$}
Our estimation framework assumes that the marginal $p(Y)$ is known. In some cases we may only have an inaccurate estimate of $p(Y)$. It is instructive to consider how the performance of the learned classifier degrades with the inaccuracy of the assumed $p(Y)$. 

Figure~\ref{fig:misspecfig} displays the performance of the learned classifier for RCV1 data as a function of the assumed value of $p(Y=1)$ (correct value is $p(Y=1)=0.3$). We conclude that knowledge of $p(Y)$ is an important component in our framework but precise knowledge is not crucial. Small deviations of the assumed $p(Y)$ from the true $p(Y)$ result in a small degradation of logloss estimation quality and testing set error rate. Naturally, large deviation of the assumed $p(Y)$ from the true $p(Y)$ renders the framework ineffective.

\begin{figure}
\centering
\includegraphics[scale=0.4]{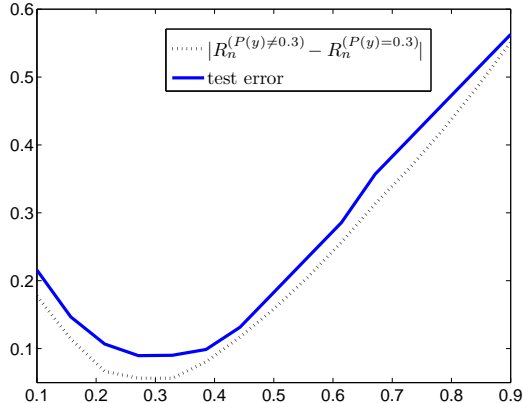}
\caption{Performance of unsupervised classifier training on RCV1 data (top class vs. classes 2-5) for misspecified $p(Y)$. The performance of the estimated classifier (in terms of training set empirical logloss $R_n$~\eqref{eq:empiricalLoss} and test error rate measured using held-out labels) decreases with the deviation between the assumed and true $p(Y=1)$ (true $p(Y=1)=0.3)$). The classifier performance is very good when the assumed $p(Y)$ is close to the truth and degrades gracefully when the assumed $p(Y)$ is not too far from the truth.}\label{fig:misspecfig}
\end{figure}

\section{Related Work}
Related problems have been addressed in \citep{Mann07} and \citep{Quad09}. The work in \citep{Mann07} performs transduction by enforcing constraints on the label proportions. However, their method requires labeled data. The work in \citep{Quad09} aims to estimate the labels of an unlabeled testing set using known label proportions of $n$ sets of unlabeled observations. The key difference between their approach and ours is that they require as many splits of the data as the number of classes and therefore require the knowledge of the label proportions in each split. This is a much stronger assumption than knowing $p(y)$. As noted previously (see comment after Proposition~\ref{prop:identifiability}), our analysis is in fact valid when only the order of label proportions is known, rather than the absolute values. 

An important distinction between our work and the references above is that our work provides an estimate for the margin-based risk and therefore leads naturally to unsupervised versions of logistic regression and support vector machines. We also provide asymptotic analysis showing convergence of the resulting classifier to the optimal classifier (minimizer of \eqref{eq:defR}). Experimental results show that in practice the accuracy of the unsupervised classifier is on the same order  (but slightly lower naturally) as its supervised analog.

\section{Discussion}
In this paper we developed a novel framework for estimating margin-based risks using only unlabeled data. We shows that it performs well in practice on several different datasets. We derived a theoretical basis by casting it as a maximum likelihood problem for Gaussian mixture model followed by plug-in estimation. 

Remarkably, the theory states that assuming normality of $f_{\theta}(X)$ and a known $p(Y)$ we are able to estimate the risk $R(\theta)$ without a single labeled example. That is the risk estimate converges to the true risk as the number of unlabeled data increase. Moreover, using uniform convergence arguments it is possible to show that the proposed training algorithm converges to the optimal classifier as $n\to\infty$ without any labeled data. 

On a more philosophical level, our approach points at novel questions that go beyond supervised and semi-supervised learning. What benefit do labels provide over unsupervised training? Can our framework be extended to semi-supervised learning where a few labels do exist? Can it be extended to non-classification scenarios such as margin based regression or margin based structured prediction? When are the assumptions likely to hold and how can we make our framework even more resistant to deviations from them? These questions and others form new and exciting open research directions. 
{
	\bibliographystyle{plain}
	\bibliography{../../common/groupPapers,../../common/externalPapers}
}

\end{document}